\documentclass{article}
\usepackage[T1]{fontenc}
\usepackage{microtype,graphicx,booktabs}
\usepackage[hidelinks]{hyperref}

\usepackage[preprint]{icml2026}
\usepackage{amsmath,amssymb,mathtools,amsthm}
\newtheorem{theorem}{Theorem}[section]

\newtheorem{proposition}[theorem]{Proposition}
\newtheorem{corollary}[theorem]{Corollary}
\newcommand{\R}{\mathbb R}
\newcommand{\cX}{\mathcal X}
\newcommand{\op}{\mathrm{op}}

\DeclareMathOperator{\diag}{diag}
\DeclareMathSizes{9}{9}{7}{5}
\makeatletter
\renewcommand{\ICML@preprint}{\textit{Preprint. October 30, 2026.}}
\makeatother
\icmltitlerunning{Certifying Retraining Regret through Residual Structure}
\hypersetup{pdfauthor={Peida Li},pdftitle={Certifying Retraining Regret through Residual Structure},pdfsubject={Certified completion-relative ridge regret},pdfkeywords={incomplete data, certified learning, ridge regression, effective dimension}}
\begin{document}
\twocolumn[
\icmltitle{Certifying Retraining Regret through\\Residual Structure}
\begin{icmlauthorlist}
\icmlauthor{Peida Li}{pku}
\end{icmlauthorlist}
\icmlaffiliation{pku}{School of Basic Medical Sciences, Peking University}
\icmlcorrespondingauthor{Peida Li}{2410301207@stu.pku.edu.cn}
\icmlkeywords{incomplete data, certified learning, ridge regression, effective dimension}
\vskip 0.3in
]
\printAffiliationsAndNotice{}
\begin{abstract}
We certify a fixed ridge predictor's regret against full-feature retraining over every continuous completion of an interval-valued table. Our main theorem gives verifiable original-row residual conditions for a uniform error bound independent of nuisance-feature scale and rank, with an asymptotically sharp leading constant and quantitative leakage stability. Full-rank constructions separate certification dimension from feature dimension, including examples requiring directions beyond the candidate. Observed-row anchors make the conditions computable, while exact row geometry and global search return valid upper bounds and feasible original-table witnesses. A finite-scale contraction and common-domain moment fallback implement the structural result. Four real-world tasks retain all model features and certify twenty rank-three instances to width $10^{-5}$. Their primary median widths improve over centered moments. Joint residual refinement additionally certifies all twenty new response-independent masks to the same precision within 76.21 seconds. For a concrete fixed-point predictor contract, pure energy certifies thirteen of fifteen stopping decisions in ten seconds, versus five for moments. Refinement certifies all fifteen within 91.68 seconds at the same objective threshold.
\end{abstract}

\section{Introduction}
Before acquiring finer training measurements, a practitioner can ask whether any legal completion would substantially improve the objective after retraining. Completion-relative regret compares a fixed candidate with each table's retrained optimum. A small uniform upper bound supports stopping acquisition; a feasible violating table rules out robust acceptance and motivates further information or a model update.

Completion-relative learning, model envelopes and repair motivate the target~\citep{zhen2024sigmod,zhu2024zorro,repair2026,vanli2014}. The retrained comparator changes with the completed table, and continuous interior worst cases can defeat endpoint enumeration. We certify arbitrary finite intervals while exploiting residual structure in the prediction task.

The central question is structural: why should a high-dimensional model admit a small certifying subspace? Feature variance describes changing inputs; regret depends on how these changes excite residuals and interact with the retrained comparator. Our original-row residual-energy conditions permit high-rank, large nuisance features on nearly fitted rows. A full-rank construction gives each uncertain column a genuine regret effect, while one learned direction supplies a small uniform error. Leading midpoint-Gram directions can miss nearly all regret even at rank $p-1$.

Residual/coercivity certification is classical~\citep{quarteroni2011,balabanov2019}. We add original-row interval conditions, a sharp uniform bound independent of nuisance scale/rank, and global certification against full retraining (Section~8). Observed anchors construct low-rank spans for four real-world tasks; the general moment verifier covers broader coarsening.
\section{Completion-Relative Ridge Regret}
Let $\cX=\prod_i[L_i,U_i]$ be a compact product of row boxes, with observed responses $y_i$, weights $w_i\geq0$, and $\lambda>0$. For a fixed candidate $\beta\in\R^p$, define
\begin{align}
 J_X(\theta)&=\sum_iw_i(x_i^\top\theta-y_i)^2+\lambda\|\theta\|^2,\\
 h_\beta(X)&=J_X(\beta)-\min_\theta J_X(\theta),\qquad
 H^*=\max_{X\in\cX}h_\beta(X).\label{eq:target}
\end{align}
The intercept, when included, is a feature and receives the same ridge penalty in our experiments. Put
\begin{align}
 M_X&=X^\top W X+\lambda I,\qquad r_X=X^\top Wy,\nonumber\\
 g_X&=M_X\beta-r_X.
\end{align}
Then $h_\beta(X)=g_X^\top M_X^{-1}g_X$. Our output is an interval $[L,U]$ with $L\leq H^*\leq U$, accompanied by a feasible original-table witness for $L$. Our fixed-point candidate uses eight-bit signed coefficients with four fractional bits, sixteen-bit inputs with ten fractional bits, and a thirty-two-bit accumulator (Appendix~\ref{app:fixed-point-v91}). The comparator is unrestricted full-feature ridge retraining.

\begin{proposition}[Endpoints can miss all regret]\label{prop:interior}
Take three observations $((t,0),0)$, $((1,0),1)$ and $((0,1),0)$ with weights $1/3$, $t\in[-A,A]$, $A>0$, and $\lambda\geq0$. For $\beta=(1/(1+A^2+3\lambda),0)$ both endpoint regrets vanish, whereas
\[
 h_\beta(t=0)=\frac{A^4}{3(1+3\lambda)(1+A^2+3\lambda)^2}>0.
\]
\end{proposition}
The endpoint and interior ridge optima follow from diagonal normal equations. At $A=1,\lambda=1/3$ the interior regret is $1/54$, with positive regularization and zero regret at both endpoints. The published SIGMOD Theorem~8.2 of \citet{zhen2024sigmod} explicitly assumes missing-value intervals $[-\infty,+\infty]$. The arXiv v2 Theorem~8.4 uses a sufficiently-wide-interval argument~\citep{zhen2024}. The 2026 repair formulation directly defines worst-completion suboptimality and searches edge repairs~\citep[Section~4.1]{repair2026}. Proposition~\ref{prop:interior} identifies the interior case that a complete certificate over finite intervals must cover.

\section{Residual Structure beyond Feature Rank}\label{sec:structure}
Put $u_X=M_X^{-1}g_X$ and $q_X(v)=2g_X^\top v-v^\top M_Xv$. For a linear subspace $S$, define the intrinsic uniform error and dimension
\begin{align}
 E(S)&=\sup_{X\in\cX}\min_{v\in S}\|u_X-v\|_{M_X}^2,\nonumber\\
 d_\varepsilon&=\min\{\dim S:E(S)\leq\varepsilon\}.\label{eq:intrinsic}
\end{align}
The identity $h_\beta(X)-\max_{v\in S}q_X(v)=\operatorname{dist}_{M_X}(u_X,S)^2$ bounds the global reduction gap by $E(S)$. Uniformity avoids the vacuous global-gap definition: one maximizing displacement alone spans a globally exact line. The energy and moment theorems bound $E$ for their physical spans. Spectral compression controls an affine span; including its offset yields linear rank at most $r+1$. Appendix~\ref{app:dimension} proves these links to classical widths~\citep{binev2011}.

We give sufficient conditions directly in the original table. Let $Q$ have orthonormal columns, $P$ be an orthonormal complement, and $\beta=Q\beta_S$. Partition rows into signal rows $\mathcal S$ and nuisance-rich rows $\mathcal N$. For every legal completion require $P^\top x_i=0$ on $\mathcal S$. Define
\begin{align*}
 A_X&=W_{\mathcal N}^{1/2}X_{\mathcal N}Q,&
 Z_X&=W_{\mathcal N}^{1/2}X_{\mathcal N}P,\\
 e_X&=W_{\mathcal N}^{1/2}(X_{\mathcal N}\beta-y_{\mathcal N}).
\end{align*}
Suppose uniformly $\|e_X\|\leq\epsilon$, $\|A_X\|_{\op}\leq a$, and
\begin{align*}
 \lambda I+X_{\mathcal S,Q}^\top W_{\mathcal S}X_{\mathcal S,Q}&\succeq\kappa I,\\
 \|\lambda\beta_S+X_{\mathcal S,Q}^\top W_{\mathcal S}
       (X_{\mathcal S}\beta-y_{\mathcal S})\|&\leq G,
\end{align*}
where $X_{\mathcal S,Q}=X_{\mathcal S}Q$ and $\kappa>0$. These conditions bound residual energy and retained signal on nuisance-rich rows, with no restriction on the size or rank of $Z_X$.

\begin{theorem}[Residual-energy certificate]\label{thm:energy}
With positive ridge, fixed nonnegative weights and a compact nonempty completion domain, put $r=\dim\operatorname{range}(Q)$, $R=(G+a\epsilon)/\kappa$ and $\eta=(\epsilon+aR)^2$. Then
\[
 E(\operatorname{range}(Q))\le\eta,\qquad d_\eta\le r.
\]
Consequently, for
\[
 V_Q=\max_{\|s\|\leq R}\max_X
 \{2g_X^\top Qs-s^\top Q^\top M_XQs\},
\]
the global enclosure is
\begin{equation}
 V_Q\leq H^*\leq V_Q+\eta.\label{eq:energy}
\end{equation}
The same enclosure holds on the containing cube $[-R,R]^r$.
\end{theorem}
\begin{proof}
At fixed $X,s$, eliminating the discarded coordinate adds
\[
 (e-As)^\top Z(\lambda I+Z^\top Z)^{-1}Z^\top(e-As)
 \le\|e-As\|^2.
\]
The retained Schur normal equation has curvature at least $\kappa$ and right-hand norm at most $G+a\epsilon$, so every full optimizer projects into the radius-$R$ ball. Apply the gain bound at that point and the pointwise error identity to obtain the uniform error. The dimension and global enclosure follow. Appendix~\ref{app:dimension} gives the full pointwise derivation and containing-box argument.
\end{proof}

Theorem~\ref{thm:energy} exploits the actual discarded curvature. A scalar-curvature residual estimate would bound its eliminated gain by $\|Z^\top(e-As)\|^2/\lambda$, which can grow without bound as discarded features are scaled. The contraction in the proof instead bounds it by $\|e-As\|^2$, uniformly in the size and rank of $Z$. Strong retained curvature also supplies the smaller radius $R$. These original-row conditions yield the dimension separation below; the leakage extension below preserves this mechanism nearby, and the general moment bound covers arbitrary tables.

\subsection{Observed Anchors and Constructive Discovery}\label{sec:anchored-energy-v84}
Let $T\in\R^{p\times p}$ be an invertible, numerically stored coordinate transformation: a parameter displacement $z$ represents the physical displacement $Tz$. Complete observations supply an aggregate quadratic anchor. Put $B=\lambda I+X_{\mathcal A}^\top W_{\mathcal A}X_{\mathcal A}$, $r_A=X_{\mathcal A}^\top W_{\mathcal A}y_{\mathcal A}$, $K=T^\top BT$, and $h=T^\top(B\beta-r_A)$. Verify $\mu I\preceq K\preceq\nu I$ for the actual $T$. For $z_0\in\operatorname{range}(Q)$ set $h_c=h-Kz_0$ and $e_c=W^{1/2}[X(\beta-Tz_0)-y]$.

\begin{corollary}[Anchored residual energy]\label{cor:anchor-energy-v84}
Partition the remaining rows into $\mathcal S,\mathcal N$ and require $X_{\mathcal S}TP=0$. Uniformly bound $\|W_{\mathcal N}^{1/2}X_{\mathcal N}TQ\|\le a$, $\|(e_c)_{\mathcal N}\|\le\epsilon$, $\|P^\top h_c\|\le d$, and
\[
 \|Q^\top h_c+(W_{\mathcal S}^{1/2}X_{\mathcal S}TQ)^\top(e_c)_{\mathcal S}\|\le G.
\]
With $c=(\nu-\mu)/2$, put
\begin{align}
 R_A&=\frac{G+a\epsilon+(c/\mu+a/(2\sqrt\mu))d+c\epsilon/(2\sqrt\mu)}{\mu},\nonumber\\
 \eta_{\rm anch}&=[\epsilon+aR_A+(d+cR_A)/\sqrt\mu]^2.\label{eq:anchor-energy-v84}
\end{align}
Then $E(\operatorname{range}(TQ))\le\eta_{\rm anch}$ and $d_{\eta_{\rm anch}}\le\dim Q$. The corresponding original retained optimum gives $V_Q\le H^*\le V_Q+\eta_{\rm anch}$; every full optimizer's retained component lies within $R_A$ of $z_0$.
\end{corollary}
The proof uses $ZD^{-1}Z^\top\preceq I$ and $\|ZD^{-1}\|\le1/(2\sqrt\mu)$ for $D=K_{PP}+Z^\top Z$. The original candidate and full retrained comparator are preserved.

Write $A=W_{\mathcal N}^{1/2}X_{\mathcal N}TQ$, $Z=W_{\mathcal N}^{1/2}X_{\mathcal N}TP$, and $e_N=(e_c)_{\mathcal N}$. Verify $\|Z\|^2\le\rho$ and $\|A^\top e_N\|\le b$. A verified radius $R\le R_A$ for both optima's retained components gives
\begin{align}
 \gamma^2&=\rho/(\mu+\rho),\nonumber\\
 E_N(R)&=\min\{(\epsilon+aR)^2,\epsilon^2+2bR+a^2R^2\},\nonumber\\
 \eta_C&=\left[\gamma\sqrt{E_N(R)}+(d+cR)/\sqrt\mu\right]^2.\label{eq:contraction-main-v89}
\end{align}
It gives $V_Q\le H^*\le V_Q+\eta_C$. The mixed moment also tightens $R$; Appendix~\ref{app:contraction-v88} gives its formula and proof. The choice $\gamma=1$, $R=R_A$ and the first residual envelope recovers~\eqref{eq:anchor-energy-v84}.

\paragraph{Common-domain adaptation.}
Let $R_c$ be the verified containing coordinate radius around the same retained reference. Aggregate discarded reference gradients and cross blocks give $G_P,C_R$ as in Theorem~\ref{thm:moment}. On this common domain use
\begin{equation}
 \eta_{\rm adapt}=\min\{\eta_C,(G_P+C_RR_c)^2/\mu\}.\label{eq:adaptive-v90}
\end{equation}
Both bounds enclose the same full-comparator problem. The single search keeps their smaller tail and pays for both verifications. Since $R_c$ is no larger than the independently centered moment radius, its sufficient tail is no larger than that moment tail. Appendix~\ref{app:contraction-v88} gives the containing-domain proof.

\paragraph{A constructive rank bound.}
Let $\theta_A=B^{-1}r_A$. Include the transformed centers and interval generators of the signal rows, and $T^{-1}(\beta-\theta_A)$, in the retained span. Taking $z_0=T^{-1}(\beta-\theta_A)$ makes $h_c=0$ and satisfies signal support. If signal row $i$ has $k_i$ unknowns, the required rank is at most $1+\sum_{i\in\mathcal S}(1+k_i)$, irrespective of nuisance rank and feature count. One signal row with one unknown therefore proposes at most three directions. Appendix~\ref{app:anchor-energy-v84} proves the construction and verifies its actual rounded implementation.

\paragraph{Support and structural rank.}
Let $J_{\mathcal S}=\{j:\exists i\in\mathcal S,\ U_{ij}>L_{ij}\}$. Under an invertible transformation $T$, exact support $X_{\mathcal S}TP=0$ for every completion implies
\[
 T^\top e_j\in\operatorname{range}(Q)\quad(j\in J_{\mathcal S}),
 \qquad r\geq |J_{\mathcal S}|.
\]
Indeed, subtract two signal rows differing only in coordinate $j$; invertibility makes the resulting generators independent. This constraint concerns uncertain signal columns. Nuisance-rich rows can have independent uncertainty in every feature. Approximate signal support is covered by the leakage extension.

\paragraph{Support leakage.} A quantitative extension allows nonzero discarded signal components, with constants converging to Theorem~\ref{thm:energy} as leakage vanishes. Appendix~\ref{app:leakage-v79} gives its original-row assumptions, radius and candidate-component correction.

\subsection{Full Rank and Many Uncertain Columns}
For $m=4^k$, take $p=m+1$, $n=2m$, weights $1/(2m)$, $\lambda=1/2$ and candidate $\beta=e_0$. Every scalar below varies independently:
\begin{align}
 x_i&=u_ie_0,&u_i&\in[3/4,5/4],& y_i&=2,\nonumber\\
 x_{m+j}&=\alpha e_0+t_je_j,&t_j&\in[\sqrt m,2\sqrt m],&y_{m+j}&=\alpha,
 \label{eq:family}
\end{align}
for $i,j=1,\ldots,m$ and dyadic $\alpha>0$. Every table has full column rank, every feature column contains a nondegenerate uncertain interval, and $\beta$ is exactly the full-feature midpoint ridge fit. Changing any nuisance $t_j$ changes regret whenever the signal gradient is nonzero. The nuisance relative interval widths and weighted feature energies stay bounded below as $m$ grows; total row energy grows with dimension.

The original and one-dimensional optima are exactly
\begin{equation}
 H^*=\frac5{4000+512\alpha^2},\qquad
 V_{e_0}=\frac1{800+512\alpha^2}.\label{eq:family_optima}
\end{equation}
Theorem~\ref{thm:energy} has $\epsilon=0$, $a=\alpha/\sqrt2$, $G=1/32$, $\kappa=25/32$ and $R=1/25$, giving error $\alpha^2/1250$, independently of $p$. At $\alpha=1/32$ this is $7.8125\times10^{-7}$, while the exact omitted regret is about $6.2453\times10^{-7}$. \begin{corollary}[Sharp leading energy constant]\label{cor:energy-sharp-v81}
Scale only the nuisance intervals in~\eqref{eq:family} by $s>0$ and put $\Delta_s=H_s^*-V_{e_0}$. For every $\alpha>0$,
\[
 \lim_{s\to\infty}\frac{\Delta_s}{\alpha^2/1250}
   =\frac{800}{800+512\alpha^2}\longrightarrow1
 \quad(\alpha\downarrow0).
\]
Thus the coefficient one in~\eqref{eq:energy} cannot be uniformly reduced over its assumption class, already at $\epsilon=0$ and retained rank one.
\end{corollary}
Schur elimination gives the limit; Appendix~\ref{app:scale-v80} supplies the exact full-rank original-row calculation.

\paragraph{Automatic direction recovery.} The uncertainty-gradient proposal defined in Appendix~\ref{app:new} recovers $e_0$ under scalar ridge-anchor preprocessing: first-order signal terms cancel, second-order terms lie in $e_0$, and nuisance generators vanish at the fitted candidate.

The endpoint-average second moment has at least $m$ eigenvalues $\ge5/4$, forcing rank-$r$ uniform ridge-whitened tails to have $\rho\ge5/2$ for $r\le m-1$. At $\alpha=1/32$, every leading midpoint-Gram PCA span of rank $1\le r\le p-1$ captures at most $9/3271682$, about 0.23\% of $H^*$. One residual-oriented direction suffices; Appendix~\ref{app:energy} proves these specific Gram-PCA and spectral comparisons.

\noindent\begin{minipage}{\columnwidth}
\begin{proposition}[Intrinsic dimension beyond one]\label{prop:intrinsic-two-v83}
For the two-active-coordinate full-rank family of Appendix~\ref{app:intrinsic-two-v83}, every line $S$ satisfies $E(S)\ge25600/1371850897$, while its active plane has error at most $17/20480000$. Hence $d_{10^{-5}}=2$.
\end{proposition}
Two legal completions have orthogonal full-optimizer displacements, giving the lower bound for every line. The energy theorem supplies the plane upper bound; the appendix gives the exact original-row proof.
\end{minipage}

\begin{figure*}[t]\centering
\includegraphics[width=\textwidth]{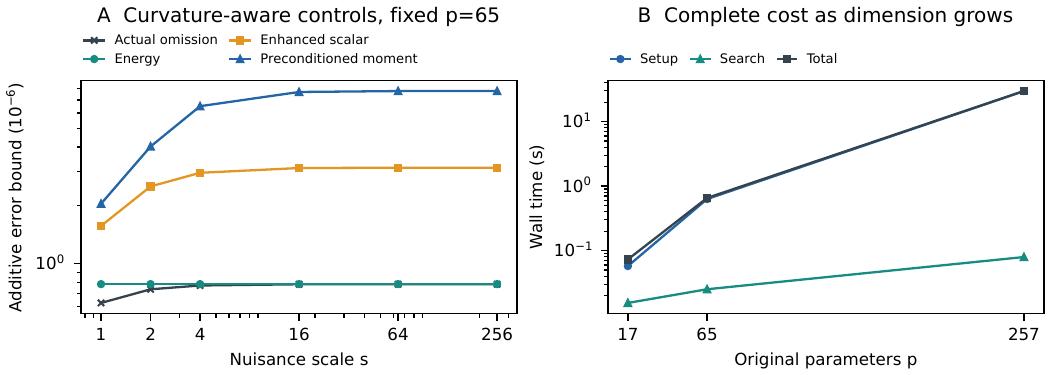}
\caption{The structural gain and its complete cost. (A) At fixed $p=65$, energy, enhanced scalar curvature and an actual dyadically preconditioned moment bound all remain bounded under nuisance scaling; energy has the smallest bound. These additive tails precede independent original-objective caps. (B) Single complete-pipeline timings at $\alpha=1/32$ and retained rank one. Setup includes input, anchor and method verification; their sum with search is the total. The ten-second search cap excludes setup; at $p=257$, setup is $29.74$ seconds and search $0.079$ seconds.}
\label{fig:structure-main-v82}\end{figure*}

\section{Exact Row Geometry}
For $v=\beta-\theta$, expose the comparator:
\begin{align}
 H^*&=\max_v F(v),\label{eq:dual}\\
 F(v)&=\max_{X\in\cX}\{2g_X^\top v-v^\top M_Xv\}\nonumber\\
 &=\lambda(2\beta^\top v-\|v\|^2)
   +\sum_iw_i\max_{x\in[L_i,U_i]}q_i(x;v),\nonumber\\
 q_i(x;v)&=(x^\top\beta-y_i)^2-(x^\top(\beta-v)-y_i)^2.\nonumber
\end{align}
The maxima share the same pairs $(X,v)$; positive ridge and bounded data give a bounded optimizing region.
\begin{theorem}[Row-face extremizers]\label{thm:row}
There exists a maximizer of~\eqref{eq:target} in which each row has at most one nondegenerate interval coordinate strictly between its endpoints.
\end{theorem}
For a fixed comparator, projecting a row to its two residuals gives a planar zonotope with at most $2k_i$ boundary edges. Sorting the generators and checking edge endpoints and concave scalar stationary points yields an exact oracle in $O(p+k_i\log(k_i+1))$ arithmetic operations, including center projection. The same procedure handles $a^2-cb^2$, $c>0$. Segment and point images are handled directly; original-box witnesses are reconstructed from the maximizing edge. The proof, coefficient intervals and degeneracies are in Appendix~\ref{app:method-details}.

\section{Computable Certificates in an Actual Metric}
The algorithm verifies any proposed span against full-feature retraining over the original intervals. Appendix Table~\ref{tab:tails} distinguishes three sufficient bounds. The sharp spectral theorem appears in Appendix~\ref{app:spectral-v79}; the actual-metric moment bound below supports general coarsening and the acquisition study. Real-world energy uses~\eqref{eq:contraction-main-v89}.

\subsection{Observed Rows and an Actual Dense Metric}\label{sec:anchor}
Let $\mathcal A$ contain fully observed rows. With the original weights and ridge included once, form
\begin{equation}
 G_0=\lambda I+\sum_{i\in\mathcal A}w_ix_ix_i^\top,
 \qquad r_0=\sum_{i\in\mathcal A}w_ix_iy_i.
\end{equation}
General moments use a rounded inverse-transpose midpoint-Gram Cholesky factor; energy uses the observed-anchor factor. Both round $T$ elementwise to $2^{-20}$. The anchor remains $G_0$; structural controls also use ridge-only preprocessing. We verify the actual metric
\begin{equation}
 K=T^\top G_0T,\qquad 0<\mu I\preceq K\preceq\nu I.
 \label{eq:actual_metric}
\end{equation}
The moment method retains $K$ exactly. Observed rows can improve curvature and residuals independently. The spectral theorem's exact whitening and safe rounded surrogate are detailed in the appendix.

\subsection{A Regret-Oriented Residual-Moment Tail}\label{sec:moment}
For arbitrary invertible $T$, define $\widehat M_X=T^\top M_XT$ and $\widehat g_X=T^\top g_X$. Suppose $\widehat M_X\succeq\mu I$ and $U_0\geq H^*$ is valid. For a full-column-rank retained matrix $R$, put $H=R^\top R\succeq\mu_HI$ and $\Pi=RH^{-1}R^\top$. The retained original objective is
\begin{equation}
 F_R(s)=\max_X\{2\widehat g_X^\top Rs-s^\top R^\top\widehat M_XRs\}.
\end{equation}
It uses the actual anchor matrix in~\eqref{eq:actual_metric}. We compute verified constants
\begin{align}
 G_r&\geq\sup_X\|(I-\Pi)\widehat g_X\|,\nonumber\\
 C_r&\geq\sup_X\|(I-\Pi)\widehat M_XR\|_{\op},\qquad
 R_s=\sqrt{U_0/(\mu\mu_H)}.\nonumber
\end{align}
\begin{theorem}[Residual-moment certificate]\label{thm:moment}
For $\mathcal B=[-R_s,R_s]^r$, let $V_r=\max_{s\in\mathcal B}F_R(s)$. The uniform original-coordinate error satisfies $E(\operatorname{range}(TR))\le\eta_r$ and $d_{\eta_r}\le r$. Consequently,
\begin{equation}
 V_r\leq H^*\leq V_r+\eta_r,\qquad
 \eta_r=\frac{(G_r+C_rR_s)^2}{\mu}.\label{eq:moment_tail}
\end{equation}
A retained search interval $[f,U_R]$ therefore gives the original interval $[f,U_R+\eta_r]$, which can be tightened by any feasible full-model witness and by $U_0$.
\end{theorem}
Schur elimination bounds the discarded gain by $\|(I-\Pi)(\widehat g_X-\widehat M_XRs)\|^2/\mu$; coercivity places the full optimizer's projection inside $\mathcal B$. Appendix~\ref{app:new} details the actual-metric calculation.

Aggregate coordinate ranges are summed across rows before taking norms, preserving cancellation near a midpoint optimum. Verified midpoint cross blocks and original row-interval perturbations supply $C_r$. Uncertainty-gradient generators propose $R$; its actual Gram/projector and every tail constant determine the guarantee. Appendix~\ref{app:new} gives their formulas.

\section{Complete Search and Complexity}
The outer objective is a maximum of concave quadratics, so direct concavity is unavailable. Its special curvature still gives a complete certificate.

For the following statement let $\mathcal F(s)=\max_X q_X(s)$ denote a search envelope, with each $q_X$ a concave quadratic; this includes $F$ and $F_R$.

\begin{theorem}[Curvature interpolation]\label{thm:interp}
Suppose the negative quadratic coefficient in coordinate $j$ is uniformly at most $H_j$. For any parameter box $B=[l,u]$,
\begin{equation}
 \sup_{s\in B}\mathcal F(s)\leq
 \max_{q\in\mathrm{vert}(B)}\mathcal F(q)
      +\frac14\sum_jH_j(u_j-l_j)^2.\label{eq:interpolation}
\end{equation}
The same statement holds with the actual nonorthogonal metric $H$ and the corresponding curvature diagonals.
\end{theorem}
The proof is in Appendix~\ref{app:method-details}.

Best-first branch and bound bisects the curvature-weighted widest side and retrains row witnesses in all features. The largest surviving leaf upper remains valid at a time cap. Outward arithmetic encloses every query. A center-gradient bound supports full-dimensional controls without root-corner enumeration.

For a root box of radius $R_0$ and exact row queries, a uniform grid with side at most $2\sqrt{\delta/\sum_jH_j}$ has global search error at most $\delta$. It uses at most
\begin{equation}
 \left[1+\left\lceil R_0\sqrt{\frac{\sum_jH_j}{\delta}}\right\rceil\right]^r
\end{equation}
grid points. The exponent is the retained model dimension $r$. Dense anchor and metric setup require $O(np^2+p^3)$ arithmetic; each attempted full-model witness fit has the same order. Proposal and tail costs are listed in Appendix~\ref{app:cost}. The total error adds its verified tail (Appendix Table~\ref{tab:tails}). Exact low rank is the zero-tail case.

\begin{proposition}[Validity and rank-adaptive completeness]\label{prop:complete-v91}
For any fixed retained rank and verified tail, every returned interval contains $H^*$. If rank enrichment reaches $p$, queries can be enclosed to arbitrary accuracy, and search time is unrestricted, then every precision $\delta>0$ is attainable with a finite original-table certificate.
\end{proposition}
At full rank the tail vanishes; a finite grid and query precision $\delta/4$ give the claimed width and a feasible original-table lower (Appendix~\ref{app:complete-v91}).

\paragraph{Refining the original domain.}
When a sufficient tail remains above the requested precision, a fixed reference $\vartheta$ gives the joint bound
\[
h_\beta(X)\le J_X(\beta)-J_X(\vartheta)+\mu^{-1}\|T^\top(M_X\vartheta-r_X)\|^2.
\]
Its degree-four cell polynomial preserves the objective--residual dependence. A moment relaxation with independently validated dyadic duals bounds each completion box; covering subdivision refines the global interval. This supplies a precision fallback with the same full retrained comparator (Appendix~\ref{app:joint-refinement-v92}).

\subsection{Rank Selection and Error Allocation}
Our selector searches specified proposals at ranks 1,2,3 for a verified tail at most $2.5\times10^{-6}$. It returns a verified sufficient rank when available. Otherwise, the span with the smallest verified tail supplies the first-stage interval, and joint residual refinement continues toward the requested total precision.

Algorithm~\ref{alg:certificate} in Appendix~\ref{app:joint-refinement-v92} gives the full procedure.

The total error is search plus tail, with a verified anchor error for identity-metric surrogates; the actual-metric moment method has zero anchor error. The complete interval determines precision, and its endpoints remain valid at every checkpoint. The decision is STOP if $U\leq\tau$, violating if $L>\tau$, and unresolved otherwise. Width $U-L$ measures numerical precision; $\tau$ is the model-retention tolerance. If $U\le\tau$, every full retrained optimum satisfies
\[
 |z^\top(\beta-\theta_X^*)|\le\sqrt{\tau z^\top G_0^{-1}z}
\]
for any feature vector $z$, using $M_X\succeq G_0\succ0$. This bounds prediction deviation from full retraining (Appendix~\ref{app:extensions}).

\begin{figure*}[t]\centering\includegraphics[width=.97\textwidth]{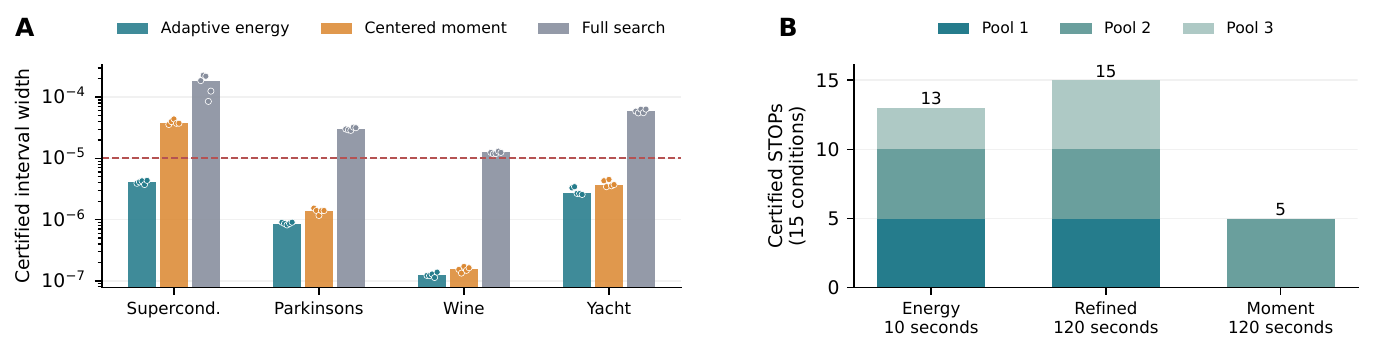}
\caption{Residual structure connects precision to acquisition decisions. A: four tasks with five masks each, all features retained; points are masks and bars are median widths under ten-second complete budgets. B: the fixed 2\% rule on fifteen conditions in three disjoint pools. Energy certifies thirteen STOPs by ten seconds. Refinement certifies all fifteen by 120 seconds, versus five for moments at the same deadline, avoiding 330 additional cells under the fixed thirty-three-cell policy.}\label{fig:public-energy-v84}\end{figure*}

\section{Experiments}\label{sec:experiments}
We test residual structure, full-feature certification and decision value. Timed pairs use one pinned i7-13800H logical CPU and single-thread libraries, reserving two performance cores. Additional queues report their own conditions. Deterministic enclosures cover masks on fixed record pools. Code and protocols: arXiv source archive, \texttt{anc/README.md}; \url{https://github.com/Peida-Lee/certifying-retraining-regret}.

\subsection{The Structural Separation in Computation}
The family~\eqref{eq:family} crosses $m=16,64,256$ with $\alpha=0,1/64,1/32,1/16$. Both methods share the learned rank-one direction, full comparator, ten-second search cap and $10^{-6}$ target. Energy verifies the conditions and meets all twelve targets; moment meets the three zero-coupling targets, reaches six tail floors and exhausts three caps. At $p=65,\alpha=1/32$, widths are $0.657\times10^{-6}$ versus $57.1\times10^{-6}$. Complete energy costs are $0.073,0.655,29.82$ seconds at $p=17,65,257$ (Figure~\ref{fig:structure-main-v82}B). The largest case spends $29.74$ seconds in setup and $0.079$ in search (Appendix~\ref{app:cost}).

\paragraph{Curvature-aware controls at fixed dimension.}
Fix $p=65$, the candidate, direction and signal intervals, and scale nuisance features by $s\in\{1,2,4,16,64,256\}$. Energy stays $\alpha^2/1250$. An enhanced scalar-curvature bound is $2\alpha^2s^2/[625(1+s^2)]$; a verified preconditioner $T_s=\operatorname{diag}(1,s^{-1}I)$ also keeps moments bounded. Figure~\ref{fig:structure-main-v82}A compares these stronger controls with exact omitted regret. Energy is sharper across this diagonal family (Appendix~\ref{app:scale-v80}).

\subsection{Automatic Multidimensional Structure}
Eighteen two/three-active-coordinate conditions test learned directions. In the unperturbed cases, candidate-line errors exceed $2\times10^{-4}$ while learned energy tails are below $8.34\times10^{-7}$. A verified sparse-block elimination completes the full residual/leakage grid to width $10^{-5}$ in 0.06--0.13 seconds per condition. Appendix~\ref{app:multirank-v82} reports the structural first-stage ablation and all eighteen exact block audits.

Nonzero-residual and leakage series, including six automatic rank-one cases, appear in Appendices~\ref{app:epsilon}, \ref{app:leakage-v79} and \ref{app:auto-energy-v80}.

\subsection{Real-World Experiments}
We evaluate Superconductivity, Parkinsons Telemonitoring, Wine Quality and Yacht Hydrodynamics~\citep{uci_super,uci_parkinsons,uci_wine,uci_yacht}. Every model retains all supplied predictor columns and its intercept. Recorded outcomes use the specified affine normalization and $2^{-10}$ quantization. Each instance has one signal-row unknown and sixteen nuisance rows with two unknowns each, giving 33 uncertain cells across multiple features. The candidate is midpoint ridge rounded to four fractional binary bits; the comparator retrains in all original features. The signal row has maximum pilot residual; sixteen nuisance rows are sampled from the 32 smallest residual-interval scores after three fixed passes. Five masks per task use the constructive rank-three span.

Adaptive energy~\eqref{eq:adaptive-v90} gives the smallest median width among A, M and full search on all four datasets (Figure~\ref{fig:public-energy-v84}, Table~\ref{tab:public-energy-v84}). The median paired width reductions against independently centered moments are 89.7\%, 38.4\%, 20.3\%, 25.3\%, respectively. All twenty adaptive intervals meet $10^{-5}$ precision within ten-second complete budgets. Appendix~\ref{app:public-energy-v84} reports the tail and search contributions to this total precision. The comparison uses identical row boxes, candidates, retained spans, feasible midpoint witnesses and coarse caps. Preprocessing, bound verification, search and full-feature witness fits all enter the budget.

The refined contraction links the structural theorem to these gains: verified discarded row energy reduces the contraction factor, and mixed residual moments tighten the retained radius and residual envelope. Independent rational audits reconstruct twenty primary receipts and feasible original-table witnesses, including a complete search tree for each dataset. Appendix~\ref{app:public-energy-v84} gives the full parameters, paired runs and cost breakdown.
\begin{table*}[t]\centering\small\setlength{\tabcolsep}{4pt}
\caption{Four real-world tasks, all features and five masks retained. All regret quantities use $10^{-6}$ normalized-objective units. $h_c$ is midpoint regret; the increment column bounds the median of $H^*-h_c$. Widths are marginal medians; gain is the median paired width reduction against M. A is adaptive energy and M independently centered moments. All methods have ten-second complete budgets.}
\label{tab:public-energy-v84}
\begin{tabular}{lrrrrrrr}\toprule Data&$p$&$h_c$&Increment bounds&A width&M width&Full width&Gain\\\midrule
Supercond. & 82 & 1664.298 & [23.232, 27.351] & 4.119 & 37.438 & 186.00 & 89.7\%\\
Parkinsons & 20 & 817.937 & [10.773, 11.677] & 0.862 & 1.403 & 30.10 & 38.4\%\\
Wine Quality & 12 & 424.515 & [2.323, 2.461] & 0.122 & 0.153 & 12.49 & 20.3\%\\
Yacht & 7 & 744.339 & [17.220, 19.789] & 2.669 & 3.723 & 58.99 & 25.3\%\\
\bottomrule\end{tabular}\end{table*}

Matched controls vary nuisance-row sampling and candidate precision on the same boxes. With a common 600-second cap, refinement certifies all twenty-four matched conditions to width $10^{-5}$ and centered moments certify fifteen. Refinement's measured algorithm costs are at most 459.69 seconds from the stored inputs. Appendix~\ref{app:precision-controls} separates the ten-second first-stage ablation from refinement; Table~\ref{tab:public-energy-v84} links midpoint cost, uncertainty increment and precision.

\paragraph{Response-independent masks and precision refinement.}
We evaluate five further masks per task, with all seventeen row indices sampled independently of responses and residuals. The reference, branching and solver rules are fixed before these runs. Our complete algorithm reaches $10^{-5}$ width in all twenty instances: eight complete by 10 seconds, sixteen by 60 seconds and all twenty by 120 seconds. The longest complete run is 76.21 seconds; centered moments reach the target in five instances under the common 600-second cap. Table~\ref{tab:independent-v91} compares the final full-comparator certificates and success counts. Appendix~\ref{app:core-decisions-v91} gives the structural-stage ablation and every mask; Appendix~\ref{app:joint-refinement-v92} gives the polynomial bound, validated moment dual and complete timing protocol.
\begin{table}[t]\centering\small\setlength{\tabcolsep}{2.5pt}
\caption{Five new response-independent masks per task. Median widths use $10^{-6}$ units. Our complete algorithm combines structural certification and joint refinement; M is centered moments. Time includes all algorithm stages. Success counts use the common $10^{-5}$ target and 600-second cap.}\label{tab:independent-v91}
\begin{tabular}{lrrrrrr}\toprule Data&Ours width&M width&Ours (s)&Ours&M\\\midrule
Supercond. & 3.174 & 254.4 & 65.55 & 5/5 & 0/5\\
Parkinsons & 0.1055 & 57.95 & 52.29 & 5/5 & 0/5\\
Wine Quality & 6.321 & 7.278 & 6.03 & 5/5 & 5/5\\
Yacht & 0.001748 & 25.39 & 9.92 & 5/5 & 0/5\\
\bottomrule\end{tabular}\end{table}

\begin{table}[t]\centering\small\setlength{\tabcolsep}{3pt}
\caption{All twenty identical primary instances. Entries are median width ($10^{-6}$) / median complete seconds across five masks. A/M use outward enclosures; SCIP uses its numerical dual with $10^{-10}$ search tolerance and a nominal ten-second cap. Actual costs include indivisible-call overruns.}
\label{tab:external-main-v90}
\begin{tabular}{lrrr}\toprule Data&A&M&SCIP\\\midrule
Supercond. & 4.119 / 6.82 & 37.44 / 6.29 & 106.6 / 9.96\\
Parkinsons & 0.8616 / 4.06 & 1.403 / 4.16 & 0.7925 / 9.98\\
Wine Quality & 0.1222 / 4.60 & 0.1532 / 4.31 & 0.2836 / 9.95\\
Yacht & 2.669 / 5.29 & 3.723 / 5.12 & 0.002423 / 9.79\\
\bottomrule\end{tabular}\end{table}

The SCIP comparison covers all twenty primary masks. ZORRO refines valid model envelopes throughout its complete ten-second budget, retaining each parent's bound until child evidence is complete. Table~\ref{tab:zorro-refined-v91} and Figure~\ref{fig:zorro-time-v91} give its widths, costs and anytime trajectories on the four first masks.

\paragraph{Additional coarsening and directions.}
All sixty full-feature conditions in Diabetes, Concrete and Superconductivity~\citep{efron2004,sklearn,uci_concrete,uci_super} reach $3\times10^{-6}$ width, with independent exact replay of every certificate. Complete costs are at most 156.76 seconds (Appendix~\ref{app:general-public}). Eighteen further masks compare rank-three gradient/POD/greedy. Gradient is narrower in 16/18 and 17/18 metric-matched pairs (Table~\ref{tab:metric-v79}); twelve further SCIP controls and their objective identity audit are in Appendix~\ref{app:scip-followup}.

\subsection{Core Energy and Measurement Decisions}\label{sec:core-decisions-v91}
We freeze three mutually disjoint 1024-record Superconductivity evaluation pools, excluding the named development and scaling pools, and five masks per pool before running their solvers. One development predictor with the fixed-point contract is stored and shared across all fifteen tables. The coarsening uses the primary signal/nuisance rule. The predeclared criterion allows uncertainty to add at most 2\% to the nominal midpoint regret:
\begin{equation}
 H^*-h_\beta(X_c)\le0.02h_\beta(X_c).\label{eq:core-decision-v91}
\end{equation}
The midpoint uses only the available intervals and responses. Its verified lower endpoint defines a safe STOP threshold; its upper defines a violation threshold. Thus the rule is fixed independently of the searched worst-regret interval. Each method has ten complete seconds. The fixed acquisition policy reveals the thirty-three uncertain cells unless STOP is certified.

Pure contraction energy and adaptive energy each certify 13/15 STOPs, compared with 5/15 for centered moments and 0/15 for the common cap, SCIP and fully refined ZORRO. Energy's pool counts are $(5,5,3)$ and moment's are $(0,5,0)$. Summed across the fifteen audit instances, the policy avoids 429 cells for energy and 165 for moment, an additional 264 cells at ten seconds. Energy's median first-STOP time among its thirteen successes is 3.10 seconds. The pure and adaptive decisions coincide, identifying the structural energy bound's contribution. All ninety method-condition records and the predeclared 1\%/5\% sensitivity rules are in Appendix~\ref{app:core-decisions-v91}. A separate general-moment acquisition study appears in Appendix~\ref{app:calibration-v81}.

Joint residual refinement on the same stored instances certifies all fifteen STOP decisions within 91.68 seconds, avoiding all 495 cells under the fixed policy (Table~\ref{tab:decisions-refined-v92}). At the common 120-second checkpoint, moments certify five STOPs and refinement certifies fifteen, avoiding 330 additional cells (Figure~\ref{fig:public-energy-v84}B). The original 2\% threshold and candidate remain fixed.

\paragraph{Verification.}
Independent rational audits check original-row queries, representative complete trees and feasible table witnesses. The input-space baseline and every timing boundary are specified in Appendix~\ref{app:input} and Table~\ref{tab:protocol-v80}.

\section{Related Work and Scope}
Table~\ref{tab:closest-main-v81} summarizes the closest optimization and certification guarantees. ZORRO~\citep[Theorem 4.2]{zhu2024zorro} encloses all exact ridge optima under training-feature/label uncertainty. Its model/data zonotope can be converted to a regret upper through $h_\beta(X)=g_X^\top(\beta-\theta_X^*)$. Our direct objective bounds add an original-table lower witness and a verified residual-structure compression error; Table~\ref{tab:zorro-v90} measures the explicit model-envelope conversion. MinPrep studies cleaning necessity~\citep{minprep2026}; dataset multiplicity analyzes predictions under possible training datasets~\citep{meyer2023}. Poisoning certificates enclose perturbed training trajectories~\citep{sosnin2025,lorenz2025mibp}; interval least-squares PPS encloses solution coordinates~\citep{shary2021}. Classical widths~\citep{binev2011} and task-oriented/sketched approximation~\citep{chen2025coupled,patil2024sketched} give complementary dimension notions.

\clearpage
\typeout{CONTENT_PAGES:\number\numexpr\value{page}-1\relax}
\section*{Impact Statement}
This work provides deterministic bounds on training-objective regret under explicitly specified feature intervals. It can help practitioners decide when further feature acquisition is useful. Application-specific interval validity and downstream predictive or clinical utility require their own assessment. Experiments use public benchmark records with artificial coarsening and controlled synthetic tables. LLM tools assisted coding and manuscript preparation, including discussion and presentation of mathematical arguments. The paper provides proofs and executable audit programs for its mathematical and computational claims. The author is responsible for the content.

\clearpage
\appendix
\section{Observed-Anchor Residual-Energy Certification}\label{app:anchor-energy-v84}
This section supplies the extension that connects the energy mechanism to the measured public tables.
Write $V=W_{\mathcal A^c}^{1/2}X_{\mathcal A^c}T$, $e=W_{\mathcal A^c}^{1/2}(X_{\mathcal A^c}\beta-y_{\mathcal A^c})$. The original gain in transformed displacement $z$ has Hessian $K+V^\top V$ and linear term $h+V^\top e$. A retained translation $z=z_0+Qu+Pv$ changes the linear term to $h_c+V^\top e_c$. Since $z_0$ belongs to the retained span, the original restricted-comparator class is unchanged.

Let $A=V_{\mathcal N}Q$, $Z=V_{\mathcal N}P$ and $S=V_{\mathcal S}Q$. Signal support removes $V_{\mathcal S}P$. The discarded block is $D=K_{PP}+Z^\top Z\succeq\mu I$. Therefore
\begin{align*}
 ZD^{-1}Z^\top&\preceq I,\\
 D^{-1}Z^\top ZD^{-1}&\preceq D^{-1}-\mu D^{-2}\preceq I/(4\mu).
\end{align*}
The last inequality follows by diagonalizing $D$: $t^{-1}-\mu t^{-2}\le1/(4\mu)$ for every $t\ge\mu$. Thus $\|ZD^{-1}\|\le1/(2\sqrt\mu)$. Also $\|K_{QP}\|\le c=(\nu-\mu)/2$, because subtracting the midpoint scalar matrix leaves the off-diagonal block unchanged.

The Schur complement of $K+V_{\mathcal N}^\top V_{\mathcal N}$ is at least $\mu I$: its quadratic form at $u$ is the minimum over $v$ of a form bounded below by $\mu(\|u\|^2+\|v\|^2)$. Adding $S^\top S$ preserves the bound. Its retained right-hand side is
\begin{align*}
 g_S&+A^\top(I-ZD^{-1}Z^\top)e_N-K_{QP}D^{-1}Z^\top e_N\\
 &\quad-(K_{QP}+A^\top Z)D^{-1}(h_c)_P,
\end{align*}
where $g_S=Q^\top h_c+S^\top(e_c)_{\mathcal S}$ and $e_N=(e_c)_{\mathcal N}$. The displayed contractions and $\|D^{-1}\|\le1/\mu$ bound its norm by the numerator of $R_A$. Every full optimizer therefore has $\|u\|\le R_A$. The retained-only optimizer has right-hand norm at most $G+a\epsilon$ and satisfies the same radius.

At fixed retained $u$, optimizing the discarded coordinate adds
\begin{align*}
 &\|D^{-1/2}[(h_c)_P-K_{PQ}u+Z^\top(e_N-Au)]\|^2\\
 &\qquad\le[\epsilon+aR_A+(d+cR_A)/\sqrt\mu]^2.
\end{align*}
Apply this at each completion's full optimum to bound the pointwise best-subspace error, then take the supremum over tables. This proves Corollary~\ref{cor:anchor-energy-v84}, including its intrinsic-dimension statement and its shifted containing domain.

\subsection{A Sharper Contraction Certificate}\label{app:contraction-v88}
The observed-anchor bound admits a finite-scale refinement that uses the actual discarded row energy while retaining its scale-independent limit. Use the notation of Corollary~\ref{cor:anchor-energy-v84}, with $A=V_{\mathcal N}Q$, $Z=V_{\mathcal N}P$ and $e=(e_c)_{\mathcal N}$. Verify uniform bounds
\[
 \|Z\|^2\le\rho,\qquad \|A^\top e\|\le b.
\]
Define
\begin{align*}
 \gamma^2&=\frac{\rho}{\mu+\rho},\\
 t&=\min\left\{\frac1{2\sqrt\mu},\frac{\sqrt\rho}{\mu}\right\},\\
 q&=\min\{a\epsilon,b+a\gamma^2\epsilon\},\\
 N_C&=G+q+(c/\mu+at)d+ct\epsilon,\\
 R_C&=\min\{R_A,N_C/\mu\}.
\end{align*}
An independently verified bound $\|h_c+V^\top e_c\|\le B$ additionally permits $R=\min\{R_C,B/\mu\}$. Both the full optimizer and the retained optimizer have retained norm at most $R$. Put
\begin{align}
 E_N(R)&=\min\{(\epsilon+aR)^2,\epsilon^2+2bR+a^2R^2\},\nonumber\\
 \eta_C&=\left[\gamma\sqrt{E_N(R)}+\frac{d+cR}{\sqrt\mu}\right]^2.\label{eq:contraction-v88}
\end{align}
Then $E(\operatorname{range}(TQ))\le\eta_C$ and the same containing search domain gives $V_Q\le H^*\le V_Q+\eta_C$. The factor $\gamma\le1$ adapts to the observed nuisance energy; its limiting value recovers the original contraction guarantee.

\paragraph{Proof.}
Since $D=K_{PP}+Z^\top Z\succeq\mu I+Z^\top Z$,
\[
 ZD^{-1}Z^\top\preceq Z(\mu I+Z^\top Z)^{-1}Z^\top\preceq\gamma^2I.
\]
Consequently $\|D^{-1/2}Z^\top\|\le\gamma$ and $\|ZD^{-1}\|\le t$. In the retained Schur right-hand side,
\[
 \|A^\top(I-ZD^{-1}Z^\top)e\|\le\min\{a\epsilon,b+a\gamma^2\epsilon\}.
\]
The remaining two terms are bounded by $ct\epsilon$ and $(c/\mu+at)d$. The Schur coercivity bound gives $R_C$. Full coercivity gives $B/\mu$; the retained-only optimum obeys both corresponding bounds. Finally,
\[
 \|e-Au\|^2\le\epsilon^2+2b\|u\|+a^2\|u\|^2
\]
and the triangle inequality gives the other term in $E_N$. Applying the refined contraction to the discarded gain proves~\eqref{eq:contraction-v88}.

\paragraph{Verified row-space implementation.}
For the actual rational projector, each discarded row norm has its maximum at an original-box vertex. Their sum supplies $\rho_{\rm tr}$. A second bound uses $Z=Z_0+\Delta Z$:
\begin{align*}
 \rho_{\rm op}&=\bigl(\sqrt{\bar\lambda(Z_0Z_0^\top)}+\sqrt{\overline{\|\Delta Z\|_F^2}}\bigr)^2,\\
 \rho&=\min\{\rho_{\rm tr},\rho_{\rm op}\}.
\end{align*}
The center Gram has only as many rows as the nuisance set. A rounded eigenvector test matrix, its actual Gram lower bound and outward Gershgorin arithmetic verify $\bar\lambda$. Original-box vertex norms bound the perturbation trace. Exact quadratic product ranges are summed before bounding $\|A^\top e\|$ in the actual retained Gram. This preserves cancellation across rows in the mixed moment. Upward square roots, a dyadic containing radius and the existing exact row oracle complete the certificate. All costs enter the energy method's complete clock.

\paragraph{Common-domain adaptive certificate.}
At the retained reference $z_0$, let $B\ge\sup_X\|\widehat g_{X,\rm ref}\|$. Strong convexity bounds both full and restricted displacement norms by $B/\mu$. Together with the structural radius $R_C$, this gives the containing coordinate radius $R_c=\lceil\min\{R_C,B/\mu\}/\sqrt{\mu_H}\rceil_{2^{-24}}$. The ordinary centered-moment radius is $R_M=\lceil B/(\mu\sqrt{\mu_H})\rceil_{2^{-24}}$, so $R_c\le R_M$. With the same actual projector and reference, Schur elimination bounds the discarded gain by $(G_P+C_RR_c)^2/\mu$. The contraction bound is also valid on the same original retained objective and containing root. Taking their minimum proves~\eqref{eq:adaptive-v90}, and its sufficient tail is at most $(G_P+C_RR_M)^2/\mu$. The dyadic containing box can include extra retained comparators; each remains feasible for the full problem. Independently verified table witnesses and the common cap tighten final endpoints. The tail comparison concerns the shared domain; completed search and witness quality determine final widths.

\subsection{Constructive discovery and actual arithmetic}
For each signal row $x_i=c_i+\sum_j[-1,1]s_{ij}e_j$, its transformed vector lies in the span of $T^\top c_i$ and the $k_i$ vectors $s_{ij}T^\top e_j$. Adjoin $z_A=T^{-1}(\beta-\theta_A)$ with $\theta_A=B^{-1}r_A$. Then $h-Kz_A=0$, and the dimension is at most $1+\sum_i(1+k_i)$. This gives a deterministic construction from a chosen signal set, with no spectral-gap assumption. The residual and retained-coupling constants decide whether its rank suffices.

The implementation uses the actual dyadic $T$ rounded at $2^{-20}$. For one signal row and one unknown, its first two proposal columns are exact power-of-two scalings of the transformed center and generator. It appends the observed-anchor correction orthogonal to that span, rounded at $2^{-24}$. With the actual proposal $R$, compute $H=R^\top R$, its exact rational inverse, and the projector $RH^{-1}R^\top$. The signal's two discarded squared norms are checked to equal zero exactly. The retained reference solves
\[
 s_0=(R^\top KR)^{-1}R^\top h.
\]
Its reference coefficients are $\beta_c=\beta-TRs_0$, and $R^\top h_c=0$ is checked exactly. The remaining discarded-gradient norm $d$ is bounded from these actual rational quantities. The verified Gram minimum is the larger of its Gershgorin lower bound and $1/\operatorname{tr}(H^{-1})$.

Nuisance residual energy uses exact affine row ranges at $\beta_c$. Retained energy uses the smaller of a rowwise trace bound and a verified three-dimensional operator bound. The latter whitens the actual retained Gram by a dyadic test matrix and encloses every moment entry from original row boxes before a Gershgorin upper bound. The signal-gradient squared norm is a univariate quartic; dyadic Bernstein subdivisions give a rational upper bound, combined with its coordinate-box bound. All square roots round upward. The search domain contains $s_0$ plus the certified coordinate-radius cube. Node radii enclose both endpoint distances from the actual evaluated float center, accounting for translated-box rounding.

\section{Residual-Energy Proof and Exact Separation}\label{app:energy}
\begin{table*}[t]
\caption{Three error budgets for the same original full-comparator regret. Each follows from eliminating discarded coordinates; the applicable conditions determine the certificate used. Search error is added separately.}
\label{tab:tails}
\centering\small\setlength{\tabcolsep}{4pt}
\begin{tabular}{p{.12\textwidth}p{.31\textwidth}p{.22\textwidth}p{.27\textwidth}}
\toprule
Tail & Verified sufficient condition & Additive error & Role and evidence \\
\midrule
Residual energy & Signal residual and coupling; exact support or verified leakage $b$ & $(\epsilon+aR)^2$; leakage $\eta_b$ & Full-rank structural separation; nonzero residual and support-leakage series \\
Residual moment & Actual metric; aggregate gradient and cross-block bounds & $(G_r+C_rR_s)^2/\mu$ & General real-table certificate; fixed-rank direction controls and acquisition decisions \\
Spectral & Uniform discarded, whitened covariance bounded by $\rho$ & $\rho B_0^2/(1+\rho)$ & Sharp benchmark and covariance-based regime; exact sharpness and high-rank separation \\
\bottomrule
\end{tabular}
\end{table*}

\subsection{Original-data conditions and anchor extension}
Theorem~\ref{thm:energy} assumes $\lambda>0$, nonnegative fixed weights, a compact nonempty completion domain, $Q^\top Q=I$ and $\beta\in\operatorname{span}(Q)$. Its signal-support condition holds for every completion, and the constants bound the stated supremums over the same domain. The box version needs the tail at the full optimum's projected point, whose norm is at most $R$; a box can contain additional points outside that ball, all of which remain feasible comparators.

For a quadratic observed-row anchor, first remove its rows $\mathcal A$ and partition the remaining rows into $\mathcal S,\mathcal N$. With original weights and ridge counted once, put $B=\lambda I+\sum_{i\in\mathcal A}w_ix_ix_i^\top\succ0$, $r_A=\sum_{i\in\mathcal A}w_ix_iy_i$ and $g_A=B\beta-r_A$. The same proof applies when $P^\top BQ=0$ and $P^\top g_A=0$. Replace retained ridge terms in $\kappa$ and $f_{\mathcal S}$ by $Q^\top BQ$ and $Q^\top g_A$. Then the discarded contraction is $Z(B_D+Z^\top Z)^{-1}Z^\top\preceq I$, with $B_D=P^\top BP\succ0$. The general moment theorem covers arbitrary actual dense metrics and retains their gradient and cross terms.

In the ridge-only variant, with a separately verified retained radius, a nonzero discarded candidate component can be represented by ridge pseudo-observations with residual $\sqrt\lambda P^\top\beta$. The energy tail becomes
\[
 \left(\sqrt{\epsilon^2+\lambda\|P^\top\beta\|^2}+aR\right)^2.
\]
The sharper original-data radius in the main theorem is not automatically preserved by this extension.

\subsection{Moments, exact optima and genuine nuisance influence}
For~\eqref{eq:family}, let $d=(2m)^{-1}\sum_i(u_i-1)^2$. Direct original moments give
\begin{align*}
 g_X&=de_0,\quad 0\leq d\leq1/32,\\
 M_{00}&=\tfrac12+(2m)^{-1}\sum_i u_i^2+\alpha^2/2,\\
 M_{0j}&=\alpha t_j/(2m),\qquad M_{jj}=\tfrac12+t_j^2/(2m),
\end{align*}
with every other nuisance off-diagonal zero. The Schur complement is
\[
 K=\tfrac12+(2m)^{-1}\sum_i u_i^2
       +\frac{\alpha^2}{2m}\sum_j\frac1{1+t_j^2/m}.
\]
Thus $h_\beta(X)=d^2/K$, while the retained value is $d^2/M_{00}$. Simultaneously $d^2\leq1/1024$, $K\geq25/32+\alpha^2/10$ and $M_{00}\geq25/32+\alpha^2/2$. All bounds are attained by $u_i=3/4$, and $K$ additionally takes $t_j=2\sqrt m$, proving~\eqref{eq:family_optima}.

At the midpoint $u_i=1$, so $g_X=0$ and $\beta=e_0$ is the unique full-feature ridge optimum. One signal row and all nuisance rows form a submatrix with determinant $u_i\prod_jt_j\ne0$. Every completed table has full column rank. Moreover
\[
 \frac{\partial K}{\partial t_j}
 =-\frac{\alpha^2t_j}{m^2(1+t_j^2/m)^2}<0
\]
for $\alpha>0$. Consequently each uncertain nuisance coordinate strictly affects regret whenever $d>0$. Its individual endpoint effect scales as $1/m$; the joint effect of moving all nuisance endpoints is
\[
 \frac5{4000+512\alpha^2}-\frac1{800+256\alpha^2}>0,
\]
independently of $m$. The nuisance row norms grow with $\sqrt m$ to keep feature energies bounded below; this scaling is part of the construction.

At $\alpha=1/32$, $H^*=10/8001$, $V_{e_0}=2/1601$, omitted regret $8/12809601$, and energy tail $1/1280000$. The full retained displacement is at most $1/25$. With $T=I$, the general moment choices $G_r=0,C_r=\alpha,\mu=1/2,U_0=1/512$ give tail $\alpha^2/128$. The residual-energy result uses both the tighter radius and the contraction by the actual discarded ridge-plus-data block. Numerical generic envelopes may be looser still; their actual verified values are reported.

\subsection{Recovery and covariance comparisons}
The first-order signal generator is zero because its midpoint residual is $-1$ and $\beta_0=1$. Its second-order generator is a nonzero multiple of $T^\top e_0$. On nuisance rows, residuals and the uncertain-coordinate candidate coefficients are zero. The proposal Gram therefore has rank one. For the scalar actual dyadic ridge-anchor transformation, its retained span maps exactly to $e_0$; the implementation verifies this actual matrix identity.

Endpoint averaging gives the feature second moment
\[
 \overline C=\begin{bmatrix}
 17/32+\alpha^2/2&b^\top\\ b&(5/4)I_m
 \end{bmatrix},\quad b_j=3\alpha/(4\sqrt m).
\]
The nuisance block has $m$ eigenvalues $5/4$. Interlacing gives at least $m$ eigenvalues of $\overline C$ at least $5/4$. Ridge whitening multiplies by 2, so the endpoint-average min-max bound implies $\rho\geq5/2$ at ranks $r\leq m-1$. Since $\theta_0=0$, $B_0^2=2+\alpha^2/2$, and the spectral tail ledger is at least $(5/7)(2+\alpha^2/2)\geq10/7$.

The discounted upper itself has a stronger obstruction. With $\gamma\leq2/7$, evaluate it at retained coordinate zero and $u_i=3/4$:
\[
 U_\gamma\geq41/32-\gamma(2+\alpha^2/2)
 \geq159/224-\alpha^2/7.
\]
At $\alpha=1/32$ this is $5087/7168>0.709$, while $H^*\approx0.00125$. Independent original-objective caps can tighten final software output; this comparison isolates the spectral-discount component.

The midpoint second moment is
\[
 C_c=\begin{bmatrix}1/2+\alpha^2/2&b^\top\\b&(9/8)I_m\end{bmatrix}.
\]
For $0<\alpha<\sqrt{5/4}$ put $\Delta=5/8-\alpha^2/2$. Its leading eigenvector's active component satisfies
\[
 q_{+,0}^2=\tfrac12\left(1-\frac\Delta{\sqrt{\Delta^2+4\|b\|^2}}\right)
 \leq\frac{9\alpha^2/16}{\Delta^2}.
\]
All remaining eigenvectors above the smallest eigenvalue are pure nuisance vectors. For any leading Gram-PCA span of rank $1\leq r\leq m$, $\|Q_c^\top e_0\|^2=q_{+,0}^2$. Using $g_X=de_0$ and $M_X\succeq I/2$,
\[
 V_{\rm PCA}\leq\frac1{512}\frac{9\alpha^2/16}{(5/8-\alpha^2/2)^2}.
\]
This equals $9/3271682$ at $\alpha=1/32$. The construction uses uncentered weighted feature second moments, matching the midpoint-Gram proposal. It makes no statement for every possible covariance-based algorithm or for sample-mean-centered PCA.

\begin{table*}[t]\caption{Complete structural mechanism series. All $p$ columns have independent nondegenerate intervals, and every positive-coupling nuisance coordinate affects the original regret. Widths and energy tails use $10^{-6}$ units. Both solvers search one learned direction and receive 10-second search caps; times include common construction/anchor costs and their own setup/search.}\label{tab:structure}\centering\small\setlength{\tabcolsep}{4pt}\begin{tabular}{rrrrrrrr}\toprule $p$&Cells&$\alpha$&Moment width&Energy width&Energy tail&Moment (s)&Energy (s)\\\midrule
17&32&0&0.535&0.501&0.000&0.09&0.07\\
17&32&1/64&5.685&0.540&0.195&0.15&0.07\\
17&32&1/32&22.740&0.657&0.781&0.14&0.07\\
17&32&1/16&90.968&0.750&3.125&0.14&0.07\\
65&128&0&0.535&0.501&0.000&0.75&0.72\\
65&128&1/64&14.268&0.540&0.195&0.77&0.66\\
65&128&1/32&57.073&0.657&0.781&0.75&0.65\\
65&128&1/16&228.297&0.750&3.125&0.88&0.78\\
257&512&0&0.535&0.501&0.000&34.76&35.30\\
257&512&1/64&42.878&0.540&0.195&43.37&33.97\\
257&512&1/32&171.514&0.657&0.781&39.57&29.82\\
257&512&1/16&686.060&0.750&3.125&40.36&31.31\\\bottomrule\end{tabular}\end{table*}

\subsection{Curvature-aware controls at fixed ambient dimension}\label{app:scale-v80}
Fix $n=128$, $p=65$, $\alpha=1/32$, $\lambda=1/2$ and $\beta=Q=e_0$ in~\eqref{eq:family}. Scale only nuisance features, giving $t_j\in[s\sqrt m,2s\sqrt m]$. The structural constants are $a=\alpha/\sqrt2$ and $R=1/25$; the exact cross-block bound is $\sup_X\|Z_X^\top A_X\|=\alpha s$.

The ridge-only scalar estimate is $\eta_\lambda=2\alpha^2s^2/625$. This family also admits a stronger curvature estimate directly from its original intervals:
\[
 D_X=\lambda I+Z_X^\top Z_X\succeq (1+s^2)I/2.
\]
At the same structural radius, the resulting bound is
\[
 \eta_D=\frac{2\alpha^2s^2}{625(1+s^2)},\quad
 \eta_E=\frac{\alpha^2}{1250},\quad
 \frac{\eta_D}{\eta_E}=\frac{4s^2}{1+s^2}<4.
\]
Thus stronger scalar curvature also removes scale divergence in this diagonal family. The energy certificate gives the smaller constant without estimating a nuisance curvature lower bound.

\paragraph{Moment control with an actual preconditioner.}
For the stated power-of-two scales, $T_s=\operatorname{diag}(1,s^{-1}I)$ is exactly dyadic and retains the same physical line $e_0$. In its actual transformed full Gram, retained curvature is at least $A=25/32$, discarded curvature at least $d_s=(1+s^{-2})/2$, and cross-block norm at most $\alpha$. Therefore
\[
 \widehat M_X\succeq\mu_s I,\quad
 \mu_s=\frac{A+d_s-\overline{\sqrt{(A-d_s)^2+4\alpha^2}}}{2}>0,
\]
where the square root is rounded upward on a $2^{-64}$ grid. Exact rational checks verify $A-\mu_s\ge0$, $d_s-\mu_s\ge0$ and $(A-\mu_s)(d_s-\mu_s)\ge\alpha^2$. These inequalities prove the bound uniformly over all original tables. The transformed discarded gradient is zero and $C_r=\alpha$, so $U_0=1/(1024\mu_s)$, $R_s=1/(32\mu_s)$ and Theorem~\ref{thm:moment} yields
\[
 \eta_{M,T}=\frac{\alpha^2}{1024\mu_s^3}.
\]
This is a curvature-aware, actually preconditioned moment control. It also stays bounded as $s$ grows. For reference, the identity-metric moment constants give $\alpha^2s^2/128$; that simpler control is recorded in the source.

The exact full and restricted optima are
\[
 \begin{gathered}
 H_s^*=\frac1{800+512\alpha^2/(1+4s^2)},\\
 V_{e_0}=\frac1{800+512\alpha^2}.
 \end{gathered}
\]
Original normal equations attain both at the signal lower endpoints and nuisance upper endpoints. The comparison keeps dimension, direction and candidate fixed. Energy's scale-independent guarantee applies without the diagonal family's available nuisance curvature estimate; the diagonal case quantifies the remaining constant advantage after supplying it.

\begin{table}[t]\centering\small\setlength{\tabcolsep}{3pt}
\caption{Additive controls at fixed $p=65$, in $10^{-6}$ units. D uses the enhanced discarded curvature; MT uses the actual dyadic preconditioner and verified full-Gram lower bound.}\label{tab:scale-v80}
\begin{tabular}{rrrrr}\toprule Scale&Actual&Energy&D&MT\\\midrule
1 & 0.625 & 0.781 & 1.562 & 2.034\\
2 & 0.735 & 0.781 & 2.500 & 4.021\\
4 & 0.769 & 0.781 & 2.941 & 6.501\\
16 & 0.780 & 0.781 & 3.113 & 7.699\\
64 & 0.781 & 0.781 & 3.124 & 7.783\\
256 & 0.781 & 0.781 & 3.125 & 7.788\\
\bottomrule\end{tabular}\end{table}

\subsection{Stable residual energy under signal-support leakage}
\label{app:leakage-v79}
The exact support condition in Theorem~\ref{thm:energy} admits a quantitative extension. Its radius controls the projection of the full optimizer for every completion and converges to the residual-energy radius as leakage vanishes.

Let $[Q\ P]$ be orthogonal and write
$\beta=Q\beta_Q+P\beta_D$ with $\|\beta_D\|\le d$.
For $j\in\{\mathcal S,\mathcal N\}$ define
\begin{align*}
 A_j&=W_j^{1/2}X_jQ,\qquad Z_j=W_j^{1/2}X_jP,\\
 e_j&=W_j^{1/2}(X_j\beta-y_j).
\end{align*}
Suppose uniformly over the compact nonempty completion domain that
\begin{align*}
 \|A_{\mathcal N}\|_{\op}&\le a,&\|e_{\mathcal N}\|&\le\epsilon,\\
 \|A_{\mathcal S}\|_{\op}&\le\sigma,&\|e_{\mathcal S}\|&\le E,
 &\|Z_{\mathcal S}\|_{\op}&\le b,\\
 \lambda I+A_{\mathcal S}^\top A_{\mathcal S}&\succeq\kappa I,&
 \|f_{\mathcal S}\|&\le G,
\end{align*}
where $f_{\mathcal S}=\lambda\beta_Q+A_{\mathcal S}^\top e_{\mathcal S}$,
$\lambda,\kappa>0$, and weights are fixed and nonnegative.
The residuals use the actual candidate. In particular, when $d>0$ the nuisance-residual bound includes $Z_{\mathcal N}\beta_D$.

\begin{theorem}[Residual-energy stability under leakage]
\label{thm:energy-leakage-v79}
Define
\begin{align}
 \omega_b&=\frac{b}{\sqrt{\lambda+b^2}},\nonumber\\
 \kappa_b&=\frac{\lambda(\kappa+b^2)}{\lambda+b^2},
 \label{eq:leakage-contractions-v79}\\
 R_{b,d}&=\frac{1}{\kappa_b}\left[
 G+a\epsilon+\frac{\omega_b}{2}(aE+\sigma\epsilon)
 +\omega_b^2\sigma E\right.\nonumber\\[-1mm]
 &\hspace{30mm}\left.{}+(a\sqrt\lambda+\sigma b)d\right],
 \label{eq:leakage-radius-v79}\\
 \eta_{b,d}&=\left[\epsilon+aR_{b,d}
 +\omega_b(E+\sigma R_{b,d})+\sqrt\lambda d\right]^2.
 \label{eq:leakage-tail-v79}
\end{align}
For every completion $X$, its full maximizing displacement
$v_X=M_X^{-1}g_X$ satisfies $\|Q^\top v_X\|\le R_{b,d}$.
For either the radius-$R_{b,d}$ Euclidean ball or its containing coordinate cube $\mathcal B$, put
\[
 V_Q=\max_{s\in\mathcal B}\max_X
 \{2g_X^\top Qs-s^\top Q^\top M_XQs\}.
\]
Then
\[
 V_Q\le H^*\le V_Q+\eta_{b,d}.
\]
Moreover, $E(\operatorname{range}(Q))\le\eta_{b,d}$ for the uniform pointwise error in Appendix~\ref{app:dimension-v78}.
The constants have no dependence on the size or rank of $Z_{\mathcal N}$.
At $b=d=0$, they recover $R=(G+a\epsilon)/\kappa$ and $\eta=(\epsilon+aR)^2$ exactly.
\end{theorem}

\begin{proof}
Fix $X$ and suppress its subscript. The normal matrix in the $[Q\ P]$ coordinates has blocks
\begin{align*}
 B&=\lambda I+A_{\mathcal S}^\top A_{\mathcal S}
                  +A_{\mathcal N}^\top A_{\mathcal N},\\
 C&=A_{\mathcal S}^\top Z_{\mathcal S}
                  +A_{\mathcal N}^\top Z_{\mathcal N},\\
 D&=\lambda I+Z_{\mathcal S}^\top Z_{\mathcal S}
                  +Z_{\mathcal N}^\top Z_{\mathcal N}.
\end{align*}
Write $K=B-CD^{-1}C^\top$. For every $s$, the Schur variational identity gives
\begin{align*}
 s^\top Ks=\min_t\bigl\{&\lambda\|s\|^2+\lambda\|t\|^2
 +\|A_{\mathcal S}s+Z_{\mathcal S}t\|^2\\
 &+\|A_{\mathcal N}s+Z_{\mathcal N}t\|^2\bigr\}.
\end{align*}
Dropping the last nonnegative term and eliminating $t$ yields
\begin{align*}
 K&\succeq\lambda I+
 A_{\mathcal S}^\top
 (I+Z_{\mathcal S}Z_{\mathcal S}^\top/\lambda)^{-1}A_{\mathcal S}\\
 &\succeq\lambda I+\frac{\lambda}{\lambda+b^2}
 A_{\mathcal S}^\top A_{\mathcal S}
 \succeq\kappa_b I.
\end{align*}
This lower bound is positive for all $b\ge0$.

The discarded block supplies the contractions
\begin{align}
 0&\preceq Z_{\mathcal N}D^{-1}Z_{\mathcal N}^\top\preceq I,
 \nonumber\\
 \|Z_{\mathcal S}D^{-1/2}\|_{\op}&\le\omega_b,\nonumber\\
 \|Z_{\mathcal N}D^{-1}Z_{\mathcal S}^\top\|_{\op}&\le\omega_b/2.
 \label{eq:leakage-block-v79}
\end{align}
For the cross bound, set $D_0=\lambda I+Z_{\mathcal S}^\top Z_{\mathcal S}$ and $T=Z_{\mathcal N}D_0^{-1/2}$. Then
\[
 Z_{\mathcal N}D^{-1}Z_{\mathcal S}^\top
 =T(I+T^\top T)^{-1}D_0^{-1/2}Z_{\mathcal S}^\top.
\]
The first factor has norm at most $1/2$ by its singular values $z/(1+z^2)$; the last factor has norm at most $\omega_b$. The remaining contractions follow from $D\succeq D_0$ and $D\succeq\lambda I+Z_{\mathcal N}^\top Z_{\mathcal N}$.

The retained coordinate $s_X=Q^\top v_X$ solves
\[
 Ks_X=g_Q-CD^{-1}g_D,
\]
where $g_Q=f_{\mathcal S}+A_{\mathcal N}^\top e_{\mathcal N}$ and
$g_D=Z_{\mathcal N}^\top e_{\mathcal N}+Z_{\mathcal S}^\top e_{\mathcal S}+\lambda\beta_D$.
Expanding its right-hand side gives
\begin{align*}
 f_{\mathcal S}
 &+A_{\mathcal N}^\top(I-Z_{\mathcal N}D^{-1}Z_{\mathcal N}^\top)e_{\mathcal N}\\
 &-A_{\mathcal N}^\top Z_{\mathcal N}D^{-1}Z_{\mathcal S}^\top e_{\mathcal S}\\
 &-A_{\mathcal S}^\top Z_{\mathcal S}D^{-1}Z_{\mathcal N}^\top e_{\mathcal N}\\
 &-A_{\mathcal S}^\top Z_{\mathcal S}D^{-1}Z_{\mathcal S}^\top e_{\mathcal S}
 -\lambda CD^{-1}\beta_D.
\end{align*}
Equation~\eqref{eq:leakage-block-v79} bounds the first four lines by
$G+a\epsilon+\omega_b(aE+\sigma\epsilon)/2+\omega_b^2\sigma E$.
Also
\[
 \lambda\|CD^{-1}\beta_D\|
 \le(a\sqrt\lambda+\sigma b)d,
\]
using $\|Z_{\mathcal N}D^{-1}\|_{\op}\le1/\sqrt\lambda$ and
$\|Z_{\mathcal S}D^{-1}\|_{\op}\le b/\lambda$.
Together with $K\succeq\kappa_bI$, these bounds prove
$\|s_X\|\le R_{b,d}$ for every completion.

At fixed $X,s$, eliminating the discarded displacement adds $q(s)^\top D^{-1}q(s)$, where
\[
 q(s)=Z_{\mathcal N}^\top(e_{\mathcal N}-A_{\mathcal N}s)
 +Z_{\mathcal S}^\top(e_{\mathcal S}-A_{\mathcal S}s)+\lambda\beta_D.
\]
The contractions give
\[
 \|D^{-1/2}q(s)\|
 \le\epsilon+a\|s\|+\omega_b(E+\sigma\|s\|)+\sqrt\lambda d.
\]
At $s=s_X$, this gain is at most $\eta_{b,d}$ and $s_X\in\mathcal B$.
Every retained displacement is a legal full displacement. Taking maxima proves the interval, and comparison at each fixed $X$ with its restricted optimum proves the uniform pointwise error bound.
\end{proof}

For fixed values of the other constants, $R_{b,d}$ and $\eta_{b,d}$ depend continuously on $(b,d)$ at zero. This extension preserves the discarded-block contraction responsible for the original residual-energy guarantee.

\paragraph{An exact rational signal-leakage specialization.}
Keep the construction in~\eqref{eq:family} with $\beta=Q=e_0$, and replace every signal row by
\[
 x_i=u_i e_0+\delta e_1,\qquad u_i\in[3/4,5/4],\quad y_i=2,
\]
where $\delta\ge0$. The nuisance rows, independent intervals, weights $1/(2m)$ and penalty $\lambda=1/2$ remain as specified by the construction. The original-row bounds are
\begin{gather*}
 \epsilon=d=0,\quad a=\alpha/\sqrt2,\\
 \sigma=E=5/(4\sqrt2),\quad b=\delta/\sqrt2,
\end{gather*}
with $G=1/32$ and $\kappa=25/32$.
Indeed, $f_{\mathcal S}=\tfrac12+(2m)^{-1}\sum_i(u_i^2-2u_i)$ lies in $[0,1/32]$.
At positive $\delta$, the fixed candidate generally differs from the new midpoint ridge fit.

For rational verification, replace $\omega_b/2$ by the upper bound $b/\sqrt\lambda$, $\omega_b^2$ by $b^2/\lambda$, and the tail factor $\omega_b$ by $b/\sqrt\lambda$. This gives the conservative exact constants
\begin{align}
 \kappa_\delta&=\frac{25+16\delta^2}{32(1+\delta^2)},\nonumber\\
 \overline R_\delta&=
 \frac{(1+20\alpha\delta+25\delta^2)(1+\delta^2)}{25+16\delta^2},
 \label{eq:leakage-rational-radius-v79}\\
 \overline\eta_\delta&=
 \frac12\left[\alpha\overline R_\delta+
 \frac{5\delta}{4}(1+\overline R_\delta)\right]^2.
 \label{eq:leakage-rational-tail-v79}
\end{align}
They are rational for dyadic $\alpha,\delta$, and give
$\overline R_0=1/25$ and $\overline\eta_0=\alpha^2/1250$.
The frozen signal-leakage experiment uses these conservative constants.

\subsection{Signal-support leakage with a uniform energy certificate}\label{app:leakage-experiment-v79}
We retain the original $n=128,p=65$ mechanism with $\alpha=1/32$, $\lambda=1/2$, weights $1/128$ and fixed candidate $\beta=e_0$, and change every signal row to $x_i=u_ie_0+\delta e_1$, $u_i\in[3/4,5/4]$. Nuisance rows, responses and independent uncertainty intervals stay as declared. The five leakage levels were fixed before outcomes. At positive $\delta$ the candidate is no longer the perturbed midpoint fit: its exact midpoint gradient is $-(\delta/2)e_1$. Both methods retain the declared $e_0$ direction, so this measures certificate stability with fixed directions.

The original-row constants satisfy $\epsilon=d=0$, $a^2=\alpha^2/2$, $\sigma^2=E^2=25/32$, $b^2=\delta^2/2$, $G=1/32$, $\kappa=25/32$. We use the conservative rational curvature, radius and energy tail
\begin{align*}
 \kappa_\delta&=\frac{25+16\delta^2}{32(1+\delta^2)},\\
 \overline R_\delta&=\frac{(1+20\alpha\delta+25\delta^2)(1+\delta^2)}{25+16\delta^2},\\
 \overline\eta_\delta&=\frac12\left[\alpha\overline R_\delta+
       \frac{5\delta}{4}(1+\overline R_\delta)\right]^2.
\end{align*}
These verify a uniform projection-error bound for every completion. Energy and general-moment searches use the same actual scalar ridge-anchor transformation, full original comparator, declared direction and measured common input/anchor cost. Each receives its own two-second search cap and $10^{-6}$ final-width target.

\begin{table*}[t]\centering\small
\caption{Signal-support leakage: tails and final widths in $10^{-6}$ units; total seconds include common input/anchor setup and each method's setup/search. All five energy widths meet $10^{-6}$; general moment reaches a tail/arithmetic floor in all five.}\label{tab:leakage-experiment-v79}
\begin{tabular}{rrrrrr}\toprule $\delta$&Energy tail&Energy width&Moment width&Energy s&Moment s\\\midrule
0&0.7812&0.6569&57.07&0.550&0.871\\
$2^{-18}$&0.7875&0.6628&57.12&0.545&0.870\\
$2^{-16}$&0.8063&0.6806&57.28&0.548&0.892\\
$2^{-14}$&0.8836&0.7538&57.91&0.541&0.893\\
$2^{-12}$&1.2286&0.6993&60.43&0.573&0.884\\
\bottomrule\end{tabular}\end{table*}

At $\delta=2^{-12}$ the energy tail alone is $1.22865\times10^{-6}$, exceeding the $10^{-6}$ target. The final bracket nevertheless has width $6.99261\times10^{-7}$: full-model witness retraining strengthens the lower endpoint beyond a retained-only lower, while further retained subdivision tightens the upper. This condition is retained as a missed tail-only budget and a met final-width target. The experiment illustrates the stated neighborhood with controlled leakage; it does not establish the same small tail for arbitrary leakage magnitude.

A solver-independent Fraction program checks every actual input cell, the midpoint gradient and a uniformly positive full-rank minor, the actual dyadic transform, rational $\kappa_\delta,\overline R_\delta,\overline\eta_\delta$, and the outward root. It computes the global cap from exact original-gradient scalar-quadratic ranges coordinate by coordinate. Across all five certificates it verifies 77 retained query points, 67 tree nodes, complete binary coverage with 36 terminal leaves, and feasible original-table witness losses. The retained optimum $V_{e_0}=1/(800+512\alpha^2)$ is checked separately by joint gradient/curvature attainment. No analytic formula for the full perturbed optimum is assumed or substituted for the uniform-tail and tree proof.

\subsection{An automatic original-row energy pipeline}\label{app:auto-energy-v80}
We connect direction selection and the leakage-energy theorem in a frozen six-condition experiment with $n=32,p=17$, $\alpha=1/32$ and $\lambda=1/2$. Perturbations cross balanced nuisance-response shifts $\delta$, a fixed signal-row discarded coefficient $\ell$, and a candidate coefficient $\gamma$, with fixed candidate $\beta=e_0+\gamma e_1$. The grid is listed in Table~\ref{tab:auto-energy-v80}.

For every input the actual uncertainty-gradient proposal supplies a rounded rank-one numerical direction $R$. Its physical vector is computed exactly as $r=TR$. We use $Q=r/\sqrt{h}$ implicitly, with $h=r^\top r$, $\Pi=rr^\top/h$ and $D=I-\Pi$, preserving the actual search span. A row enters the signal partition when its exact maximum absolute candidate residual is at least half the largest row score. Thus partition labels and direction both come from the current input. The threshold is fixed across conditions.

The verifier computes retained projection ranges, actual residual energies, aggregate retained signal-gradient ranges and a positive retained-curvature lower bound directly from the original intervals. It obtains signal leakage through the endpoint maximum of $x^\top Dx$ and candidate leakage through $\beta^\top D\beta$. Each row in this pilot has one uncertain coordinate, so exact scalar-quadratic ranges verify these quantities. Rational upward square roots supply $a,\epsilon,\sigma,E,b,d,G$ for the strongest leakage-energy theorem. Its all-completions radius is converted to actual numerical coordinates with the exact check $hR_{\mathrm{coord}}^2\geq R_{\mathrm{physical}}^2$.

All six inputs yield verified energy constants and the same automatically selected partition of 16 signal and 16 nuisance rows. The rounded gradient proposals recover $\operatorname{span}(e_0)$ in all six cases. Condition 6 has actual discarded candidate norm $d=2^{-18}$ and nuisance residual-energy bound $\epsilon^2=2(2^{-18})^2$. This grid evaluates automatic recovery and verified perturbation stability at the axis-aligned mechanism; the generic projector formulas also accept non-axis directions.

\begin{table*}[t]\centering\small
\caption{Automatic energy pipeline. Tails and final widths use $10^{-6}$ units. Both methods receive the same automatically proposed direction and the full original comparator; each has a two-second search cap and a $10^{-6}$ final-width target.}\label{tab:auto-energy-v80}
\begin{tabular}{lrrrrrr}\toprule Condition&$\delta$&$\ell$&$\gamma$&Energy tail&Energy width&Moment width\\\midrule
1&0&0&0&0.7813&0.6569&22.74\\
2&$2^{-18}$&0&0&0.7860&0.6617&22.79\\
3&0&$2^{-18}$&0&0.7875&0.6625&22.77\\
4&$2^{-18}$&$2^{-18}$&0&0.7923&0.6673&22.81\\
5&$2^{-16}$&$2^{-16}$&0&0.8258&0.6988&23.01\\
6&$2^{-18}$&$2^{-18}$&$2^{-18}$&0.8019&0.6793&22.85\\
\bottomrule\end{tabular}\end{table*}

All six energy tails satisfy the budget and all six final widths meet the target. The paired general-moment certificate reaches a tail/arithmetic floor in all six, with widths $22.74$--$23.01\times10^{-6}$. Direction construction takes $0.30$--$0.34$ ms; exact partition and constant verification take $7.1$--$7.4$ ms. Complete energy times are $0.057$--$0.071$ seconds. Both branches are charged the same complete pipeline stage, including the energy-constant verification; the moment totals of $0.141$--$0.160$ seconds therefore include that common diagnostic overhead. The experiment evaluates applicability, certified widths and complete-pipeline costs. Positive perturbations retain their fixed candidates and actual residuals throughout.

An independent Fraction auditor reconstructs every original-row residual and projection bound, partition, actual dyadic span, anchor metric and outward search root. It verifies all six energy trees: 90 original-row query points, 78 nodes, 42 terminal leaves, complete binary coverage and six feasible full-model table witnesses. The auditor imports no implementation and assumes no analytic full optimum. Proposal execution is established by the runner; the independent proof audit begins from its saved actual direction. All six empirical directions span $e_0$.

\subsection{Automatic structure beyond the candidate line}\label{app:multirank-v82}
The pre-outcome factorial grid uses $r\in\{2,3\}$, $\delta\in\{0,2^{-10},2^{-6}\}$ and $\ell\in\{0,2^{-12},2^{-4}\}$. Let $m=16$, $n=2mr$, $p=r+mr$, weights $1/n$, $\lambda=1/(2r)$ and $\beta=\sum_{j=0}^{r-1}e_j$. For every $j=0,\ldots,r-1$ and $i=1,\ldots,m$, use
\begin{align*}
 x_{j,i}&=u_{j,i}e_j+\ell e_{r+jm+i-1},
 &u_{j,i}&\in[1-h_j,1+h_j],\\
 y_{j,i}&=2,&h_j&=2^{-j-2},\\
 x'_{j,i}&=\alpha e_j+t_{j,i}e_{r+jm+i-1},
 &t_{j,i}&\in[4,8],\\
 y'_{j,i}&=\alpha+\sigma_i\delta,&\alpha&=1/32,
\end{align*}
where eight signs $\sigma_i$ are positive and eight negative. Every uncertain scalar varies independently. The candidate is fixed, including under perturbations. Thus $p=34,51$ and the proposed retained ranks are two and three. All columns have nondegenerate intervals. The signal/nuisance two-by-two row-block determinant is $u_{j,i}t_{j,i}-\alpha\ell>0$, which yields a full-rank submatrix for every completion.

The existing uncertainty-gradient proposal is used without a direction override. A residual-score rule defines the row partition. The actual physical Gram/projector gives norm and gradient bounds. An interval enclosure of retained signal moments, minus $cH$, is verified diagonally dominant with nonnegative diagonal; this supplies $\kappa=\lambda+c$. All norm roots and radii are rounded outward. Energy search uses that radius and tail; same-span general moment and an explicit rank-one $\operatorname{span}(\beta)$ control use their own valid moment bounds. The candidate-line control tests the need to discover directions beyond the candidate; it has a different rank. Every method has a two-second search cap and a predeclared $10^{-5}$ final-width target. Complete costs include input, anchor, proposal, constants, method setup and search.

\paragraph{Exact unperturbed comparison.}
At $\delta=\ell=0$, put
\begin{align*}
 g_j&=h_j^2/(2r),\\
 b_j&=[1+(1-h_j)^2+\alpha^2]/(2r),\\
 k_j&=[1+(1-h_j)^2+\alpha^2/5]/(2r).
\end{align*}
The block Schur equations give
\begin{gather*}
 H^*=\sum_jg_j^2/k_j,\qquad V_{\rm active}=\sum_jg_j^2/b_j,\\
 V_\beta=(\sum_jg_j)^2/\sum_jb_j.
\end{gather*}
Signal lower endpoints simultaneously maximize each numerator and minimize each denominator; nuisance upper endpoints minimize the full Schur denominator. Distinct $h_j$ make the active optimal displacement $(g_j/b_j)_j$ nonparallel to $\beta$. At $r=2,3$, omission from the candidate line is $2.013\times10^{-4}$ and $2.257\times10^{-4}$, respectively. The learned spans have energy tails $8.301\times10^{-7}$ and $8.331\times10^{-7}$; their final widths are $6.96\times10^{-6}$ and $9.64\times10^{-6}$. These analytic formulas concern the nominal exact family. The software interprets its stored ridge coefficient as an exact binary value; the independent row-and-tree audit uses that actual value.

\begin{table*}[t]\centering\small\setlength{\tabcolsep}{5pt}
\caption{All eighteen multidimensional conditions. Widths use $10^{-6}$ units; the target is 10. Energy and moments use the learned rank-$r$ span, and the candidate control uses rank one. Block elimination verifies the sparse row-pair structure and certifies the full comparator through scalar searches.}\label{tab:multirank-v82}
\begin{tabular}{rrrrrrrr}\toprule $r$&$\delta$&$\ell$&Energy width&Moment width&Candidate-line width&Block width&Block (s)\\\midrule
2&0&0&6.96&37.6&378&6.239&0.065\\
2&0&$2^{-12}$&7.57&38.3&378&6.239&0.074\\
2&0&$2^{-4}$&509&415&509&6.239&0.068\\
2&$2^{-10}$&0&8.28&56.3&380&6.239&0.057\\
2&$2^{-10}$&$2^{-12}$&9.31&56.5&380&6.239&0.075\\
2&$2^{-10}$&$2^{-4}$&519&433&519&6.239&0.072\\
2&$2^{-6}$&0&45.4&769&769&6.239&0.060\\
2&$2^{-6}$&$2^{-12}$&52.7&769&769&6.239&0.076\\
2&$2^{-6}$&$2^{-4}$&945&945&945&6.239&0.078\\
3&0&0&9.64&49.9&253&7.206&0.084\\
3&0&$2^{-12}$&9.86&50.7&253&7.206&0.110\\
3&0&$2^{-4}$&377&377&377&7.206&0.112\\
3&$2^{-10}$&0&8.1&71.2&255&7.206&0.087\\
3&$2^{-10}$&$2^{-12}$&9.37&71.4&255&7.206&0.117\\
3&$2^{-10}$&$2^{-4}$&571&357&387&7.206&0.114\\
3&$2^{-6}$&0&812&644&644&7.206&0.088\\
3&$2^{-6}$&$2^{-12}$&812&644&644&7.206&0.119\\
3&$2^{-6}$&$2^{-4}$&988&806&803&7.206&0.126\\
\bottomrule\end{tabular}\end{table*}

The structural first stage reaches the target in eight conditions. Exact block elimination completes all eighteen to the same precision in 0.06--0.13 seconds per condition. It verifies the two-row star structure directly: each leaf coefficient occurs only in its paired signal and nuisance rows, and each block has one center coefficient. The implementation proposes the pairing from the family row order and checks every support, response, coefficient and row assignment. Independent blocks add their regret maxima.

For a fixed center coefficient $s$, eliminating a leaf gives a quotient $(a+bt)^2/(K+t^2)$ with $K>0$. Its maximum over the nuisance interval occurs at an endpoint or at $t=Kb/a$. On either endpoint the remaining signal-coordinate objective is quadratic; on the stationary branch it is quadratic on the interval where $t$ is feasible. Endpoint and scalar-stationary checks therefore give an exact pair oracle. Curvature interpolation certifies the remaining one-dimensional search. The block routine is a precision fallback for this verified sparse structure; the energy columns quantify what the learned spans alone certify.

All eighteen block receipts are checked independently with exact rational arithmetic, including every original pair maximum, profiled leaf coefficient, containing root, complete subdivision tree and feasible witness. This supplies an exact-reference interval throughout the residual/leakage grid.

The independent standard-library Fraction auditor imports no solver, proposal or energy-verifier implementation. It reconstructs every original-row constant, actual projector, retained-curvature enclosure, rounded radius and ridge gradient cap for all eighteen cases. For the two unperturbed cases it additionally verifies every saved original-row query, every node bound, binary coverage and full-model table witness. The analytic optima are separate explanatory checks. All checks pass in \texttt{multirank\_audit\_results.json}.

\subsection{A nontrivial intrinsic certification dimension}\label{app:intrinsic-two-v83}
Use the unperturbed $r=2,m=16$ construction above, with $p=34,n=64$, $\lambda=1/4$, $\alpha=1/32$ and $\beta=e_0+e_1$. Its signal halfwidths are $h_0=1/4,h_1=1/8$. The original independent feature intervals and full-rank property are preserved. Theorem~\ref{thm:energy} applies to $S_{\rm active}=\operatorname{span}(e_0,e_1)$ with
\[
 \epsilon=0,\quad a^2=1/2048,\quad
 G^2=17/65536,\quad\kappa=25/64.
\]
The Frobenius upper for $a$ suffices. Thus $R^2=17/10000$ and
\[
 E(S_{\rm active})\le a^2R^2=17/20480000.
\]

For a uniform lower bound, choose legal completion $X^{(j)}$, $j=0,1$, by placing all signal rows in block $j$ at their lower endpoints and all signal rows in the other block at their midpoints; put every nuisance coordinate at eight. The original normal matrix is block diagonal between the two feature blocks. Its gradient is supported only in block $j$, so its full optimizer displacement $u^{(j)}=M_{X^{(j)}}^{-1}g_{X^{(j)}}$ has that same block support. In particular, $u^{(0)}\perp u^{(1)}$. Direct Schur equations give active components
\[
 s_0=320/8001,\qquad s_1=80/9041.
\]
Put $A=\|u^{(0)}\|^2\ge s_0^2$ and $B=\|u^{(1)}\|^2\ge s_1^2$. For every line $S$,
\[
 \max_j\operatorname{dist}(u^{(j)},S)^2\ge AB/(A+B).
\]
Indeed, for a unit line direction $q$, the two normalized displacement vectors are orthogonal, so the sum of their squared projections onto $q$ is at most one. Simultaneously attaining errors smaller than $t$ would require that sum to exceed $2-t/A-t/B$, forcing $t>AB/(A+B)$. Components of $q$ outside the two-vector span cannot improve the inequality. The zero-dimensional space also obeys it.

Every $M_X\succeq\lambda I$, and $AB/(A+B)$ increases in both positive arguments. Therefore
\begin{align*}
 E(S)&\ge\lambda\frac{AB}{A+B}
 \ge\lambda\frac{s_0^2s_1^2}{s_0^2+s_1^2}\\
 &=\frac{25600}{1371850897}>1.866\times10^{-5}.
\end{align*}
Combining upper and lower bounds yields $d_\varepsilon=2$ throughout
\[
 \frac{17}{20480000}\le\varepsilon<\frac{25600}{1371850897},
\]
including $\varepsilon=10^{-5}$. This concerns the uniform pointwise error in~\eqref{eq:intrinsic}. The two-completion argument supplies a lower bound against every line, strengthening the candidate-line control.

The independent Fraction program \texttt{intrinsic\_dimension.py} generates both completions from their original rows, solves the full 34-dimensional normal equations exactly, verifies orthogonality and checks all constants. The exact squared norms are $102416/64016001$ and $6401/81739681$, giving the slightly stronger two-point lower bound $25604/1371850897$. The displayed simpler bound suffices. No source public table or solver is imported.

\section{Sharp Spectral Tail and Affine Compression}\label{app:spectral-v79}
Separate an arbitrary fixed anchor, formed from fully observed rows and ridge, from the remaining rows. Let its Gram and response moment be $G_0\succ0,r_0$, and set
\begin{equation}
 \theta_0=G_0^{-1}r_0,\quad T=G_0^{-1/2},\quad
 C=(\beta-\theta_0)^\top G_0(\beta-\theta_0).
\end{equation}
For remaining rows write $u_i(x)=Tx$, $a_i(x)=x^\top\beta-y_i$, and $d_i(x)=x^\top\theta_0-y_i$. Let $Q$ be an orthonormal retained basis and $P$ an orthonormal complement. The weighted retained and discarded row matrices have rows $\sqrt{w_i}u_i(x)^\top Q$ and $\sqrt{w_i}u_i(x)^\top P$, denoted $A_X,E_X$.

Suppose $\|E_X\|_\op^2\leq\rho$ uniformly over legal completions. Define
\begin{align}
 B_0^2&=\sum_iw_i\left(|c_i^\top\theta_0-y_i|
               +\sum_j s_{ij}|\theta_{0j}|\right)^2,\\
 \gamma&=(1+\rho)^{-1},\nonumber\\
 F_c(s)&=C-\|s\|^2\nonumber\\[-2mm]
 &+\sum_iw_i\max_x\{a_i(x)^2-c(d_i(x)-u_i(x)^\top Qs)^2\}.
\end{align}
All sums here are over the non-anchor rows, which may themselves include fixed observations.

\begin{theorem}[Sharp spectral-tail certificate]\label{thm:tail}
Let $L_r=\max_s F_1(s)$ and $U_r=\max_s F_\gamma(s)$. Then
\begin{equation}
 L_r\leq H^*\leq U_r,\qquad
 0\leq U_r-L_r\leq\frac{\rho}{1+\rho}B_0^2.\label{eq:tail}
\end{equation}
The tail factor is sharp over original interval-valued ridge problems. If a search upper $U$ and a tested point with value $f$ satisfy $U-f\leq\delta$, the row-maximizing table $\widehat X$ at that point satisfies
\begin{equation}
 0\leq U-h_\beta(\widehat X)
       \leq\delta+\frac{\rho}{1+\rho}B_0^2.\label{eq:witness}
\end{equation}
Thus a single retained upper problem supplies both the full-domain upper and a near-worst original-table witness.
\end{theorem}
The proof is in Appendix~\ref{app:method-details}.

\subsection{Automatic Bases and Verified Rounding}
Covariance envelopes or uncertainty-gradient generators propose retained spans. Every proposed vector is rounded to a dyadic matrix $R$; its actual Gram $H=R^\top R$ and projector $\Pi=RH^{-1}R^\top$ are verified. The retained penalty uses $H$. Outward spectral bounds on the actual projected covariance supply $\rho$, and full rank uses $R=I,\rho=0$. The appendix gives the general envelope and protected-uncertainty construction.

\section{Observed-Anchor and Moment-Tail Details}\label{app:new}
\subsection{A safe metric surrogate for the sharp spectral theorem}
Let $\bar\theta$ be a dyadic approximation to $G_0^{-1}r_0$, $b=T^\top(G_0\bar\theta-r_0)$, and $\bar C=J_{\mathcal A}(\beta)-J_{\mathcal A}(\bar\theta)$. For $\theta=\bar\theta+Tz$, the anchor difference is $\bar C-2b^\top z-z^\top Kz$. By~\eqref{eq:actual_metric},
\[
 \begin{gathered}
 \bar C-2b^\top z-z^\top Kz\leq C_{\rm safe}-\mu\|z-z_0\|^2,\\
 z_0=-b/\mu,\qquad C_{\rm safe}=\bar C+\|b\|^2/\mu.
 \end{gathered}
\]
Use $\theta_c=\bar\theta+Tz_0$ and $B_c^2=\sum_{i\notin\mathcal A}w_i\max_x|x^\top\theta_c-y_i|^2$. Scaling by $\sqrt\mu$ gives the spectral-discount upper. At a surrogate optimum, $\|z-z_0\|^2\leq B_c^2/(4\mu)$. The surrogate exceeds the exact anchor by $z^\top(K-\mu I)z$, bounded by
\[
 \eta_{\rm metric}=(\nu-\mu)\left(2\|b\|^2/\mu^2+B_c^2/(2\mu)\right).
\]
The total spectral certificate error is at most search error plus $\rho B_c^2/(1+\rho)+\eta_{\rm metric}$. The moment method retains $K$ in its query, so its corresponding anchor approximation is exactly zero.

\subsection{Computable moment constants and basis proposal}
For a projected coordinate $v$, each uncertain-row contribution to $v^\top\widehat g_X$ is $\ell_i(x)a_i(x)$, where $\ell_i(x)=v^\top T^\top x$. The identity $\ell a=((\ell+a)/2)^2-((\ell-a)/2)^2$ permits the planar oracle or original-box-edge reference to maximize it and its negative exactly. Sum these row ranges with the fixed-anchor coordinate before taking the sum of squared coordinate maxima. This gives $G_r$ and an analogous original bound $U_0\geq\sup_X\|\widehat g_X\|^2/\mu$.

Write $P=I-\Pi$, $u_i=T^\top c_i$, and $\delta_i=T^\top(x_i-c_i)$. Let $d_{Pi}=\sum_j s_{ij}\|PT^\top e_j\|$ and $d_{Ri}=\sum_j s_{ij}\|R^\top T^\top e_j\|$. Then
\[
 \begin{aligned}
 \sup_X\|P\widehat M_XR\|_{\op}&\leq\|P\widehat M_cR\|_F\\
 &\quad+\sum_iw_i\bigl[d_{Pi}\|R^\top u_i\|\\
 &\qquad+\|Pu_i\|d_{Ri}+d_{Pi}d_{Ri}\bigr].
 \end{aligned}
\]
Choose $C_r$ as an outward upper enclosure of the right-hand side. Norms and products are rounded outward. The first term includes the exact fixed-anchor cross block as well as midpoint rows.

For a proposal, collect first-order uncertainty-gradient generators
\[
 w_is_{ij}\{\beta_jT^\top c_i+(c_i^\top\beta-y_i)T^\top e_j\}
\]
and second-order generators $w_is_{ij}s_{il}\beta_lT^\top e_j$. Numerical leading eigenvectors of their Gram sum are rounded to a $2^{-24}$ grid. The actual Gram and inverse define the projector. The proposal covariance determines no certificate by itself. Full rank uses $R=I$ and a zero tail exactly.

For the retained Gram $H=R^\top R$, a positive Gershgorin lower bound supplies $\mu_H$. Otherwise, an exact inverse verifies full column rank and gives $\mu_H=1/\|H^{-1}\|_\infty$, rounded downward. Since $H$ is a Gram matrix and $H^{-1}$ is symmetric, $\lambda_{\min}(H)=1/\|H^{-1}\|_2\geq1/\|H^{-1}\|_\infty$. This also certifies nonorthogonal proposals whose Gram is not diagonally dominant.

\subsection{Root, interpolation and center-gradient bounds}
For the moment method, $\|s\|\leq\sqrt{U_0/(\mu\mu_H)}$ follows from $\|Rs\|^2=s^\top Hs$ and $\widehat M_X\succeq\mu I$. A valid curvature diagonal is $(R^\top KR)_{jj}+\sum_{i\notin\mathcal A}w_i\max_x(x^\top TR_j)^2$. At low rank the implementation combines corner interpolation and a center-gradient bound. At high rank it uses the latter to avoid $2^p$ root corners. For each fixed completion, concavity gives $f_X(s)\leq f_X(c)+\nabla f_X(c)^\top(s-c)$. An interval enclosing all completion gradients at $c$ therefore bounds the envelope on the box. Both methods inherit the parent upper and the original cap $U_0$.

The total moment error is search plus $\eta_r$, before feasible-witness and $U_0$ tightening. If the retained objective converges while the target is finer than $\eta_r$, the implementation returns a tail-or-arithmetic-floor status. An adaptive production implementation may then increase rank or arithmetic precision; the fixed experiments preserve their specified rank cap.

\subsection{A limit on uniform spectral compression}
Let $K_{\rm avg}=\mu^{-1}T^\top[\sum_{i\notin\mathcal A}w_i(c_ic_i^\top+\diag(s_i^2))]T$. For any retained rank $r<p$,
\[
 \sup_X\|E_X\|_{\op}^2/\mu\geq\lambda_{r+1}(K_{\rm avg}).
\]
Indeed, averaging independent symmetric endpoints produces $K_{\rm avg}$. For a fixed discarded projector, the supremum dominates the largest eigenvalue of the projected average. The min-max principle gives the displayed rank restriction. This limits a uniform covariance-tail guarantee; it need not limit the complementary residual-moment bound. The computable average-eigenvalue diagnostic describes this spectral obstruction; the certified tail uses outward bounds.

\section{Proof and Implementation Details}
\subsection{Quadratic elimination and the root}
For non-anchor weighted residual vector $d_X$, use the Euclidean norm after including $\sqrt{w_i}$ in every row. The exact discarded optimum is $t=(I+E_X^\top E_X)^{-1}E_X^\top(d_X-A_Xs)$. Substitution gives the inverse identity in the main proof. For $S\succeq I$, diagonalize $S$ to see that every eigenvalue of $S^{-1}-(S+\rho I)^{-1}$ is $\rho/(\sigma(\sigma+\rho))\leq\rho/(1+\rho)$. This supplies the tail constant without a bound on $\beta$ or on the retained parameter.

For a fixed completion, the retained discounted objective has Hessian $-2(I+cA_X^\top A_X)$ and maximizer $c(I+cA_X^\top A_X)^{-1}A_X^\top d_X$. The scalar inequality $c\sigma/(1+c\sigma^2)\leq\sqrt c/2$ gives the radius simultaneously for every completion. A maximizing pair exists by compactness of the completion domain and uniform coercivity in the comparator. Hence a ball of that radius contains at least one global maximizer of the row-maximized function.

\subsection{Exact row queries and degeneracies}
Orient each nonzero planar generator into a common half-plane, sort by the sign of cross products, and start at the center minus the oriented generator sum. Traverse the oriented generators and then their negatives. This visits the boundary in $2k$ steps, including collinear subdivisions. For a quadratic $At^2+Bt+C$ on $[0,1]$, examine $t=0,1$ and, when $A<0$ and $0<B<-2A$, the stationary point $t=-B/(2A)$. The rank-one image is a single complete segment; zero generators can be dropped. Exact rational orientation and stationary values remove geometric tolerances.

The row-face theorem can also be proved directly by holding the second residual fixed. A box cut by one scalar equality has an extreme-point representation with at most one interior coordinate. The convex first-residual square has a maximizing extreme point. The argument includes zero rows, proportional residuals and constant row objectives.

\subsection{Verifying a rounded learned subspace}
Let $R$ have full column rank and $H=R^\top R$. The orthogonal projector is $\Pi=RH^{-1}R^\top$. The ridge-only controls use $\lambda=1/64$ and exact scale 8. That implementation parameterizes $\theta=+8Rz$; replacing $z$ by its negative recovers the main theorem's convention. Its exact retained objective is
\[
\begin{aligned}
 &\lambda\|\beta\|^2-z^\top Hz\\
 &+\frac1n\sum_i\max_{x_i}\{(x_i^\top\beta-y_i)^2-
              \gamma(8x_i^\top Rz-y_i)^2\}.
\end{aligned}
\]
If $H\succeq\mu_HI$, the root box radius satisfies $\mu_H R_0^2\geq\gamma B_0^2/4$. Curvature diagonals replace 1 in the main formula by $H_{jj}$. The interpolation identity is unchanged because it preserves all cross terms.

For protected uncertain coordinates $U$, the learned columns have exact zero entries on $U$ and the coordinate vectors of $U$ are included in $R$. Thus $(I-\Pi)e_j=0$ for every uncertain direction. With $K_c=n^{-1}\sum_i(8c_i)(8c_i)^\top$, the discarded covariance is $(I-\Pi)K_c(I-\Pi)$ for every original completion. Round a full eigenvector test matrix to a dyadic invertible $S$ and enclose $A=(I-\Pi)S$. Gershgorin bounds give
\[
 S^\top S\succeq\mu_SI\succ0,\qquad A^\top K_cA\preceq\nu I.
\]
Writing an arbitrary vector as $x=Sz$ proves a spectral upper $\rho=\nu/\mu_S$. The inverse of $H$ is computed exactly as a rational matrix and enclosed before interval products. A fully retained model uses $R=I,\rho=0$ exactly. The rank cap is 3; if no rank meets the requested tail budget, the smallest rank within numerical verification slack of the best available bound is returned and the unmet budget is flagged explicitly.

\subsection{Arithmetic and independent checks}
Each input float is interpreted as its exact binary value. Elementary float64 interval operations round outward with nextafter; finite sums add conservative accumulation bounds. The sparse-row public and residual-energy comparisons use original-box-edge queries with outward scalar-quadratic bounds. The integrated planar adapter replaces those queries in observed-anchor and moment solvers for the dense-row study, enclosing coefficient uncertainty by a uniform perturbation bound. It preserves original-box witnesses. The sweep records the best edge and reconstructs its original-box witness once, preserving the stated sorting-plus-linear-query arithmetic complexity. The ridge-only controls instead use exact Fraction planar row values, individually rounded outward on a $2^{-80}$ grid before summation, with independent original-edge checks on a $2^{-112}$ grid. Independent checks use exact Fraction or dyadic-integer arithmetic on the original rows, reconstructing the actual metric, rounded projector and witness objective separately from the solvers.

Solver-independent audits verify complete 6- and 66-feature automatic search trees, including query values, root geometry, curvature, binary coverage and leaf uppers. Generic-baseline checks reconstruct local gradient ranges, Gram bounds and the fixed-row fallback. The retained suite covers five complete synthetic/public trees, 26 general row cases and four sharpness controls.

\section{Additional Geometry and Search Details}\label{app:method-details}

\subsection{The interior counterexample}
Proposition~\ref{prop:interior} follows by minimizing each diagonal quadratic: the first-coordinate optimum is $(1+t^2+3\lambda)^{-1}$. At $t=\pm A$ this equals the candidate, whereas the difference at $t=0$ gives the stated positive regret. All tables have full column rank, including the unregularized example.

\subsection{Row-face extremizers (Theorem~\ref{thm:row})}
\begin{proof}
Fix a globally maximizing pair $(X,\theta)$. In one row, fix $b=x^\top\theta-y_i$ at its maximizing value. Maximizing the remaining convex function $(x^\top\beta-y_i)^2$ over the box intersected with this one hyperplane has an extreme-point solution. Such a point has at most one non-endpoint coordinate: with two free coordinates there is a nonzero feasible perturbation preserving the hyperplane. Replacing the row preserves or increases the joint objective. Applying this to every row preserves global optimality.
\end{proof}

\subsection{Sharp spectral tail (Theorem~\ref{thm:tail})}
\begin{proof}
For each completion, eliminating discarded coordinates replaces the residual penalty by $(I+E_XE_X^\top)^{-1}$, between $\gamma I$ and $I$. Eliminating retained coordinates too gives an inverse difference with scalar eigenvalue bound $\rho/(1+\rho)$. Apply it to the reference residual $d_X$ and maximize over completions. For the witness statement, let $f=F_\gamma(s)$ and let $\widehat X$ attain its row maxima. At that same $s$, the undiscounted fixed-table objective is $f-(1-\gamma)\|d_{\widehat X}-A_{\widehat X}s\|^2$. Optimize the undiscounted retained comparator for this fixed table: the inverse-difference bound above gives its optimum at least the discounted fixed-table optimum minus $\rho B_0^2/(1+\rho)$, and the latter optimum is at least $f$. Full retraining can only increase this feasible regret. Thus $h_\beta(\widehat X)\geq f-\rho B_0^2/(1+\rho)$, proving~\eqref{eq:witness} after adding $U-f\leq\delta$.
\end{proof}

\subsection{Sharpness on an original table.}
Take one record and one feature, $w=\lambda=1$, $\beta=0$, $y=B$, $x\in[-\sqrt{\rho},\sqrt{\rho}]$, and retain rank zero. The lower projected value is zero, while the exact full regret and the upper value equal $\rho B^2/(1+\rho)$.

\subsection{A finite root and computable curvature.}
For fixed $X$, the retained optimum is
\begin{equation}
 s_c=c(I+cA_X^\top A_X)^{-1}A_X^\top d_X.
\end{equation}
The singular-value multiplier $c\sigma/(1+c\sigma^2)$ is at most $\sqrt c/2$. Therefore a global optimizer lies in $\|s\|\leq\sqrt c B_0/2$. Valid curvature diagonals are
\begin{equation}
 H_j=1+c\sum_iw_i\left(|q_j^\top Tc_i|+
                 \sum_l s_{il}|q_j^\top Te_l|\right)^2.\label{eq:curv}
\end{equation}

\subsection{Curvature interpolation (Theorem~\ref{thm:interp})}
\begin{proof}
Independently randomize the coordinates of a box vertex $q$ to have mean $s$. For each fixed completion, the quadratic identity expresses its value at $s$ as its expected value at $q$ plus diagonal curvature times $(s_j-l_j)(u_j-s_j)$. The expected value is bounded by the largest vertex value of $F_c$, and each product is at most $(u_j-l_j)^2/4$. Maximize over completions and parameters. Cross terms are preserved by independent endpoint interpolation.
\end{proof}

\section{Nonzero-Residual Mechanism Details}
\subsection{Predeclared nonzero nuisance residuals}
\label{app:epsilon}
We keep the $m=64$, $p=65$, $\alpha=1/32$ construction, its independent feature intervals, $\lambda=1/2$, candidate $\beta=e_0$, and retained original direction $Q=e_0$ fixed. Only the nuisance responses change: $y_{64+j}=\alpha+\sigma_j\delta$, with 32 signs of each kind. The predeclared levels are $\delta\in\{0,2^{-15},2^{-14},2^{-13},2^{-12}\}$. They were frozen using the theoretical tail budget before the experiment. For positive $\delta$, the fixed candidate has nuisance midpoint-gradient entries $-3\sigma_j\delta/32$; the zero-residual control is the midpoint ridge fit. The retained direction is specified explicitly by the mechanism.

The original-data conditions give $\epsilon^2=\delta^2/2$, $a^2=\alpha^2/2$, $G=1/32$, and $\kappa=25/32$. Consequently,
\[
R=\frac{1+16\alpha\delta}{25},\qquad
\eta=\frac{(\delta+\alpha R)^2}{2}.
\]
Both $R$ and $\eta$ are exact rational numbers, although $\epsilon=\delta/\sqrt{2}$. The first four levels have $\eta<10^{-6}$; the fifth has $\eta=1.116456103\times10^{-6}$. All five remain in the reported queue. Tail and final-width budgets differ because the final lower bound includes an original-table feasible comparator.

The actual dyadic ridge-anchor transform is $T=cI$ with $c=741455/524288$ and the saved numerical basis is $e_0$. Each coordinate root is $[-\widehat R,\widehat R]$, where $\widehat R=2^{-20}\lceil2^{20}R/c\rceil$. The tail theorem uses the full optimizer projection of magnitude at most $R$; the outward rounded root encloses that projection. Exact radii are listed below.

For an additional original-target check, let $r_j=-\sigma_j\delta$ and $f_j=t_j^2/(64+t_j^2)$. After eliminating nuisance coefficients, the derivative of optimized regret with respect to $f_j$ is $(r_j-\alpha s_*)^2/128\ge0$. Thus all $t_j=16$ maximize the nuisance part. Balanced signs then cancel the retained linear perturbation, and all $u_i=3/4$ maximize the signal part. This yields
\[
H^*=\frac5{4000+512\alpha^2}+\frac{2\delta^2}{5},\qquad
V_Q=\frac1{800+512\alpha^2}.
\]
These exact optima serve as extra checks after the certificate is constructed. Point and node bounds use the original-row oracle and the residual-energy theorem.

Each condition uses the same $10$\,s search cap, absolute original-target tolerance $10^{-6}$, and zero relative tolerance. The numerical queue runs sequentially with one library thread on logical CPU4, reserving CPUs0--3. The reported times are single wall-clock observations and exclude process/import startup and JSON serialization. Setup includes input construction, anchor setup, and method setup. All conditions terminate at tolerance after 6 expanded nodes and 15 points (13 saved tree nodes).

\begin{center}
\small
\begin{tabular}{@{}rrrrr@{}}
\toprule
$\delta$ & Input (s) & Anchor (s) & Method (s) & Setup (s) \\
\midrule
$0$ & 0.0002 & 0.4378 & 0.0130 & 0.4509 \\
$2^{-15}$ & 0.0001 & 0.4412 & 0.0132 & 0.4545 \\
$2^{-14}$ & 0.0001 & 0.4355 & 0.0143 & 0.4499 \\
$2^{-13}$ & 0.0001 & 0.4671 & 0.0148 & 0.4820 \\
$2^{-12}$ & 0.0001 & 0.5101 & 0.0150 & 0.5252 \\
\bottomrule
\end{tabular}
\end{center}
\begin{center}
\small
\begin{tabular}{@{}rrr@{}}
\toprule
$\delta$ & Original radius $R$ & Coordinate radius $\widehat R$ \\
\midrule
$0$ & $\tfrac{1}{25}$ & $\tfrac{29659}{1048576}$ \\[2pt]
$2^{-15}$ & $\tfrac{65537}{1638400}$ & $\tfrac{29659}{1048576}$ \\[2pt]
$2^{-14}$ & $\tfrac{32769}{819200}$ & $\tfrac{7415}{262144}$ \\[2pt]
$2^{-13}$ & $\tfrac{3277}{81920}$ & $\tfrac{29661}{1048576}$ \\[2pt]
$2^{-12}$ & $\tfrac{8193}{204800}$ & $\tfrac{14831}{524288}$ \\[2pt]
\bottomrule
\end{tabular}
\end{center}
An independent standard-library Fraction audit checks all five original tables and response patterns, candidates, actual transforms and roots, midpoint gradients, full-rank minors, residual energies, and tails. It independently checks every saved original-row point value, every node upper bound, complete binary coverage, the original-table feasible witness and ridge-loss difference, and the final upper obtained from the tree and original gradient cap. All five complete certificates pass. The largest perturbation reaches width $9.681819\times10^{-7}$ in $0.550$\,s total ($0.025$\,s search), despite its larger tail-only budget.

\paragraph{Task-oriented guarantees.} Table~\ref{tab:related_guarantees} compares the certified quantity, completion domain and full-dimensional comparator.\label{app:sensitivity-v79}
\begin{table*}[t]
\centering
\caption{Guarantees for different dimension-reduction tasks. ``All completions'' refers to a fixed finite interval-valued table domain. References are to the inspected arXiv versions.}
\label{tab:related_guarantees}
\small\setlength{\tabcolsep}{4pt}
\begin{tabular}{p{0.19\textwidth}p{0.29\textwidth}p{0.18\textwidth}p{0.25\textwidth}}
\toprule
Work & Guaranteed quantity and mode & All completions & Full-dimensional comparison \\
\midrule
Quarteroni et~al. (2011), Sec.~6; Balabanov--Nouy (2019), Secs.~2.2--2.3 & Residual/coercivity solution and output-error certificates; Galerkin/POD spaces & Parameterwise certificates; uniformity requires parameter-domain bounds & Full-order parametrized equation solution \\
Chen et~al. (2025), Thm.~5 & Distribution-averaged $L^2$ function error; Jacobian bounds & No such guarantee stated & Original function $G$; reduced function approximation \\
Patil--LeJeune (2024), Thms.~2--3 & Prediction-risk equivalence and GCV consistency; almost-sure asymptotics & No such guarantee stated & Unsketched $p$-feature ridge at implicit $\mu$; finite-ensemble variance \\
Present residual-energy theorem & Additive enclosure of maximum completion-relative training regret; finite problem & Yes, under uniform row conditions & Exact retraining in all $p$ features at the original $\lambda$ \\
\bottomrule
\end{tabular}
\end{table*}

\begin{table*}[t]\centering\small\setlength{\tabcolsep}{4pt}
\caption{Closest guarantees and the present target. Fixed-rank geometry supplies the row oracle; the continuous-table comparator and uniform, verifiable compression bounds distinguish the complete method.}\label{tab:closest-main-v81}
\begin{tabular}{p{.20\textwidth}p{.34\textwidth}p{.39\textwidth}}\toprule
Work&Domain and comparator&Guaranteed quantity / relation here\\\midrule
Reduced basis~\citep{quarteroni2011,balabanov2019}&Parametrized full-order equations; Galerkin/POD/greedy spaces&Residual/coercivity approximation and output-error certificates\\
Fixed-rank box-QP~\citep{hladik2021}&A fixed low-rank quadratic on a box&Global optimum by zonotope faces; foundation for our row geometry\\
Robust regret~\citep{vanli2014}&Norm-bounded matrix/response perturbations&SDP for a regularized minimax-regret formulation using a first-order comparator approximation\\
Completion-relative learning, 2024~\citep{zhen2024sigmod}&Incomplete feature tables; published endpoint theorem assumes unbounded intervals&Learning under feature completions; the stated interval assumption matters for finite-box certification\\
Minimal repairs, 2026~\citep{repair2026}&Worst-repair suboptimality gap; Section~4.1 uses adversarial edge repairs&Repair selection and optimization; finite continuous boxes can also have interior worst regret (Proposition~\ref{prop:interior})\\
ZORRO, NeurIPS 2024~\citep{zhu2024zorro}&Uncertain training features/labels; all exact ridge optima at specified $\lambda$&Sound joint model/data zonotope; our same-target regret conversion and complete costs appear in Table~\ref{tab:zorro-v90}\\
Present method&Continuous cell completions; exact full-feature ridge retraining at original $\lambda$&Feasible original-table witness and global upper; computable uniform subspace error; sharp energy constant under original-row conditions\\
\bottomrule\end{tabular}\end{table*}

\section{Intrinsic Dimension, Baselines and Costs}
\paragraph{Intrinsic approximation dimension.}
Let $u_X=M_X^{-1}g_X$ and define
\begin{align*}
 E(S)&=\sup_{X\in\mathcal X}\min_{v\in S}\|u_X-v\|_{M_X}^{2},\\
 d_\varepsilon&=\min\{r:\min_{S\in\operatorname{Gr}(r,p)}E(S)\leq\varepsilon\}.
\end{align*}
The inner minimum is attained, and $h_\beta(X)-\max_{v\in S}(2g_X^\top v-v^\top M_Xv)=\operatorname{dist}_{M_X}(u_X,S)^2$. Thus the global reduction gap is at most $E(S)$. A verified uniform tail $\eta$ gives $d_\eta\leq\dim S$; it supplies an upper candidate for this intrinsic quantity. This is a varying-energy-norm version of classical Kolmogorov-width approximation, as studied in reduced-basis theory \citep{binev2011}. Our results give verifiable original-data sufficient conditions for small width. Spectral compression around $\theta_0$ is affine: its direct dimension bound is affine rank $r$, or linear rank at most $r+1$ after including $\beta-\theta_0$.

\subsection{Intrinsic dimension and the scope of each tail}\label{app:dimension-v78}\label{app:dimension}
Throughout, $\mathcal X$ is nonempty and compact, the weights are fixed and nonnegative, and $\lambda>0$. Put $u_X=M_X^{-1}g_X$. For $S\subseteq\mathbb R^p$ linear, set $f_X(v)=2g_X^\top v-v^\top M_Xv$ and $h_S(X)=\max_{v\in S}f_X(v)$.

\begin{proposition}[Pointwise identity and existence]
For every fixed $S$,
\[
 h_\beta(X)-h_S(X)=\min_{v\in S}\|u_X-v\|_{M_X}^2.
\]
Consequently $0\leq H^*-\max_Xh_S(X)\leq E(S)$. For every $r\in\{0,\ldots,p\}$, the minimum $e_r=\min_{S\in\operatorname{Gr}(r,p)}E(S)$ is attained. Also $e_0=H^*$, $e_p=0$, and $e_{r+1}\leq e_r$.
\end{proposition}
\begin{proof}
Completing the square gives
\[
 f_X(v)=g_X^\top M_X^{-1}g_X-(u_X-v)^\top M_X(u_X-v).
\]
For a full-column basis $B$ of $S$, the unique restricted displacement is $B(B^\top M_XB)^{-1}B^\top g_X$. The global implication follows by taking maxima of the pointwise inequality $h_\beta(X)\leq h_S(X)+E(S)$.

Identify the Grassmannian with the compact set of rank-$r$ orthogonal projectors $P$. The matrix $A(P,X)=PM_XP+I-P$ has minimum eigenvalue at least $\min(\lambda,1)>0$, and
\[
 h_S(X)=g_X^\top P A(P,X)^{-1}P g_X.
\]
This expression is jointly continuous in $P,X$. The pointwise error is therefore jointly continuous; its maximum on compact $\mathcal X$ is continuous in $P$. It attains a minimum on the Grassmannian. Subspace augmentation gives monotonicity. The endpoint identities follow from $S=\{0\}$ and $S=\mathbb R^p$.
\end{proof}

\paragraph{Why a global-gap definition would be vacuous.}
Let $X_*$ attain $H^*$. The space $S_* =\operatorname{span}(u_{X_*})$ has dimension at most one and already satisfies $\max_Xh_{S_*}(X)=H^*$. Its uniform pointwise error can remain large. For example, take one row $x=(1,t)$, $t\in[-1,1]$, $W=1$, $\lambda=1$, $y=1$, and $\beta=0$. Then $u_t=-(1,t)/(2+t^2)$ and $h_\beta(t)=(1+t^2)/(2+t^2)$. The span of $(1,1)$ attains the global value $2/3$ at $t=1$ but has pointwise error $2/3$ at $t=-1$. Thus a small global reduction gap alone yields no bound on $E(S)$.

\paragraph{Connection to classical widths.}
The set $\mathcal K=\{u_X:X\in\mathcal X\}$ is compact. Let $\delta_r(\mathcal K)=\inf_{\dim S=r}\sup_{u\in\mathcal K}\operatorname{dist}_2(u,S)$ be its linear Kolmogorov width, and let $\Lambda=\sup_X\lambda_{\max}(M_X)<\infty$. Norm equivalence gives
\[
 \lambda\delta_r(\mathcal K)^2\leq e_r\leq\Lambda\delta_r(\mathcal K)^2.
\]
The width is a standard approximation-theoretic object up to these varying-energy-norm constants; see, for example, \citet{binev2011}. The present sufficient conditions bound it from original table intervals and preserve the full ridge comparator.

\paragraph{Actual-metric residual moments.}
Let the actual verified $T$ be invertible, $K_X=T^\top M_XT\succeq\mu I$, $b_X=T^\top g_X$, and $U_0\geq\sup_Xh_\beta(X)$. For full-column $R$, put $H=R^\top R\succeq\mu_HI$, $\Pi=RH^{-1}R^\top$, $D=I-\Pi$, and $S=\operatorname{range}(TR)$. Suppose uniformly
\[
 \|Db_X\|\leq G_r,\qquad \|DK_XR\|_{\mathrm{op}}\leq C_r.
\]
For every $X$, $w_X=T^{-1}u_X$ satisfies $\|w_X\|^2\leq U_0/\mu$. Its projected coordinate $z_X=H^{-1}R^\top w_X$ therefore satisfies
\[
 \|z_X\|\leq R_s=\sqrt{U_0/(\mu\mu_H)}.
\]
For fixed $X,z_X$, set $q_X=D(b_X-K_XRz_X)$ and let $K_{D,X}$ denote $K_X$ restricted to $\operatorname{range}(D)$. Elimination yields
\begin{align*}
 h_\beta(X)-f_X(TRz_X)&=q_X^\top K_{D,X}^{-1}q_X\\
 &\leq (G_r+C_rR_s)^2/\mu.
\end{align*}
Since $h_S(X)\geq f_X(TRz_X)$, this proves $E(S)\leq\eta_r$ and $d_{\eta_r}\leq r$. The radius must cover these projected coordinates for every $X$; a containing cube and outward-rounded radius suffice. At full rank the discarded space and tail vanish exactly. This argument uses the actual metric and has no whitening-surrogate anchor error.

\paragraph{Original-row residual energy.}
Under the signal-support and residual-energy conditions of Theorem~\ref{thm:energy}, put $q_X(s)=e_X-A_Xs$ and $B_X=Z_X(\lambda I+Z_X^\top Z_X)^{-1}Z_X^\top$. The elimination identity at fixed $X,s$ adds
\[
 q_X(s)^\top B_Xq_X(s)\leq\|q_X(s)\|^2.
\]
Writing $L_Z=I-Z_X(\lambda I+Z_X^\top Z_X)^{-1}Z_X^\top$, the full optimizer's retained coordinate obeys
\[
 [\lambda I+C_{\mathcal S,X}+A_X^\top L_ZA_X]s_X
 =f_{\mathcal S,X}+A_X^\top L_Ze_X.
\]
The uniform bounds and $0\prec L_Z\preceq I$ imply $\|s_X\|\leq R=(G+a\epsilon_0)/\kappa$ for every $X$, where $\epsilon_0$ denotes nuisance residual energy. Comparing with $Qs_X$ proves
\[
 E(\operatorname{range}(Q))\leq(\epsilon_0+aR)^2.
\]
The all-completions radius supplies the required root. Extensions to anchors or candidate leakage require their stated additional pointwise conditions.

\paragraph{Spectral compression requires affine dimension.}
For ideal exact whitening $T^\top G_0T=I$ and $\theta_0=G_0^{-1}r_0$, retained comparators have $\theta=\theta_0+TQs$. Their displacements form
\[
 \mathcal A=(\beta-\theta_0)+\operatorname{range}(TQ),
\]
which need not be linear. Define $E_{\mathrm{aff}}(\mathcal A)=\sup_X\min_{v\in\mathcal A}\|u_X-v\|_{M_X}^2$. Let $d_X=W_{\mathcal N}^{1/2}(X_{\mathcal N}\theta_0-y_{\mathcal N})$, and let $A_X,E_X$ denote the retained and discarded whitened row matrices. If $\|E_X\|_{\mathrm{op}}^2\leq\rho$ and $\|d_X\|^2\leq B_0^2$ uniformly, the restricted-versus-full minimized loss difference at each $X$ is
\[
 d_X^\top\{C_X^{-1}-(C_X+B_X)^{-1}\}d_X,
\]
where $C_X=I+A_XA_X^\top\succeq I$ and, in this paragraph, $B_X=E_XE_X^\top$. Loewner monotonicity and scalar spectral calculus give
\begin{align*}
 0&\preceq C_X^{-1}-(C_X+B_X)^{-1}\\
 &\preceq C_X^{-1}-(C_X+\rho I)^{-1}\\
 &\preceq\frac{\rho}{1+\rho}I.
\end{align*}
Hence $E_{\mathrm{aff}}(\mathcal A)\leq\rho B_0^2/(1+\rho)$. This gives affine dimension at most $r$. For the linear definition use
\[
 S_+=\operatorname{span}\bigl(\{\beta-\theta_0\}\cup\operatorname{range}(TQ)\bigr),
\]
whose dimension is at most $\min(p,r+1)$ and whose error is no greater. Rank $r$ suffices when the offset already belongs to the retained range.

This implication is proved pointwise for ideal exact whitening. A global discounted-surrogate upper-minus-lower bound alone does not establish this uniform error. For rounded whitening, an intrinsic-dimension assertion additionally requires a uniform per-$X$ comparator inequality and an all-$X$ root, including any anchor correction. We make no such inference solely from a global rounding correction. The actual-metric residual-moment argument already applies to the verified rounded $T$.

For an explicit ideal-spectral root that is uniform in $X$, the whitened full comparator minimizes $\|w\|^2+\|d_X+U_Xw\|^2$. Its value is no greater than the value at $w=0$, hence $\|w_X\|\leq B_0$. Its orthogonal retained projection therefore lies in $[-B_0,B_0]^r$ for every completion. This conservative root is sufficient for the pointwise comparison; a smaller implementation root requires its own all-completions proof.

\subsection{Reconstructable continuous input-box baseline}\label{app:input-baseline-v78}\label{app:input}
The input-space baseline uses cellwise branch and bound and an observed-row fallback. All weights here are $1/n$. The baseline branches directly over uncertain input cells and uses the same full-feature ridge comparator.

\paragraph{Feasible lower bounds.}
For a legal table $X$, solve the numerical ridge equations for $\widehat\theta$ and evaluate $J_X(\beta)-J_X(\widehat\theta)$ with outward interval arithmetic. Its lower endpoint is a valid lower bound on $h_\beta(X)$ even when the numerical solve is inexact, because $\widehat\theta$ is a feasible comparator. Initialize $L=0$ and retain the best table and comparator. Each node tests its midpoint. Before branching, the baseline also runs two bounded L-BFGS-B starts: the midpoint and a uniform legal completion from seed $73101$, at most $35$ iterations per start, with \texttt{ftol}$=10^{-14}$ and \texttt{gtol}$=10^{-12}$. For uncertain cell $(i,j)$ the optimization gradient is
\[
 \frac{2}{n}\{(x_i^\top\beta-y_i)\beta_j-(x_i^\top\widehat\theta-y_i)\widehat\theta_j\}.
\]
This heuristic proposes tables; clipping to the input box and an interval loss difference supply the actual lower certificates.

\paragraph{A local Gram and gradient enclosure.}
For a node $[L_B,U_B]$, let $C=(L_B+U_B)/2$, $D=(U_B-L_B)/2$, and
\[
 T_B=2^{-24}\operatorname{round}\!\left[2^{24}\operatorname{chol}
 (C^\top C/n+\lambda I)^{-\top}\right].
\]
The matrix inside rounding is the inverse transpose of the lower Cholesky factor. Every enclosure below concerns this actual dyadic $T_B$. Write $t_j$ for column $j$, $c_{ij}=c_i^\top t_j$, $r_i=c_i^\top\beta-y_i$, $\delta_{ij}=\sum_\ell D_{i\ell}|T_{B,\ell j}|$, and $\delta_{i\beta}=\sum_\ell D_{i\ell}|\beta_\ell|$. Thus $x_i^\top t_j$ lies in $[c_{ij}-\delta_{ij},c_{ij}+\delta_{ij}]$ and the residual lies in $[r_i-\delta_{i\beta},r_i+\delta_{i\beta}]$, all computed outward.

If row $i$ has at most one uncertain coordinate, use their common scalar $z\in[-1,1]$ to range the products exactly as univariate quadratics, including an interior stationary point. With uncertain coordinate $\ell$ and radius $s=D_{i\ell}$, the gradient-product polynomial is
\[
 s^2T_{B,\ell j}\beta_\ell z^2
 +s(T_{B,\ell j}r_i+c_{ij}\beta_\ell)z+c_{ij}r_i,
\]
and the Gram-product polynomial for $(j,k)$ is
\[
 s^2T_{B,\ell j}T_{B,\ell k}z^2
 +s(T_{B,\ell j}c_{ik}+T_{B,\ell k}c_{ij})z+c_{ij}c_{ik}.
\]
For a fixed row $s=0$ and the selected index is immaterial. For a row with multiple uncertain coordinates, multiply the respective interval ranges instead. Sum these row enclosures and add the exact-form ridge terms to enclose
\[
 b_X=T_B^\top g_X,\qquad K_X=T_B^\top M_XT_B.
\]
Specifically, add $\lambda t_j^\top\beta$ to gradient coordinate $j$ and $\lambda t_j^\top t_k$ to Gram entry $(j,k)$. If the enclosures are $[\underline b_j,\overline b_j]$ and $[\underline K_{jk},\overline K_{jk}]$, compute the downward Gershgorin bound
\[
 \mu_B=\min_j\left\{\underline K_{jj}
 -\sum_{k\ne j}\max(|\underline K_{jk}|,|\overline K_{jk}|)\right\}.
\]
When $\mu_B>0$, positive definiteness also certifies that $T_B$ is invertible, and
\[
 h_\beta(X)\leq U_{\mathrm{loc},B}
 =\operatorname{up}\!\left[\frac{\sum_j\max(\underline b_j^2,\overline b_j^2)}{\mu_B}\right]
 \quad(X\in B).
\]
The implementation obtains the numerator with outward interval squares and sums. If $\mu_B\leq0$, it uses the inherited global fallback $U_0$ instead of this quotient. The node upper is
\[
 U_B=\min\{U_{\mathrm{parent}},U_0,U_{\mathrm{loc},B}\},
\]
with root parent cap $U_0$.

\paragraph{The global fallback and its $\lambda$ correction.}
The constructor receives $U_0$; the matched driver supplies the independently valid \texttt{solver.initial\_upper} from the full-box reference solver. This is the minimum of its fixed-row transformed-gradient bound and its direct fixed-row-loss bound. For the first, form $M_A=X_A^\top X_A/n+\lambda I$ from fully fixed rows, round its inverse Cholesky transpose on a $2^{-16}$ grid, verify $T_A^\top M_AT_A\succeq\mu_AI$ by outward Gershgorin, and uniformly enclose $T_A^\top g_X$. Then $U_{\mathrm{grad}}=\operatorname{up}(\sup\|T_A^\top g_X\|^2/\mu_A)$ is valid. The one-cell implementation ranges scalar quadratics; the multi-cell implementation uses interval product ranges.

For the direct fallback, let $J_A(\theta)=\|X_A\theta-y_A\|^2/n+\lambda\|\theta\|^2$. A numerical fixed-row ridge fit $\widetilde\theta$ is corrected by its interval gradient $a_A=M_A\widetilde\theta-X_A^\top y_A/n$:
\[
 \min_\theta J_A(\theta)\geq
 J_A(\widetilde\theta)-\|a_A\|^2/\lambda.
\]
Since $J_X\geq J_A$, this gives
\begin{align*}
 v_A&=\sup_XJ_X(\beta)-J_A(\widetilde\theta)+\|a_A\|^2/\lambda,\\
 U_{\mathrm{direct}}&=\max\{0,\operatorname{up}(v_A)\},\\
 U_0&=\min(U_{\mathrm{grad}},U_{\mathrm{direct}}).
\end{align*}
The candidate-loss supremum is enclosed by scalar quadratic ranges for one-cell rows and by absolute residual envelopes for general row boxes. The factor $1/\lambda$ corrects an approximate fixed-row solve; a feasible fitted loss alone would be an upper bound on the fixed-row minimum and could not serve this role. If the fixed-row transformed curvature cannot be verified, that reference constructor fails explicitly; the input baseline does not silently replace it with an uncertified number. A separately supplied fallback must likewise be valid. The generic inequality $h_\beta(X)\leq\|g_X\|^2/\lambda$ offers another possible safe cap, but it is not an additional node calculation in this implementation.

\paragraph{Queue, split and stopping.}
Use a max-upper heap, stable by node ID, and remove nodes with $U_B\leq L$. On each surviving node, split the first row-major maximizer of
\[
 \frac{(U_B^{\mathrm{input}}-L_B^{\mathrm{input}})_{ij}}
 {\max((U^{\mathrm{input}}-L^{\mathrm{input}})_{ij},10^{-300})}
\]
at its midpoint, then bound both children with their parent cap. This is the largest remaining fraction of an original cell width, with ties resolved by NumPy's first \texttt{argmax}. Stop when the heap is empty, when
\begin{align*}
 U-L&\leq\max(10^{-9},\tau_{\mathrm{rel}}L),\\
 U&=\max\{L,\max_{B\text{ live}}U_B\},
\end{align*}
or when elapsed time reaches its cap. Here $\tau_{\mathrm{rel}}$ is the relative tolerance, whose default is $0.01$. The clock starts before the two-start heuristic and root bound. Time is checked between node expansions, so a heuristic, root, or pair of child bounds can overrun it. Retained records contain node boxes, parent/child IDs, uppers, local $\mu_B,T_B,b$ receipts, the global fallback, and a feasible table/comparator witness. A time-limit status preserves the live-leaf upper.

\paragraph{Driver-specific fallback and clock boundary.}
The preceding fixed-row fallback is the one supplied by \texttt{run\_matched.py}, whose baseline-only clock begins after that inherited fallback was prepared by the proposed solver. Its recorded own baseline elapsed time therefore excludes fresh fallback construction. The physical-unit decision experiment uses a different driver: \texttt{v75/decision\_experiment.py} constructs each method's own \texttt{ObservedAnchor} and passes its \texttt{a.gradient\_upper} to \texttt{TimedInput}. The total clock starts before CSV/input construction, candidate fitting and that anchor, so those costs are charged to the input method. This fallback uses the actual-anchor uniform gradient cap. The timed driver uses the same local node enclosures and normalized split rule, stops once all declared thresholds are decided or the total cap is reached, and records only events completed by each deadline. Its root/heuristic can overrun the cap; completed late events receive no earlier credit. Thus the v73 inherited-fallback timing boundary is not the v75 decision experiment's accounting boundary.

\subsection{Cost beyond the retained search}\label{app:cost}
Let $n,p,r$ denote rows, original features and retained coordinates, let $k_i$ count uncertain cells in row $i$, and put $L=\sum_i k_i$ and $q=\sum_i k_i\log(k_i+1)$. These are arithmetic-operation counts; dense preprocessing retains its dependence on $p$.

\paragraph{Setup and verification.}
Forming candidate/anchor moments costs $O(np^2)$, and candidate fitting or a dense Cholesky transformation costs $O(p^3)$. Verifying the actual dyadic $T^\top G_0T$, projecting rows through dense $T$, and building dense metric envelopes cost $O(p^3+np^2)$. For a rounded retained matrix $R$, form its exact Gram in $O(pr^2)$, invert it in $O(r^3)$ rational operations, and form $\Pi=R(R^\top R)^{-1}R^\top$ in $O(p^2r+pr^2)$, with $O(p^2)$ projector storage in this implementation. Generic tail construction additionally uses dense projector/metric products and projected-gradient queries, up to $O(p^3+np^2+p(L+q))$ after forming the needed projections. The existing gradient proposal materializes $\sum_i(k_i+k_i^2)$ vectors in $\mathbb R^p$, forms their dense Gram and diagonalizes it; its cost can grow quadratically in row uncertainty before the planar oracle is called. A structural specialization that inherits this setup still pays for it.

\paragraph{Search and full witnesses.}
After retained projections are cached, a planar query costs $O(pr+nr+r^2+L+q)$. The sweep stores the best edge/local parameter and reconstructs a witness once. A node upper adds coordinate-envelope work; low-rank corner interpolation may query $2^r$ corners (enabled only for $r\leq3$ in the current code). Node count still depends on tolerance, conditioning and the root. Each improved original-table witness may require $O(np^2+p^3)$ for full-feature ridge retraining, plus $O(np)$ outward loss verification. This is charged whenever performed. The direct input-box baseline instead forms dense local transformed Grams and full midpoint fits at each node, costing $O(np^2+p^3)$ per bound/witness pair up to constants; it stores full original node boxes.

\paragraph{Arithmetic, bit cost and clocks.}
Exact cross-product comparisons, Fraction operations and matrix inversion are not constant-time bit operations: their cost depends on intermediate numerator/denominator length, gcd normalization and serialization. The planar $O(k_i\log(k_i+1))$ comparison count and $O(k_i)$ remaining arithmetic count therefore require an additional integer-bit-cost factor; exact inverse entries need not remain dyadic. Numerical interval execution and independent rational auditing have separate costs. Report input/candidate plus anchor, basis/projector/tail, search plus full witness, and independent audit/export separately. For the input-space component control, a supplied fixed-row fallback is prepared outside its baseline-only clock. Complete-budget comparisons instead charge each method for constructing its anchor and gradient cap. All deadline claims use completed events. A small intrinsic dimension bounds a sufficient outer-search dimension while leaving these original-feature costs present.

\subsection{Metric-matched proposals under complete budgets}\label{app:metric-v79}
The approximation geometry affects a snapshot baseline. We compare eighteen conditions using solver-coordinate POD: for all 32 fixed legal sample tables, compute the full displacement $u_X=\beta-\theta_X$, transform $z_X=T^{-1}u_X$, and only then take the leading three left singular vectors. This differs from selecting physical-coordinate singular vectors before mapping.

Metric greedy likewise uses solver coordinates. On each sample it solves the current reduced problem for $K_X=T^\top M_XT$ and $b_X=T^\top g_X$, selects the largest numerical residual/coercivity indicator $\|b_X-K_Xv\|^2/\mu$, and appends that selected table's transformed full-fit displacement with two-pass orthogonalization. It performs three full fits, versus 32 for POD; both charge pool construction, moments, reduced indicator solves, transformations and verification. The numerical proposal indicator is separate from the verified uniform moment tail. Both orthonormal numerical bases are rounded to $2^{-24}$ before the identical verifier.

Each method receives a complete ten-second budget, including identical within-condition input/candidate/anchor charges and its own proposal, verification, search and full-model witness retraining costs. At one, five and ten seconds only completed endpoint events are credited; a late event is retained without retrospective credit. Before the common anchor is ready no finite endpoint is credited. Method order rotates, numerical libraries use one thread, and early tail/arithmetic-floor stops retain their current enclosures. The methods share each condition's masks and record pool.

\begin{table}[t]\centering\small\setlength{\tabcolsep}{4pt}
\caption{Metric-matched comparison: median interval widths over three masks at the ten-second complete-budget checkpoint, in $10^{-6}$ units. Each proposed basis has rank three.}\label{tab:metric-v79}
\begin{tabular}{llrrr}\toprule Data&Pattern&Gradient&Metric POD&Metric greedy\\\midrule
Diabetes&Conc.&76.62&94.69&101.39\\
Diabetes&Disp.&68.55&68.41&69.52\\
Concrete&Conc.&6.38&7.26&7.39\\
Concrete&Disp.&8.59&8.63&9.14\\
Supercond.&Conc.&1.19&1.32&1.55\\
Supercond.&Disp.&7.99&8.01&8.10\\
\bottomrule\end{tabular}\end{table}
At ten seconds, gradient intervals are narrower than metric POD in 16 of 18 conditions; metric POD is narrower in 2, with 0 exact ties. The median gradient/control width ratio is 0.918. At ten seconds, gradient intervals are narrower than metric greedy in 17 of 18 conditions; metric greedy is narrower in 1, with 0 exact ties. The median gradient/control width ratio is 0.920. 

Metric POD wins by 0.80\%, 0.21\% in its winning pairs, measured relative to gradient width. Metric greedy wins by 0.33\% in its winning pairs, measured relative to gradient width. At one second all six Superconductivity conditions lack a completed common anchor, so their endpoints are uncredited for every method; all 18 conditions have finite endpoints by five seconds. The certificate applies to the actual proposed span and does not require one basis to dominate. All pairs, certified tails, stage costs and checkpoint endpoints are retained. An independent program reconstructs transformed displacements and POD spectra before truncation, replays all 54 metric-greedy selections, checks orthogonality and dyadic/physical-span mapping, and reclassifies all 162 checkpoints directly from completed events. These are proposal and aggregate/deadline checks; full outward search-tree verification remains separate.

\subsection{Primary independent-reference acquisition}\label{app:calibration-v81}
The twelve-condition grid uses feature means and ranges estimated solely from 512 independent reference records. Seed 81041 selects these records outside all specified excluded pools and all three target pools. The constants are frozen before target masks, midpoint candidates or certificates. Target records, mask/reveal seeds, rank-three methods, full rank 82, $\lambda=1/32$, $0.25\,\mathrm{K}^2$ tolerance and ten-second complete online budget are fixed across comparisons. A candidate is fitted to each normalized midpoint and then held fixed across its reveal trajectory. The same reference constants apply to all target conditions and methods.

\begin{table}[t]\centering\small\setlength{\tabcolsep}{3pt}
\caption{Independent calibration: shared coarse STOPs, search additions and target-cell savings over twelve conditions. V denotes a violating completion and U unresolved.}\label{tab:calibration-v81}
\begin{tabular}{lrrrrr}\toprule Method&Coarse S&Extra S&V&U&Saved\\\midrule
anchor&9&0&0&3&896\\
gradient&9&1&1&1&960\\
metric pod&9&1&1&1&960\\
metric greedy&9&1&1&1&960\\
full&9&0&1&2&896\\
\bottomrule\end{tabular}\end{table}

Nine conditions stop at the shared coarse bound. Each reduced proposal adds one STOP and one violation; full search adds the violation. All reduced methods save 960 target cells versus 896 for full search. The additional STOP occurs in pool 79011's concentrated 64-row case: upper bounds are 0.23499,0.23765,0.24194\,$\mathrm{K}^2$ for gradient, metric POD and metric greedy, versus 0.30834 for full search. First STOP times are 7.47,7.01,8.34 seconds. Pool 79023's concentrated case is unresolved for every method, and pool 79037's concentrated case is violating. 

The reference contains 512\,$\times$\,81=41,472 observed feature cells. We report that reference size separately from gross target-cell savings. The reference statistics take 0.00574 seconds to compute, outside the online certificate clock. The operational interpretation assumes an existing independent reference; acquiring that reference solely for this small target task would add its measurement cost. Public CSV loading and imports also precede the online clock. Input/candidate/anchor construction is remeasured and identically charged to all methods; each pays its own proposal/verification/search/witness work. The full queue takes 123.13 seconds. Exact record exclusions, statistics reconstructed solely from reference values, eligible deadline endpoints, original-completion coverage and the fixed half-reveal policy pass the semantic audit.

\subsection{Sensitivity of the primary independently calibrated study}\label{app:calibrated-sensitivity-v83}
This post hoc analysis uses the saved endpoints from the primary independently calibrated twelve-condition experiment, with its same fixed candidates, reference statistics, row pools and completed-evidence clocks. It reclassifies all twelve conditions over $[0,1]\,\mathrm{K}^2$ at five/ten seconds; every original endpoint and its adjacent floating values enter the grid. The primary tolerance remains $0.25\,\mathrm{K}^2$. At ten seconds the counts are 10 STOP/1 violation/1 unresolved for each reduced proposal, 9/1/2 for full search and 9/0/3 for the common coarse bound.

For the additional pool-79011 concentrated STOP, all three reduced methods satisfy STOP while full search remains unresolved on
\[
 \tau\in[0.24194363619769974,\,0.3083408101900695)\,\mathrm{K}^2.
\]
The closed decimal interval $[0.242,0.308]\,\mathrm{K}^2$ is conservative. The corresponding common-advantage interval is empty at five seconds.

\begin{figure*}[t]\centering
\includegraphics[width=\textwidth]{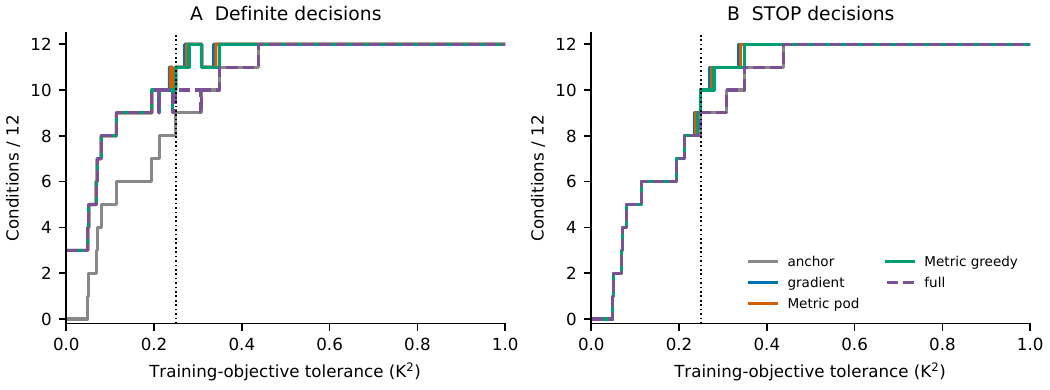}
\caption{Threshold sensitivity for the primary independently calibrated twelve-condition study, using each method's own evidence completed by ten seconds. The vertical line is the unchanged $0.25\,\mathrm{K}^2$ primary tolerance. Curves reclassify the saved primary-threshold runs; early coarse STOPs can have skipped search. They do not represent a rerun of the acquisition policy at every new tolerance.}\label{fig:calibrated-sensitivity-v83}
\end{figure*}

Scalar endpoint manifests regenerate these curves without solver execution. 

\subsection{Study settings and computing budgets}\label{app:protocol-v80}
\begin{table*}[t]\centering\small\setlength{\tabcolsep}{4pt}
\caption{Experiment settings, targets and accounting boundaries. Each study retains its own candidate, normalization and precision target.}\label{tab:protocol-v80}
\begin{tabular}{p{.19\textwidth}p{.22\textwidth}p{.25\textwidth}p{.23\textwidth}}\toprule
Study&Inputs and comparison&Targets and clock&Reported evidence\\\midrule
Structural family, 12&Synthetic full-rank tables; rank-one energy/moment&$10^{-6}$ width; 10\,s search; setup separate&12 energy successes; complete costs\\
Nuisance scale, 6&Fixed $p=65$; scalar/moment/energy&Exact rational calculation&Scale-independent contraction\\
Multidimensional structure, 18&Active ranks 2/3; residual and leakage changes&$10^{-5}$ width; 2\,s structural search; exact block refinement&18 block successes; first-stage and candidate-line controls\\
Real-world tasks, 20&Four datasets; finite-precision candidates; rank 3&$10^{-5}$ width, $2.5\times10^{-6}$ tail; 10\,s complete&20 final-width, 10 strict-tail successes\\
Precision/row controls, 24&Four tasks, three precisions, two row rules&Same 33 cells; 10\,s first stage, 600\,s refinement cap&24 refined successes; 15 centered-moment successes at the same cap\\
Independent masks, 40&Two sets of five response-independent masks per task&$10^{-5}$ width; 600\,s complete cap&40 refined successes; separate new-mask timings\\
Fixed-predictor decisions, 15&Three disjoint pools; fixed two-percent criterion&10\,s energy; shared 120\,s refinement/moment checkpoint&15 refined STOPs versus 5 moments; 330 additional cells\\
General solver, 12 runs&First mask of four tasks; all features&10/30\,s complete; scales $1,10^6$&Numerical duals and objective identity audit\\
General coarsening, 60&Three datasets; two widths and column patterns&$3\times10^{-6}$ width; 1\,s structural stage, 9000\,s complete cap&60 final-width successes; all exact replay audits\\
Dense row certificates, 4&Diabetes and one density sweep; all features&$10^{-8}$ width; 1800\,s complete cap&4 final-width successes; all exact replay audits\\
Metric controls, 18&Gradient/metric POD/greedy; fixed rank 3&10\,s complete, input included&All pairs and completed-event checkpoints\\
Acquisition, 12&Three target pools, 512 disjoint reference rows&$0.25\,\mathrm{K}^2$; 10\,s online&One additional STOP, 64 target cells\\
\bottomrule\end{tabular}\end{table*}

General-coarsening studies use source-pool means/ranges as supplied preprocessing metadata. The energy study uses the explicit scaling pools in Appendix~\ref{app:public-energy-v84}. The acquisition study uses 512 disjoint reference records and reports 41,472 reference feature cells separately. Responses remain observed and undergo each task's specified normalization. Imports precede algorithm clocks; only evidence completed by a deadline receives credit. Queue errors are recorded and subsequent conditions continue.

\section{Dense Row Geometry and Additional Controls}\label{app:extensions}
\paragraph{Generic real-data tail regime.}
We evaluate the verified $G_r,C_r,R_s,\mu$ and $\eta_r=(G_r+C_rR_s)^2/\mu$ at ranks 1,2,3,4,6 for both proposals. The source CSV reports all sixty points and the three contributions $G_r^2/\mu$, $2G_rC_rR_s/\mu$, $(C_rR_s)^2/\mu$. The gradient term is largest at 50 points and the mixed term at 10. Exact retained-Gram checks certify all sixty points; two use the inverse-norm bound stated in Appendix~\ref{app:new}.

This diagnostic evaluates the generic residual-moment certificate, with positive discarded-gradient and coupling terms. Its verified envelopes depend on the basis metric and rowwise coupling bounds, so their rank curves can be nonmonotone. The constructive observed-anchor energy study is reported in Appendix~\ref{app:public-energy-v84}.

\begin{table}[t]\centering\small
\caption{Residual-gradient rank curve for the six original conditions. Entries are verified $\eta_r$ in $10^{-6}$ units; the accompanying CSV reports every constant and both proposals.}\label{tab:rank-v77}
\resizebox{\columnwidth}{!}{\begin{tabular}{llrrrrr}\toprule Data&Pattern&$r=1$&2&3&4&6\\\midrule
Diabetes&Conc.&260.28&123.58&85.69&55.64&25.18\\
Diabetes&Disp.&123.09&109.31&108.92&95.33&59.16\\
Concrete&Conc.&14.86&8.34&5.09&3.83&1.77\\
Concrete&Disp.&11.33&10.18&8.79&7.15&4.11\\
Supercond.&Conc.&11.50&1.94&1.71&1.57&1.25\\
Supercond.&Disp.&10.87&9.23&10.10&9.09&5.58\\
\bottomrule\end{tabular}}\end{table}

A fixed 24-condition diagnostic sweep crosses Concrete and Superconductivity, both patterns, 64/128 uncertain rows and ridge multipliers $0.5,1,2$. Every row retains two uncertain cells and each candidate is refit at its own midpoint and penalty. Both rank-three proposals completed in all 24 conditions. The gradient proposal has a smaller verified tail in 24 of 24 conditions; gradient/POD tail ratios range from 0.413 to 0.970. Doubling uncertain rows multiplies the tail by 3.95--13.48 across paired methods and penalties. This sweep evaluates bound diagnostics, with no additional search or row-density comparison.

\begin{figure*}[t]
\centering
\includegraphics[width=\textwidth]{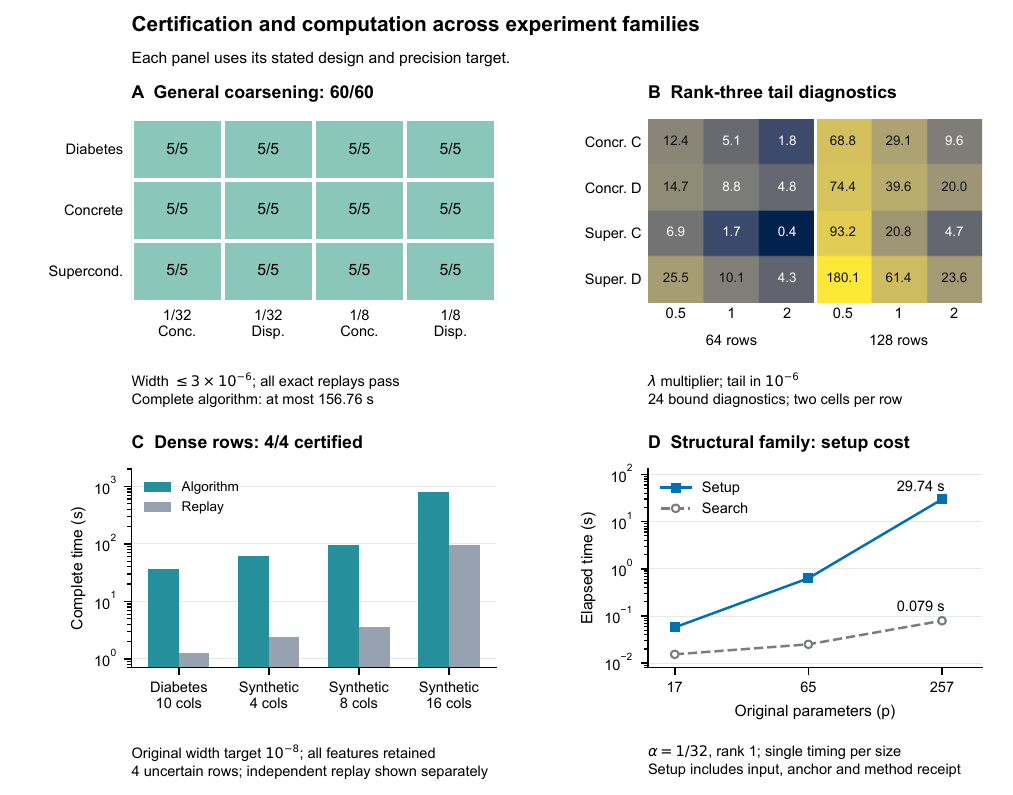}
\caption{Certification and computation across the stated experiment families. (A) All sixty general-coarsening conditions attain $3\times10^{-6}$ width, with five masks per dataset, width and pattern. (B) Verified rank-three gradient tails with ridge multipliers $0.5,1,2$ relative to $1/32$, measured in the separate twenty-four-condition diagnostic design. (C) All four dense-row conditions attain the original $10^{-8}$ width; complete algorithm and independent replay costs are shown separately. (D) Structural-family setup and search times at $\alpha=1/32$, retaining one direction; timings are single observations.}
\label{fig:scope-v79}
\end{figure*}

\subsection{Certified planar row oracle for dense uncertainty}
For a fixed retained vector, each original row contributes a difference of squared affine residuals. Its box image is a planar zonotope. We sort rationally oriented generators by cross products and inspect both endpoints and any concave interior maximizer on every boundary segment. A maintained original-box preimage supplies an exact rational feasible witness. Zero generators and collinear faces are included. Interval coefficients are enclosed using the exact midpoint optimum plus a uniform coefficient-radius perturbation bound; the rounded witness is re-evaluated with interval arithmetic. The planar adapter supplies row queries to both observed-anchor and residual-moment searches.

All twelve comparisons against independent box-edge enumeration at $k\in\{1,2,4,8\}$ agree exactly, including degenerate and interior cases. On three-repeat single-row microbenchmarks, the certified planar adapter takes median $0.693$, $0.920$, and $1.329$ ms for $k=4,8,16$; the vectorized interval box enumeration takes $15.059$ and $580.989$ ms for $k=4,8$. These timings concern row geometry.

A full-feature Diabetes instance has four rows uncertain in every feature. The density sweep uses one synthetic design with $n=128$, all sixteen features and its intercept, and four rows with $k\in\{4,8,16\}$ uncertain coordinates. The same candidate, ridge penalty and continuous intervals define the original-regret objective in every condition. The complete row-edge method of Appendix~\ref{app:edge-refinement-v22} certifies all four instances to the common $10^{-8}$ precision target.

\begin{table*}[t]\centering\small\setlength{\tabcolsep}{4pt}
\caption{Dense-row certification with all model features retained. Bounds and widths use $10^{-6}$ original-objective units; the width target is $0.01$ in these units. Endpoints are rounded outward and widths use the original certificates. The common 1800-second algorithm cap includes input construction, reference and witness searches, screening, conic proposals and exact bound evaluation. Independent replay is timed separately.}\label{tab:dense-edge-v22}
\begin{tabular}{lrrrrrrrr}\toprule Setting&$n$&$p$&Cells&Lower&Upper&Width&Algorithm (s)&Audit (s)\\\midrule
Diabetes & 442 & 11 & 40 & 2.016236 & 2.026137 & 0.00990 & 37.11 & 1.27\\
Synthetic, $k=4$ & 128 & 17 & 16 & 7.857718 & 7.867619 & 0.00990 & 61.05 & 2.40\\
Synthetic, $k=8$ & 128 & 17 & 32 & 17.765885 & 17.775786 & 0.00990 & 96.12 & 3.58\\
Synthetic, $k=16$ & 128 & 17 & 64 & 39.909574 & 39.919475 & 0.00990 & 812.99 & 96.50\\
\bottomrule\end{tabular}\end{table*}

All four independent audits reconstruct the original polynomial, the metric, the feasible completed-table witness, every coupling certificate, every screened exclusion and the complete residual partition. Certificate-guided screening reduces the row-edge search while preserving every possible target violation. The final upper remains valid for full-feature retraining.

A best-edge sweep in the planar query reconstructs the witness once, using $O(k\log k)$ rational comparisons and $O(k)$ further arithmetic operations. The reported row-query timings include witness copying during the sweep; rational bit complexity is separate.

\subsection{Leakage outside the declared residual-energy family}
This perturbation control uses the general moment verifier. Structural support and candidate leakage are verified through Theorem~\ref{thm:energy-leakage-v79} and the automatic original-row pipeline in Appendix~\ref{app:auto-energy-v80}. The grid below isolates the generic sufficient certificate on a broader perturbation grid.

For clarity, a general sufficient bound also explains this fallback. Work in the transformed coordinates of the verifier. For every admissible completion let $M\succeq\mu I$, $g$ be its candidate gradient, and $h=g^TM^{-1}g\le U$. Let $P$ project orthogonally onto the complement of a retained span. Suppose verified uniform bounds satisfy $\|Pg\|\le G$ and $\|PM(I-P)\|\le C$. Write the full optimum displacement $d=M^{-1}g=u+w$ with $u=(I-P)d$. Then $\|u\|\le\sqrt{U/\mu}=:R$, since $d^TMd=h$. For $q(d)=2g^Td-d^TMd$, completing the square in $w$ gives
\[
q(u+w)-q(u)\le {\|Pg-PMu\|^2\over\mu}
 \le {(G+CR)^2\over\mu}=:\eta.
\]
Thus the full maximum is at most the retained maximum over the radius-$R$ ball plus $\eta$. For a rounded, nonorthonormal retained basis, the implementation verifies its Gram minimum eigenvalue and enlarges the coordinate radius accordingly. The interval moment receipt computes $G,C,\mu,U$ on the actual perturbed inputs; it covers candidate leakage and signal-row leakage without requiring either to vanish. This is the generic sufficient bound used by this earlier control; the final leakage-energy guarantee has its separately verified original-row conditions.

The frozen perturbation pilot uses $n=128,p=65$, separately adding $\delta$ to feature 1 of all 64 signal rows or to candidate coefficient 1, for $\delta\in\{0,2^{-12},2^{-10},2^{-8}\}$. Ranks 1 and 2 regenerate the automatic gradient proposal at every condition; all use the same generic MomentSearch and five-second setup-plus-search cap. The complete directions, tail receipts, radii, endpoints and costs are in \texttt{structure\_robust\_results.json}; scalar results are in \texttt{structure\_robust\_summary.csv}. These controlled numerical perturbations probe the neighborhood of the constructed family.

\paragraph{Prediction deviations implied by an objective certificate.}
For a fixed completed design, let $M=X^TX/n+\lambda I$ and let $\theta^*$ be its ridge minimizer. The objective regret equals
$h=(\beta-\theta^*)^TM(\beta-\theta^*)$.
If $h\le\tau$, Cauchy--Schwarz in the $M$ inner product gives, for every feature vector $z$,
\[
|z^T(\beta-\theta^*)|\le\sqrt{\tau z^TM^{-1}z}
\le\sqrt{\tau z^TG_0^{-1}z},
\]
where the second inequality requires the verified positive-definite anchor bound $M\succeq G_0\succ0$. This controls prediction deviation from the completed-design ridge fit at a given $z$.

\subsection{Data construction and executable reproduction}
Every public input column is kept, including a penalized intercept. For the normalized public study, each feature is centered by its selected-pool mean, divided by its selected-pool range (one for constant columns), and rounded to multiples of $2^{-10}$. Responses are min--max scaled and rounded on the same grid. Diabetes and Concrete use their complete supplied records. The sixty-condition general-coarsening study samples 1024 Superconductivity records with seed 74101 after excluding the 1024-record pool defined by seed 74001. The component study uses the seed-74001 pool. Target pools apply the specified record-index exclusions before sampling; indices and seeds accompany the aggregate receipts.

For bin width $b$, a coarsened entry $x$ is replaced by $[-1+b\lfloor(x+1)/b\rfloor,-1+b\lfloor(x+1)/b\rfloor+b]$. A seeded row permutation chooses affected rows. The concentrated columns are the first two original feature columns for Concrete and Superconductivity, and the bmi/bp columns for Diabetes; dispersed patterns cycle through the first eight eligible columns. Each sparse-row record has two coarsened entries. This general-coarsening rule uses fixed feature columns and seeded row permutations; the energy experiment's residual row rule is specified in Appendix~\ref{app:public-energy-v84}. The candidate is the full-feature midpoint ridge fit with weights $1/n$ and $\lambda=1/32$. The physical decision study instead fixes $y=\mathrm{round}(2^{20}T_c/128)/2^{20}$, where $T_c$ is critical temperature in kelvin, and converts objective bounds by $128^2$.

The source archive on this preprint\textquotesingle s arXiv record contains the submission \LaTeX\ and the companion directory \texttt{anc/}. Its \texttt{README.md} gives offline checks and public-data rerun commands; \texttt{code/} contains the implementations and independent audit programs, and \texttt{results/} contains aggregate CSV/JSON records. The same versioned material is hosted at \url{https://github.com/Peida-Lee/certifying-retraining-regret}. Raw records and table-containing receipts remain local and can be regenerated from public sources. Saved scalar endpoints regenerate figures without solver execution. The new response-independent masks and uniform general-coarsening refinement charge input construction, including their loader calls. Reruns of stored primary, control and decision tables charge fresh anchor, direction, bound and search computations from those binary inputs. Other studies retain the timing boundaries stated in their protocols. Interpreter imports precede algorithm clocks. Timer checks occur between atomic operations, and only completed evidence is credited at deadlines.

General-coarsening studies treat complete-pool scaling as supplied metadata. The acquisition experiment freezes means/ranges from 512 reference rows disjoint from all target and specified excluded pools before target coarsening. Its 41,472 reference feature cells are reported separately from target-cell savings. Both settings retain observed responses and explicit information boundaries.

\subsection{Certificates and independent checks}\label{app:verification-ledger}
Table~\ref{tab:verification-ledger} maps each reported study to its independent checks, including original-row witnesses and complete covering trees.
\begin{table*}[t]\centering\small\setlength{\tabcolsep}{4pt}
\caption{Correspondence between reported evidence and independent verification. Each complete-tree check reconstructs its root, original-row query enclosures, node bounds, binary coverage, terminal upper and feasible original-table witness.}\label{tab:verification-ledger}
\begin{tabular}{p{.23\textwidth}p{.46\textwidth}p{.22\textwidth}}\toprule
Evidence&Independently checked quantities&Complete trees\\\midrule
Core row/interval suite&26 original-row cases, four sharpness controls; actual rounded metrics/projectors and original normal equations&Five synthetic/public trees\\
Automatic energy&Actual partition, direction, radius, leakage constants and original witnesses&Six trees across the fixed perturbation grid\\
Multidimensional structure&Original independent boxes, rational signal coefficients and exact star-block elimination&All eighteen block trees; two energy trees\\
Real-world energy&All twenty anchors, projectors, mixed moments, contraction constants, containing radii and original feasible witnesses&Four, first mask per dataset\\
Joint residual refinement&Original polynomial coefficients, actual metric, nonnegative dyadic duals, exact PSD factors and feasible witnesses&Complete covers for forty independent masks, forty primary/controls and fifteen decisions\\
Dense row-edge refinement&Original joint polynomial, coupling PSD, exact edge exclusions and feasible witnesses&All four complete residual-partition trees\\
General coarsening&Geometry-based representation choice, original objective, exact row/polynomial queries and complete timing&All sixty original-table certificates and complete covers\\
Metric proposal controls&Transformed displacements, POD spectra, 54 greedy selections and 162 completed-event checkpoints&Proposal/deadline audit\\
Acquisition&Reference/target disjointness, calibration statistics, interval coverage and fixed reveal-policy counts&Semantic/policy audit\\
SCIP&Eight exact original/compressed-objective identities&Objective identity audit\\
\bottomrule\end{tabular}\end{table*}

\section{Real-World Experiment Settings}\label{app:public-energy-v84}
Superconductivity records measured critical temperatures; Yacht records measured residuary resistance; Wine Quality records sensory quality scores and physicochemical tests. Parkinsons Telemonitoring contains home voice recordings, demographic variables and dataset-provided total UPDRS scores, which the source defines by linear interpolation of clinical assessments. The identifier and the alternative motor-UPDRS response are excluded from its predictors; all 19 documented predictor columns are retained. Binary sex remains observed. These tasks keep their supplied regression responses, followed by affine normalization and $2^{-10}$ quantization.

The candidate is the midpoint ridge fit rounded to four fractional binary bits. One signal row has one unknown and sixteen nuisance rows have two unknowns each, giving 33 uncertain cells. The signal and nuisance column lists are fixed in the supplied settings. The signal row has the largest pilot residual. After three fixed passes, nuisance rows are a seeded sample from the 32 smallest residual-interval scores under the observed-row reference. The pilot uses coarsened candidate-column values. Every compared method uses the same final row rule, candidate and continuous feature boxes.

\begin{table*}[t]\centering\small
\caption{Real-world parameters. Widths are fractions of the feature range. All candidates use four fractional bits and the retained rank is three.}
\begin{tabular}{lrrrrr}\toprule
Data & $n$ & $p$ & $\lambda$ & Signal width & Nuisance width\\\midrule
Superconductivity&1024&82&$2^{-10}$&$1/2$&$1/2$\\
Parkinsons&2048&20&$2^{-10}$&$1/2$&$1/4$\\
Wine Quality (white)&2048&12&$2^{-10}$&$1/2$&$1/4$\\
Yacht&308&7&$2^{-5}$&$1/2$&$1/4$\\\bottomrule
\end{tabular}\end{table*}

Each task uses mask seeds 88211--88215. Superconductivity retains its target pool 84101 and scaling pool 84001, with the specified earlier record-pool exclusions. Parkinsons and Wine use scaling pool 88101 and disjoint target pool 88201, each with 2048 records; means and ranges come from the scaling pool. Yacht uses all 308 source records. Parkinsons record pools describe recordings from the fixed cohort; the certification experiment makes a deterministic statement about these tables. The seed lists, column indices, record-pool construction and numeric settings are supplied with the code.

Each method has a ten-second complete budget, with a $10^{-10}$ search tolerance to resolve its attainable width. We report success against a common final-width target of $10^{-5}$ and a separate energy-tail budget of $2.5\times10^{-6}$. Adaptive energy uses the smaller verified contraction/common-domain moment tail; centered moment verifies its own reference gradient and radius in the same actual rank-three span. Full search retains all model coordinates. All three methods receive the same feasible midpoint witness and the smaller gradient/observed-loss cap. The complete clock charges input/candidate construction, anchor, common caps/witness, the method's required proposal and verification, setup, search and full-feature witness fits. Imports and certificate serialization are outside the clock; source-data reads in input construction are charged. Moment pays reference construction, and full search has no reduced-space proposal or tail verification. All sixty method-condition runs and the original pure-energy component results are supplied. The CSV reports any small indivisible-node overrun.

An independent Fraction program imports no proposal, verifier or search implementation. It reconstructs the observed anchors, actual metrics, retained projectors and references; checks the original-row constants, mixed residual moments, projected center-Gram operator bound, discarded perturbation trace and containing radius; and verifies every feasible original-table witness. The first mask of each dataset additionally receives a complete search-tree audit of all row-query enclosures, curvature, binary coverage and terminal uppers. The independent adaptive audit additionally verifies its selected moment branch by exact projected gradient and coupling reconstruction. Forty-eight exact normal-equation examples separately verify the refined theorem under nonzero metric cross blocks, nonzero residuals and retained translations.

\begin{table*}[t]\centering\small
\caption{Complete computation cost. Per-method medians over five masks, in seconds. Setup includes input, anchor, caps, proposal and verification; search includes original-table witness fits.}
\label{tab:contraction-costs-v88}
\begin{tabular}{lrrrrrrr}\toprule
&\multicolumn{3}{c}{Adaptive energy}&\multicolumn{3}{c}{Centered moment}&Full\\
Data&Setup&Search&Total&Setup&Search&Total&Total\\\midrule
Supercond. & 2.653 & 4.211 & 6.823 & 1.884 & 4.428 & 6.287 & 10.021\\
Parkinsons & 0.470 & 3.553 & 4.060 & 0.369 & 3.797 & 4.156 & 10.003\\
Wine Quality & 0.329 & 4.216 & 4.601 & 0.220 & 4.089 & 4.314 & 10.002\\
Yacht & 0.156 & 5.119 & 5.289 & 0.086 & 5.039 & 5.125 & 10.004\\
\bottomrule\end{tabular}\end{table*}

\subsection{Tail components and original-regret levels}
Table~\ref{tab:tail-components-v89} separates the contraction, retained-radius and residual-envelope contributions; Table~\ref{tab:regret-levels-v89} reports the corresponding original-regret levels.
\begin{table*}[t]\centering\small
\caption{Tail components on identical primary instances; medians over five masks, in $10^{-6}$ units. Contraction changes $\gamma$ at the original $R_A$. Radius then uses the mixed-moment radius and the independently verified full-gradient restriction. The last column adds the mixed residual-energy envelope in~\eqref{eq:contraction-main-v89}. All quantities use saved verified constants; these are bound evaluations with no new solver searches.}
\label{tab:tail-components-v89}
\begin{tabular}{lrrrr}\toprule
Data&Original&Contraction&+Radius&+Mixed envelope\\\midrule
Supercond. & 7.476 & 5.547 & 5.547 & 4.453\\
Parkinsons & 3.066 & 0.953 & 0.953 & 0.904\\
Wine Quality & 1.867 & 0.143 & 0.142 & 0.134\\
Yacht & 30.287 & 4.132 & 3.879 & 2.788\\
\bottomrule\end{tabular}\end{table*}
\begin{table}[t]\centering\small
\caption{Original-regret magnitude. Marginal medians of energy lower and upper endpoints over five masks, in $10^{-3}$ normalized squared-response units. Numerical width and model-retention tolerance are distinct quantities.}
\label{tab:regret-levels-v89}
\begin{tabular}{lrr}\toprule Data&Median $L$&Median $U$\\\midrule
Supercond. & 1.689359 & 1.693478\\
Parkinsons & 0.828699 & 0.829603\\
Wine Quality & 0.426868 & 0.426990\\
Yacht & 0.761605 & 0.764284\\
\bottomrule\end{tabular}\end{table}

\subsection{Matched candidate precision and nuisance rows}\label{app:precision-controls}
The first primary mask (88211) fixes the signal interval and nuisance-column list. Sixteen nuisance rows follow the primary residual rule or uniform sampling excluding the signal row. Unrounded, eight-bit and four-bit candidates share each mask, its 33 uncertain cells and rank three. The four-bit residual conditions reuse eight primary runs; forty new runs complete each matched two-method grid. The pure-energy grid and adaptive grid share the same twenty-four boxes and candidates. C, A and M in Table~\ref{tab:matched-v89} use completed evidence within the original ten-second budgets. R reruns all algorithm stages and refines unresolved cases under a 600-second cap. All twenty-four matched conditions meet width $10^{-5}$, with a maximum R cost of 459.69 seconds. The main text's common-600-second comparison uses the extended moment runs in \texttt{controls\_baseline\_extended.json}, restricted to these same twenty-four conditions; they meet the target in 15/24. The companion \texttt{extend\_validation\_baselines.py} reproduces the extended control, independent-mask and decision baselines, including their completed-event checkpoints.

\begin{table*}[t]\centering\small\setlength{\tabcolsep}{4pt}
\caption{Matched controls in $10^{-6}$ objective units. C/A/M are the original ten-second contraction/adaptive/moment stages. R adds refinement with a 600-second cap; its measured seconds start at the stored binary instance. R/U identify residual/uniform nuisance rows, $\infty$ an unrounded candidate, and cap a successful coarse bound.}
\label{tab:matched-v89}
\begin{tabular}{llrrrrrrr}\toprule
Data&Rows&Bits&Cap&C width&A width&M width&R width&R seconds\\\midrule
Supercond. & R & $\infty$ & yes & 1.791 & 1.791 & 1.791 & 1.791 & 10.04\\
Supercond. & R & 8 & yes & 3.726 & 3.726 & 4.228 & 3.726 & 9.83\\
Supercond. & R & 4 & -- & 3.911 & 3.911 & 35.711 & 3.911 & 10.05\\
Supercond. & U & $\infty$ & -- & 75.400 & 75.400 & 75.400 & 8.04 & 459.68\\
Supercond. & U & 8 & -- & 90.952 & 90.952 & 90.938 & 8.349 & 220.10\\
Supercond. & U & 4 & -- & 203.117 & 203.117 & 321.280 & 1.122 & 87.00\\
Parkinsons & R & $\infty$ & yes & 0.491 & 0.491 & 0.491 & 0.4909 & 6.70\\
Parkinsons & R & 8 & yes & 0.924 & 0.924 & 1.258 & 0.9241 & 5.40\\
Parkinsons & R & 4 & -- & 0.904 & 0.904 & 1.537 & 0.904 & 8.89\\
Parkinsons & U & $\infty$ & -- & 14.726 & 14.726 & 14.726 & 2.208 & 60.93\\
Parkinsons & U & 8 & -- & 18.386 & 18.386 & 19.159 & 2.104 & 70.09\\
Parkinsons & U & 4 & -- & 52.133 & 41.323 & 49.379 & 0.2721 & 59.20\\
Wine Quality & R & $\infty$ & yes & 0.119 & 0.119 & 0.119 & 0.1189 & 6.68\\
Wine Quality & R & 8 & yes & 0.121 & 0.121 & 0.146 & 0.1208 & 6.07\\
Wine Quality & R & 4 & -- & 0.122 & 0.122 & 0.153 & 0.1222 & 9.70\\
Wine Quality & U & $\infty$ & yes & 1.371 & 1.371 & 1.371 & 1.371 & 8.87\\
Wine Quality & U & 8 & yes & 2.502 & 2.502 & 2.502 & 2.502 & 7.29\\
Wine Quality & U & 4 & -- & 19.028 & 8.657 & 8.742 & 8.657 & 6.87\\
Yacht & R & $\infty$ & yes & 0.669 & 0.669 & 0.669 & 0.6688 & 9.81\\
Yacht & R & 8 & yes & 2.410 & 2.410 & 2.410 & 2.41 & 10.01\\
Yacht & R & 4 & -- & 3.296 & 3.296 & 4.304 & 3.296 & 10.01\\
Yacht & U & $\infty$ & -- & 8.950 & 8.950 & 8.950 & 8.95 & 10.01\\
Yacht & U & 8 & -- & 14.085 & 14.085 & 14.085 & 0.2487 & 82.36\\
Yacht & U & 4 & -- & 81.405 & 79.666 & 79.671 & 0.001598 & 15.52\\
\bottomrule\end{tabular}\end{table*}
Unrounded residual-guided models meet the target at the common cap in all four tasks. The four-bit residual condition isolates the cost of retaining a finite-precision candidate. Uniform-row controls test the same algorithms with larger nuisance residuals: common-domain adaptation selects the moment tail when it is stronger. The additive-tail inequality in~\eqref{eq:adaptive-v90} applies to the common search domain. Final interval widths also depend on completed search and feasible witnesses. Forty fresh adaptive constants/witness checks include the twenty primary conditions and twenty fresh controls; four primary trees are checked completely. Exact input-box equality preserves the earlier fixed-mask checks across candidate precisions.

\subsection{Additional general nonlinear-solver controls}\label{app:scip-followup}
These additional controls use Concrete, Superconductivity, Airfoil and Power Plant, with mask 84111 for each task. We compare SCIP using all features, the same candidate, ridge penalty and 33 independent uncertain cells. These are the first masks in the fixed mask list. The original regret is represented as a joint maximum over cell deviations and the retraining displacement; observed rows are aggregated into their exact normal equations. Centering the displacement at the observed-anchor solution avoids cancellation between large fixed-row losses. A shared feasible midpoint witness initializes the incumbent and a shared gradient/observed-loss cap bounds the result.

Each run uses one CPU thread, includes common input/anchor/witness and model construction within its total budget, and gives SCIP the remaining time. The main ten/thirty-second comparison scales the objective by $10^6$ and scales the $10^{-5}$ absolute-gap target accordingly. Ten-second unscaled runs check numerical scaling. Feasibility and dual-feasibility tolerances are $10^{-9}$. SCIP endpoints are numerical dual estimates under these tolerances. Our outward-arithmetic endpoints and independently checked original-table witnesses use the separate certification procedure.

\begin{table*}[t]\centering\small\setlength{\tabcolsep}{5pt}
\caption{All twelve SCIP runs. Widths are in $10^{-6}$ normalized squared-response units. The first two columns use objective scale $10^6$; the third uses scale 1. All initial solutions were accepted.}\label{tab:scip-followup}
\begin{tabular}{lrrr}\toprule
Data&10\,s scaled&30\,s scaled&10\,s unscaled\\\midrule
Concrete & 9.583 & 9.583 & 7.980\\
Supercond. & 36.305 & 36.305 & 151.492\\
Airfoil & 7.012 & 7.012 & 7.012\\
Power Plant & 6.194 & 6.194 & 6.195\\
\bottomrule\end{tabular}\end{table*}

Concrete, Airfoil and Power Plant reach the numerical width target; Superconductivity reaches a time limit at both budgets, with scaled width $36.305\times10^{-6}$. An independent eight-point Fraction audit reconstructs the original loss difference and the centered compressed objective separately, obtaining exact equality. It imports no SCIP model builder; float coefficient evaluations differ from the exact identity by about $10^{-19}$. These checks validate the compared objective independently of the solver's numerical dual guarantee.

\section{Matched External Solvers and Regret Decomposition}\label{app:external-v90}
All methods use the primary first mask 88211, unchanged $\beta$, $\lambda$, boxes, midpoint feasible witness and common cap. Input/candidate construction reuses its measured primary cost; each new external run measures its own anchor, cap, model construction and solver costs. SCIP 10.0.2 through PySCIPOpt 6.3.0 uses objective scale $10^6$, absolute objective tolerance $10^{-10}$, primal and dual feasibility tolerances $10^{-9}$ and one thread. The original polynomial regret formulation and exact identity audit are in Appendix~\ref{app:scip-followup}. Its tolerance-based numerical dual is listed separately from outward arithmetic. A second target-$10^{-5}$ queue is included in the source. Small indivisible-call overruns are included in actual total cost.

\begin{table*}[t]\centering\small
\caption{Same-instance endpoints. $L,U$ use $10^{-3}$ units; widths use $10^{-6}$. All costs are complete seconds.}
\label{tab:external-endpoints-v90}
\begin{tabular}{llrrrr}\toprule Data&Method&$L$&$U$&Width&Seconds\\\midrule
Supercond. & Adaptive energy & 1.690289243 & 1.694199943 & 3.910700 & 7.096\\
Supercond. & Centered moment & 1.690288699 & 1.725999809 & 35.711110 & 6.287\\
Supercond. & SCIP numerical & 1.690299779 & 1.850295633 & 159.995854 & 9.992\\
Parkinsons & Adaptive energy & 0.828699094 & 0.829603046 & 0.903953 & 3.913\\
Parkinsons & Centered moment & 0.828699094 & 0.830236015 & 1.536921 & 4.102\\
Parkinsons & SCIP numerical & 0.828699094 & 0.829491585 & 0.792491 & 9.977\\
Wine Quality & Adaptive energy & 0.426867749 & 0.426989908 & 0.122158 & 4.808\\
Wine Quality & Centered moment & 0.426867749 & 0.427020937 & 0.153188 & 4.744\\
Wine Quality & SCIP numerical & 0.426867749 & 0.427110887 & 0.243138 & 9.990\\
Yacht & Adaptive energy & 0.760845932 & 0.764142010 & 3.296078 & 5.251\\
Yacht & Centered moment & 0.760845932 & 0.765150401 & 4.304469 & 5.590\\
Yacht & SCIP numerical & 0.760845946 & 0.760847427 & 0.001482 & 9.969\\
\bottomrule\end{tabular}\end{table*}
\subsection{ZORRO model envelopes converted to the same regret}
ZORRO~\citep{zhu2024zorro}, Theorem 4.2, encloses all possible optimal ridge models at the specified regularization. We implement its Appendix I equations (22), (23), (39) for observed responses and independently uncertain feature cells, with its Appendix J box splitting. The eigenvector transformation is rounded to a dyadic invertible matrix $B$; an outward inverse enclosure covers its actual inverse. Writing $M_0=X_c^\top X_c/n+\lambda I$, the real component is $w_R=M_0^{-1}X_c^\top y/n$ and the shared-symbol component uses the derivative of the exact ridge equations. Rounding residuals remain in the fixed-point forcing. High-order data monomials are collected before absolute coefficients are summed, following ZORRO's linearization.

The resulting joint enclosure has $\theta=w_R+D\xi+B\operatorname{diag}(k)\zeta$, with the original data symbols $\xi$ shared with $X(\xi)$ and independent nondata symbols $\zeta$. Its nonnegative fixed-point radii satisfy an outward verified M-matrix inequality. For each possible optimal model,
\begin{align*}
 d(\xi,\zeta)&=\beta-w_R-D\xi-B\operatorname{diag}(k)\zeta,\\
 h_\beta(X)&=g_X^\top(\beta-\theta_X^*)
 \le\sup_{\xi,\zeta\in[-1,1]}g_{X(\xi)}^\top d(\xi,\zeta).
\end{align*}
We collect the adapter's degree-at-most-three polynomial coefficients across rows, preserve shared data symbols, bound every even monomial on $[0,1]$ and remaining monomials on $[-1,1]$. The smaller of this valid upper and the common cap is returned with the same initial feasible lower. This measures one explicit polynomial adapter of the model envelope. We use the paper's ridge equations; the inspected repository notebook is an implementation reference. The original $\lambda$ is unchanged throughout splitting.

\begin{table*}[t]\centering\small
\caption{ZORRO fixed-point/adapter controls on the same primary instances. Widths use $10^{-6}$ units. A complete cover is required to credit its model-derived upper. Incomplete cover returns the shared cap. Cost includes all attempted splits and the adapter.}
\label{tab:zorro-v90}
\begin{tabular}{lrrrrrr}\toprule Data&Complete cover&Attempts&Completed boxes&Unresolved boxes&Width&Seconds\\\midrule
Supercond. & -- & 19 & 0 & 20 & 185.998 & 10.329\\
Parkinsons & yes & 1 & 1 & 0 & 17.134 & 0.798\\
Wine Quality & yes & 1 & 1 & 0 & 4.503 & 0.547\\
Yacht & yes & 1 & 1 & 0 & 19.788 & 0.278\\
\bottomrule\end{tabular}\end{table*}
The complete unsplit envelopes for Parkinsons, Wine and Yacht take less than one complete second each. The Superconductivity queue attempts nineteen boxes and leaves twenty unresolved at the cap. The independent rational program reconstructs the three complete fixed-point systems and their shared-symbol regret adapters from the original row boxes. All three pass exact coefficient, nonnegative-radius and adapter-upper checks. Thirty-two sampled original normal-equation solutions per complete envelope additionally check the implementation.

\subsection{Midpoint regret and uncertainty increment}
The midpoint quantity is $h_\beta(X_c)$. We enclose it by a feasible comparator objective and a verified normal-equation residual correction. If $\widetilde\theta$ is its numerical ridge solution and $r_c=M_c\widetilde\theta-X_c^\top y/n$, then
\[
 J_c(\beta)-J_c(\widetilde\theta)\le h_\beta(X_c)
 \le J_c(\beta)-J_c(\widetilde\theta)+\|T^\top r_c\|^2/\mu.
\]
Subtracting the midpoint interval from $[L,U]$ encloses the nonnegative uncertainty increment. Midpoint regret records the fixed candidate's representation cost; the increment records additional worst-completion cost for that same candidate. Componentwise medians below summarize five fixed masks; the source retains every paired interval.
\begin{table}[t]\centering\small
\caption{Median midpoint regret and median increment bounds, all in $10^{-6}$ normalized-objective units.}
\label{tab:increments-v90}
\begin{tabular}{lrrr}\toprule Data&Midpoint&Increment $L$&Increment $U$\\\midrule
Supercond. & 1664.298 & 23.232 & 27.351\\
Parkinsons & 817.937 & 10.773 & 11.677\\
Wine Quality & 424.515 & 2.323 & 2.461\\
Yacht & 744.339 & 17.220 & 19.789\\
\bottomrule\end{tabular}\end{table}

\subsection{Full-budget model-envelope refinement and every primary SCIP mask}\label{app:refined-external-v91}
All twenty primary masks receive the same-objective SCIP control. The four existing first-mask runs are reused and sixteen additional runs use unchanged candidates, boxes, regularization and initial midpoint witnesses. A ten-second nominal complete budget and $10^{-10}$ search tolerance match the primary search. All seed-level endpoints and actual costs are supplied.

ZORRO's valid unsplit fixed point is also retained as a parent bound. Within the same complete ten-second cap, a best-first queue refines the box with the largest upper. Its regret-polynomial coefficient sensitivities select a split coordinate once a model envelope exists; the fixed-point existence stage uses its geometric coordinate score. Both children inherit their legal parent's upper immediately. A completed child envelope can tighten that bound; incomplete work leaves the inherited upper available. All data/model shared symbols and the original $\lambda$ remain in the conversion. Original-table midpoint fits update feasible lowers. This gives a valid complete cover at every time, including partial refinement. The original single-envelope results remain in Table~\ref{tab:zorro-v90}.

\begin{table*}[t]\centering\small
\caption{Full-budget ZORRO refinement on the same four first-mask instances. Widths use $10^{-6}$ units. Every attempted call enters actual complete cost; the plot credits completed evidence by its timestamp.}
\label{tab:zorro-refined-v91}
\begin{tabular}{lrrrrr}\toprule Data&Single-envelope width&Refined width&Completed envelopes&Complete seconds&Model-envelope cover\\\midrule
Supercond. & 185.998 & 185.998 & 0 & 10.058 & --\\
Parkinsons & 17.134 & 10.063 & 15 & 10.047 & yes\\
Wine Quality & 4.503 & 3.200 & 24 & 10.339 & yes\\
Yacht & 19.788 & 2.014 & 41 & 10.183 & yes\\
\bottomrule\end{tabular}\end{table*}
The independent rational audit checks all eighty model-envelope bounds used by refinement on the three completed-model tasks, each shared-symbol regret conversion, every parent-bound inheritance, and complete binary partition coverage. Superconductivity retains the previously checked common cap. Figure~\ref{fig:zorro-time-v91} compares completed-width trajectories, including each method's full setup and candidate/input cost.
\begin{figure*}[t]\centering\includegraphics[width=.97\textwidth]{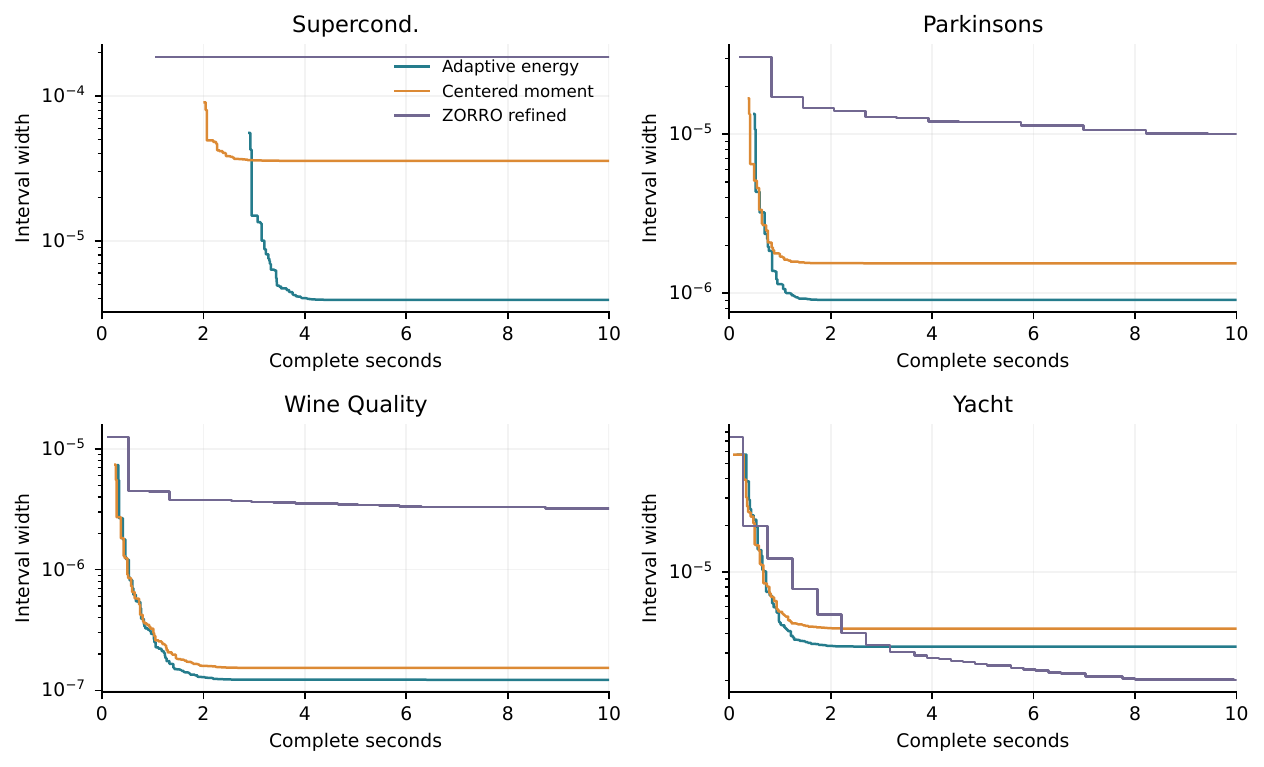}
\caption{Same-instance width against complete computation time. A and M keep their numerical/tail-floor endpoints once finished. Refined ZORRO preserves the parent's valid upper during incomplete child work and uses the full cap. All widths are in normalized objective units.}
\label{fig:zorro-time-v91}\end{figure*}

\section{Independent Row Masks and Core-Energy Decisions}\label{app:core-decisions-v91}
The frozen design file records masks, methods, time caps, response-independent row sampling and the decision rule before the independent-row and new-pool decision runs. The random-row study retains the four primary column lists and widths. It uniformly samples seventeen distinct rows, using the first as signal and the rest as nuisance rows, without responses or residual scores. The candidate is the four-fractional-bit midpoint ridge model. All twenty instances retain $|J_{\mathcal S}|=1$ and the constructive rank-three span. The first-stage sufficient tails and ranks are stored with each receipt. The joint refinement uses the same original boxes and full comparator. It attains width $10^{-5}$ on all twenty original masks, within 255.07 seconds, and all twenty new masks, within 76.21 seconds. Centered moments on the new masks receive the same 600-second cap and terminate at their tail/arithmetic floors. All new cases finish by 120 seconds with refinement; moments meet the target in five.

\begin{table*}[t]\centering\small
\caption{All original and new response-independent masks. Widths use $10^{-6}$ units. Every row uses a rank-three structural first stage A. R adds joint residual refinement; M is centered moments. Original-set M columns are ten-second first-stage records; new-set M uses the common 600-second cap. Original instances restart from their stored binary input; new-mask costs include input construction.}\label{tab:independent-details-v91}
\begin{tabular}{llrrrr}\toprule Data&Seed&A width&R width&M width&R seconds\\\midrule
Supercond. & 91311 & 275.5 & 7.706 & 312.7 & 62.32\\
Supercond. & 91312 & 157.1 & 2.227 & 202.5 & 87.43\\
Supercond. & 91313 & 291.2 & 7.927 & 412.4 & 255.06\\
Supercond. & 91314 & 159.7 & 7.173 & 185.6 & 91.09\\
Supercond. & 91315 & 214.8 & 1.256 & 252.1 & 90.26\\
Parkinsons & 91311 & 27.96 & 0.009376 & 34.56 & 99.38\\
Parkinsons & 91312 & 44.03 & 0.001401 & 56.99 & 57.29\\
Parkinsons & 91313 & 66.33 & 0.00286 & 86.19 & 35.84\\
Parkinsons & 91314 & 52.71 & 0.00234 & 65.16 & 21.22\\
Parkinsons & 91315 & 50.62 & 0.0004626 & 68.75 & 51.57\\
Wine Quality & 91311 & 3.074 & 3.074 & 3.167 & 10.01\\
Wine Quality & 91312 & 3.051 & 3.051 & 3.051 & 8.72\\
Wine Quality & 91313 & 8.444 & 8.444 & 8.444 & 10.03\\
Wine Quality & 91314 & 9.452 & 9.452 & 9.83 & 9.66\\
Wine Quality & 91315 & 4.678 & 4.678 & 5.081 & 6.82\\
Yacht & 91311 & 87.79 & 7.614e-05 & 87.79 & 18.27\\
Yacht & 91312 & 29.7 & 0.009634 & 29.7 & 20.06\\
Yacht & 91313 & 42.41 & 0.001967 & 42.41 & 69.49\\
Yacht & 91314 & 34.11 & 0.001701 & 35.49 & 97.67\\
Yacht & 91315 & 17.43 & 0.002385 & 17.43 & 17.46\\
\midrule
Supercond. & 92411 & 294.7 & 3.174 & 328.3 & 76.20\\
Supercond. & 92412 & 171.4 & 2.715 & 226.2 & 57.06\\
Supercond. & 92413 & 211.3 & 8.08 & 254.4 & 65.55\\
Supercond. & 92414 & 177.8 & 0.2004 & 201.5 & 59.27\\
Supercond. & 92415 & 292.2 & 6.931 & 365.8 & 72.79\\
Parkinsons & 92411 & 40.76 & 0.008721 & 53.53 & 24.64\\
Parkinsons & 92412 & 37.54 & 0.2545 & 47.61 & 56.21\\
Parkinsons & 92413 & 45.09 & 0.007161 & 57.95 & 52.29\\
Parkinsons & 92414 & 60.84 & 0.1055 & 60.84 & 48.53\\
Parkinsons & 92415 & 66.94 & 0.5991 & 72.34 & 55.60\\
Wine Quality & 92411 & 8.039 & 8.039 & 8.039 & 6.03\\
Wine Quality & 92412 & 9.008 & 9.008 & 9.008 & 6.84\\
Wine Quality & 92413 & 4.938 & 4.938 & 5.242 & 4.59\\
Wine Quality & 92414 & 4.273 & 4.273 & 4.273 & 8.27\\
Wine Quality & 92415 & 6.321 & 6.321 & 7.278 & 4.57\\
Yacht & 92411 & 25.39 & 0.0008197 & 25.39 & 9.92\\
Yacht & 92412 & 75.8 & 0.001748 & 75.8 & 13.85\\
Yacht & 92413 & 39.26 & 0.0008925 & 43.87 & 9.78\\
Yacht & 92414 & 24.76 & 0.001942 & 25.23 & 67.88\\
Yacht & 92415 & 13.84 & 0.00289 & 14.39 & 9.88\\
\midrule
\bottomrule\end{tabular}\end{table*}
\subsection{Disjoint pools, a stored predictor and the relative objective budget}
Using the original normalization, we exclude the 1024-row pool 74001, the 1024-row scaling pool 84001 and the primary development pool 84101. A seeded permutation of remaining row indices defines three mutually disjoint 1024-row evaluation pools. The stored coefficient vector is the development first-mask predictor and is held fixed across all new pools and masks. Pool construction and its disjointness checks are recorded locally; its exact sampler and seed are in the portable source. The pool preparation cost is charged conservatively in each condition's input cost.

Rule~\eqref{eq:core-decision-v91} is a model-retention contract for additional uncertainty cost, relative to the current encoded model's nominal representation cost. It is fixed at $\alpha=1/50$ before the new runs. With midpoint interval $[l_c,u_c]$, STOP uses $U\le(1+\alpha)l_c$ and violation uses $L>(1+\alpha)u_c$. Endpoint products and labels are independently checked as exact rationals. The nominal comparator fit includes a verified normal-equation residual correction. Only decisions completed by ten seconds receive credit; any last-call overrun is retained in the timing record. The fixed policy reveals all thirty-three uncertain cells for every non-STOP case and none for STOP. Costs are charged independently to each audit instance and summed over the fixed grid.

\begin{table*}[t]\centering\small
\caption{Core-energy decisions across three disjoint pools. Stops are by the ten-second complete cap. Saved cells follow the fixed thirty-three-cell policy. Median time is conditional on STOP.}
\label{tab:core-decisions-v91}
\begin{tabular}{lrrrrrr}\toprule Method&Pool 1&Pool 2&Pool 3&Total STOP&Saved cells&Median STOP seconds\\\midrule
Coarse & 0 & 0 & 0 & 0 & 0 & --\\
Energy & 5 & 5 & 3 & 13 & 429 & 3.098\\
Adaptive & 5 & 5 & 3 & 13 & 429 & 3.067\\
Moment & 0 & 5 & 0 & 5 & 165 & 2.199\\
SCIP numeric & 0 & 0 & 0 & 0 & 0 & --\\
ZORRO refined & 0 & 0 & 0 & 0 & 0 & --\\
\bottomrule\end{tabular}\end{table*}
\begin{table*}[t]\centering\small
\caption{All new condition labels at the primary 2\% budget. E/A/M are contraction energy/adaptive energy/centered moment. S/U denote STOP/unresolved; all coarse, SCIP and refined-ZORRO labels are U. Midpoint regret and budget widths use $10^{-6}$ units.}
\begin{tabular}{llrrrrrr}\toprule Pool&Seed&Midpoint&E width&M width&E&A&M\\\midrule
1 & 91411 & 2397.494 & 4.341 & 54.203 & S & S & U\\
1 & 91412 & 2398.536 & 3.652 & 42.800 & S & S & U\\
1 & 91413 & 2396.162 & 4.290 & 56.023 & S & S & U\\
1 & 91414 & 2394.529 & 3.799 & 39.834 & S & S & U\\
1 & 91415 & 2394.600 & 3.945 & 45.031 & S & S & U\\
2 & 91411 & 2349.725 & 3.624 & 25.013 & S & S & S\\
2 & 91412 & 2348.823 & 2.958 & 18.873 & S & S & S\\
2 & 91413 & 2349.773 & 3.927 & 26.123 & S & S & S\\
2 & 91414 & 2349.700 & 3.507 & 24.502 & S & S & S\\
2 & 91415 & 2350.506 & 3.287 & 23.409 & S & S & S\\
3 & 91411 & 2751.258 & 8.948 & 121.385 & S & S & U\\
3 & 91412 & 2761.790 & 11.413 & 139.351 & U & U & U\\
3 & 91413 & 2753.522 & 8.993 & 117.779 & S & S & U\\
3 & 91414 & 2765.084 & 10.214 & 118.237 & U & U & U\\
3 & 91415 & 2760.569 & 10.885 & 133.867 & S & S & U\\
\bottomrule\end{tabular}\end{table*}
\begin{table*}[tp]\centering\small
\caption{All predeclared relative-budget sensitivities, reported as STOP counts out of fifteen.}
\begin{tabular}{lrrr}\toprule Method&1\%&2\%&5\%\\\midrule
Coarse & 0 & 0 & 0\\
Energy & 5 & 13 & 15\\
Adaptive & 5 & 13 & 15\\
Moment & 0 & 5 & 10\\
SCIP numeric & 0 & 0 & 2\\
ZORRO refined & 0 & 0 & 0\\
\bottomrule\end{tabular}\end{table*}

\begin{table*}[t]\centering\small
\caption{Precision refinement of the same fifteen decision instances. The original 2\% rule and predictor are unchanged. Widths use $10^{-6}$ units. R seconds measure all recomputed algorithm stages from the stored binary instance.}\label{tab:decisions-refined-v92}
\begin{tabular}{llrrrr}\toprule Pool&Seed&R width&R seconds&2\% label&Saved cells\\\midrule
1 & 91411 & 4.342 & 10.07 & STOP & 33\\
1 & 91412 & 3.652 & 10.04 & STOP & 33\\
1 & 91413 & 4.29 & 10.05 & STOP & 33\\
1 & 91414 & 3.799 & 10.05 & STOP & 33\\
1 & 91415 & 3.945 & 10.04 & STOP & 33\\
2 & 91411 & 3.624 & 9.84 & STOP & 33\\
2 & 91412 & 2.958 & 9.61 & STOP & 33\\
2 & 91413 & 3.927 & 9.89 & STOP & 33\\
2 & 91414 & 3.507 & 9.65 & STOP & 33\\
2 & 91415 & 3.287 & 9.86 & STOP & 33\\
3 & 91411 & 8.948 & 10.04 & STOP & 33\\
3 & 91412 & 0.3638 & 91.67 & STOP & 33\\
3 & 91413 & 9.027 & 10.07 & STOP & 33\\
3 & 91414 & 0.449 & 87.90 & STOP & 33\\
3 & 91415 & 0.3911 & 80.03 & STOP & 33\\
\bottomrule\end{tabular}\end{table*}
Joint refinement certifies all fifteen STOP decisions within 91.68 seconds, saving 495 cells under the same acquisition policy. The ten-second energy and moment results above retain their original timing boundary. The additional two decisions occur after that boundary.

\section{General-Coarsening Experiments and Components}\label{app:general-public}
The sixty-condition study uses midpoint ridge candidates, $\lambda=1/32$, 128 uncertain cells and all original features. Its first-stage tail and final-width targets are $2.5\times10^{-6}$ and $3\times10^{-6}$, respectively. A one-second structural search is followed, when needed, by the geometry-based choice of moment or compressed-row refinement in Appendix~\ref{app:compact-profile}, within a 9000-second complete algorithm cap. Complete runs use one pinned logical CPU of an Intel Xeon Gold 5320 with single-thread numerical libraries.
\begin{table*}[t]
\caption{Complete original-regret certification on all sixty general-coarsening conditions, five masks per row. Every condition retains all model features and 128 uncertain cells in 64 rows. Widths use $10^{-6}$ normalized objective units; the common target is 3. Complete seconds include representation selection, input and setup, reference and witness fitting, numerical proposals and exact bound evaluation. Replay seconds report the median independent rational audit time.}\label{tab:confirmation}
\centering\scriptsize\setlength{\tabcolsep}{4pt}\begin{tabular}{lrrrrrrrrr}\toprule
Data&$p$&Bin&Cols&Certified&Median width&Max width&Median(s)&Max(s)&Replay(s)\\\midrule
Diabetes&11&1/32&2&5/5&0.066&0.112&30.36&32.13&16.08\\
Diabetes&11&1/32&8&5/5&0.117&0.147&14.09&22.96&15.94\\
Diabetes&11&1/8&2&5/5&2.770&2.999&152.95&156.76&43.25\\
Diabetes&11&1/8&8&5/5&2.131&2.549&43.17&45.79&14.77\\
Concrete&9&1/32&2&5/5&0.656&0.706&0.29&0.44&1.23\\
Concrete&9&1/32&8&5/5&0.509&0.790&0.29&0.31&1.20\\
Concrete&9&1/8&2&5/5&0.001&0.008&37.60&38.18&14.93\\
Concrete&9&1/8&8&5/5&0.331&0.339&43.71&44.49&14.84\\
Superconductivity&82&1/32&2&5/5&0.336&0.541&3.44&3.54&43.46\\
Superconductivity&82&1/32&8&5/5&0.552&0.859&3.38&3.52&43.14\\
Superconductivity&82&1/8&2&5/5&2.702&2.949&3.92&4.06&45.69\\
Superconductivity&82&1/8&8&5/5&1.193&1.547&110.77&111.87&114.72\\
\bottomrule\end{tabular}\end{table*}

\subsection{Additional fixed-rank component controls}
Six fixed tables additionally use rank 3, one-second caps and a $10^{-6}$ target under three specified alternatives: center-gradient bounds only, leading directions of the transformed midpoint ridge metric, and an anchor-Cholesky transformation. Every alternative is run on every table. The midpoint transformation approximately whitens $M_{\mathrm{mid}}$, so this metric control is distinct from the feature-Gram PCA comparison in the main text. The table separates these implementation choices from the rank change.
\begin{table*}[t]\caption{Own absolute widths in $10^{-6}$ units at fixed rank 3. Base: common midpoint transformation, gradient proposal, combined corner/center bound. Center: remove corner interpolation. Metric: use leading eigenvectors of the actual transformed midpoint ridge metric $T^\top M_{\mathrm{mid}}T$. Anchor: use a verified anchor-Cholesky transformation. Original candidate and penalty are held fixed.}\centering\small\begin{tabular}{lrrrrr}\toprule Data&Cols&Base&Center&Metric&Anchor\\\midrule
Diabetes&2&94.24&125.07&131.28&93.29\\
Diabetes&8&78.85&78.85&78.96&63.74\\
Concrete&2&5.21&10.27&11.11&5.32\\
Concrete&8&8.30&8.33&8.95&7.73\\
Superconductivity&2&2.65&6.35&7.57&1.96\\
Superconductivity&8&7.64&7.64&8.43&6.04\\\bottomrule\end{tabular}\end{table*}

\section{End-to-End Precision and the Fixed-Point Contract}\label{app:complete-v91}
At a fixed retained rank, valid row-query enclosures, a verified containing root, valid node uppers and the verified additive tail yield $L\le H^*\le U$ at every search cap and tail floor. A requested width is attained when the recorded total error meets that request. For rank-adaptive completeness, permit the full physical span with $R=I$ and an invertible verified $T$; $T=I$ is always available. Positive ridge supplies a uniform coercivity constant, and bounded row boxes supply a finite original upper $U_0$ and containing displacement radius. The discarded error is exactly zero.

For target $\delta>0$, choose a finite grid for which the curvature-interpolation error is at most $\delta/2$. Enclose every required row sum and regularization term to overall query error at most $\delta/4$ on either side, increasing arithmetic precision as required. The maximum grid lower comes from a feasible original table and comparator. The largest interpolated upper exceeds it by at most $\delta/2+2(\delta/4)=\delta$. The number of grid points is finite by the complexity bound in Section~6. This establishes attainability with sufficient search and arithmetic precision. The fixed-rank implementation records its tail and arithmetic floor separately; rank enrichment is an explicit option in this proposition.

\subsection{Executable fixed-point predictor contract}\label{app:fixed-point-v91}
The experiment specifies signed eight-bit coefficient integers $b_j$ decoded as $b_j/16$, and signed sixteen-bit feature integers $q_{ij}$ decoded as $q_{ij}/1024$. The coefficient range is $[-8,127/16]$; the input range is $[-32,32767/1024]$. Prediction uses the signed thirty-two-bit sum $s_i=\sum_jq_{ij}b_j$ and decoding $s_i/16384$. All twenty primary candidates and midpoint tables satisfy this representation. The executable check verifies exact agreement with decoded real arithmetic and bounds every partial accumulation by $\sum_j\max\{|1024L_{ij}|,|1024U_{ij}|\}|b_j|<2^{31}$. The contract supplies a concrete fixed-grid auditing task; candidates with eight fractional bits and unrounded candidates are separate candidate-precision conditions. Continuous-box regret certificates also cover all legal encoded completions.

\section{Joint Residual Refinement}\label{app:joint-refinement-v92}
The residual-energy first stage identifies small certifying spans. When its sufficient tail exceeds the requested precision, we refine the original completion domain with a reference-residual certificate. The two stages share the original candidate, ridge penalty, completion boxes and full-feature retrained comparator.

\begin{proposition}[Joint reference-residual bound]\label{prop:joint-refinement-v92}
Let $G_0\succ0$ be a uniform lower bound on $M_X$, and verify $T^\top G_0T\succeq\mu I$ for an invertible actual matrix $T$ and $\mu>0$. For any fixed reference $\vartheta$, put $g_\vartheta(X)=M_X\vartheta-r_X$. Then
\begin{align}
h_\beta(X)&=J_X(\beta)-J_X(\vartheta)\nonumber\\
 &\quad+g_\vartheta(X)^\top M_X^{-1}g_\vartheta(X),\label{eq:reference-identity-v92}\\
h_\beta(X)&\le P_\vartheta(X):=J_X(\beta)-J_X(\vartheta)\nonumber\\
 &\quad+\mu^{-1}\|T^\top g_\vartheta(X)\|^2.\label{eq:joint-refinement-v92}
\end{align}
On an original completion subbox, $P_\vartheta$ is a polynomial of degree at most four in normalized cell coordinates $z\in[0,1]^d$.
\end{proposition}
\begin{proof}
Completing the square around the exact full-feature ridge solution gives~\eqref{eq:reference-identity-v92}. The verified metric implies $M_X^{-1}\preceq\mu^{-1}TT^\top$, which gives~\eqref{eq:joint-refinement-v92}. Each row is affine in its uncertain cells. The objective difference is quadratic, $g_\vartheta$ is quadratic, and its squared transformed norm has degree at most four.
\end{proof}

The polynomial retains the dependence between the objective difference and the normal-equation residual. Write $x_{ij}=x^0_{ij}+d_a z_a$ for cell $a=(i,j)$, with $x^0$ an endpoint table and signed length $d_a$. Before applying $T^\top$, the constant gradient coefficient is $g_\vartheta(X^0)$; the linear coefficient is $w_i d_a[(x_i^{0\top}\vartheta-y_i)e_j+\vartheta_j x_i^0]$. The square coefficient is $w_i d_a^2\vartheta_j e_j$. For two distinct cells $a=(i,j)$ and $b=(i,k)$ in one row, their product coefficient is $w_i d_a d_b(\vartheta_k e_j+\vartheta_j e_k)$. Squaring these explicit coefficients constructs the joint polynomial with outward arithmetic.

\paragraph{Lifted row residuals and box localizers.}
For a row with several uncertain entries, introduce a normalized residual $u_i\in[0,1]$ through
\[
 x_i^\top\vartheta-y_i=\ell_i+s_i u_i,
\]
where $s_i>0$ and $[\ell_i,\ell_i+s_i]$ contains the exact residual range. We use dyadic endpoints and spans. Two outward-rounded affine inequalities encode this relation, together with their products by each normalized variable and its complement. The transformed normal residual is bilinear in cell coordinates and the corresponding $u_i$. Its moment basis therefore has at most $1+2d+m$ entries for $d$ uncertain cells in $m$ rows.

Let $v_i$ contain the constant, the normalized cells in row $i$, and $u_i$. For each coordinate $q$ in that row, we add the localizing constraint
\begin{equation}
 \mathcal L_{i,q}(y)
 :=\mathbb E_y\!\left[q(1-q)v_i v_i^\top\right]\succeq0.
 \label{eq:row-localizer-v92}
\end{equation}
Here $\mathbb E_y$ denotes substitution of monomials by lifted moments. Every actual completion satisfies~\eqref{eq:row-localizer-v92}, since $q(1-q)\ge0$ and $v_i v_i^\top\succeq0$. For a row containing $k_i$ uncertain cells, this adds $k_i+1$ matrices of order $k_i+2$. These constraints retain the row's box geometry in the fourth-degree relaxation.

\paragraph{A validated moment dual.}
We use polynomial moment relaxations~\citep{lasserre2001}, modeled in CVXPY~\citep{diamond2016cvxpy}, with SCS numerical proposals~\citep{odonoghue2016scs}.
Let $y_\alpha=z^\alpha$ for every required monomial and its factors, with $0\le y_\alpha\le1$. Product relations give $y_{\alpha+\gamma}\le y_\alpha,y_\gamma$ and $y_{\alpha+\gamma}\ge y_\alpha+y_\gamma-1$. Square tangents use $y_{2\alpha}\ge2t y_\alpha-t^2$ for $t\in\{0,1/4,1/2,3/4,1\}$. Write these inequalities as $Ay\le b$. The moment matrix $\mathcal M(y)$ indexed by the constant, linear and within-row quadratic gradient monomials satisfies $\mathcal M(y)\succeq0$ on every genuine completion.

Enclose the polynomial coefficients above as $c_0,c$. For any nonnegative multiplier $\lambda$ and any matrix $R$, write
\[
 \langle RR^\top,\mathcal M(y)\rangle=s_0+s^\top y.
\]
Every allowed completion obeys the independently checkable upper bound
\begin{equation}
 P_\vartheta(X)\le c_0+\lambda^\top b+s_0
       +\sum_j\max\{0,c_j-(A^\top\lambda)_j+s_j\}.
 \label{eq:validated-moment-dual-v92}
\end{equation}
Indeed, add the nonnegative moment quadratic and the nonnegative inequality slack to the polynomial, then bound the remaining linear form on $[0,1]$. Numerical conic optimization proposes the multipliers. The implementation rounds $\lambda$ to nonnegative dyadics and stores a dyadic factor $R$, which establishes positive semidefiniteness directly. Outward evaluation of~\eqref{eq:validated-moment-dual-v92} includes the complete dual residual. The independent auditor reconstructs the original polynomial, actual metric, multiplier constraints and factor products as exact rationals.

For the localized relaxation, each matrix in~\eqref{eq:row-localizer-v92} receives its own dyadic factor. Their inner products contribute additional constant and monomial coefficients to $s_0$ and $s$ in~\eqref{eq:validated-moment-dual-v92}. The same residual correction proves the resulting upper bound. The independent checker also reconstructs the residual range, affine links, and every localizing polynomial $q-q^2$.

\paragraph{Reference and subdivision rules.}
A fixed 64-start feasible search supplies a completed-table lower witness. An eight-step finite-pool minimax exchange proposes a reference, using 32 starts per exchange step and fixed random seeds. Each proposed reference receives the uniform polynomial certificate above. For an unresolved node, we bisect the cell maximizing relaxed moment variance times its polynomial coefficient sensitivity. The parent remains active until both covering children finish. The largest active upper, intersected with the first-stage upper, and the best feasible original-table lower give the global interval. Every reference fit, witness search, polynomial construction, conic solve and dual evaluation enters the refinement clock.

The prespecified new-mask computation cap is 600 seconds, with completed-evidence checkpoints at 10, 60, 120, 300 and 600 seconds. All twenty new masks finish within 76.21 seconds. The original rank-three guarantee belongs to the structural first stage; this completion-domain refinement supplies additional accuracy when needed. Its computational cost depends on the unresolved cells and the moment relaxation.

\subsection{Row-edge duals and certificate screening}\label{app:edge-refinement-v22}
For a small number of dense uncertain rows, we use Theorem~\ref{thm:row} to place the coupling certificate on a product of row edges. Let $m$ be the number of uncertain rows, fix a reference $\vartheta$, and write
\begin{align*}
 g_0&=G_0\vartheta-r_0,\\
 \gamma_i(x_i)&=T^\top\!\left[w_i x_i(x_i^\top\vartheta-y_i)+g_0/m\right].
\end{align*}
Then $\sum_i\gamma_i=T^\top g_\vartheta(X)$. With
\[
 f_i(x_i)=w_i[(x_i^\top\beta-y_i)^2-(x_i^\top\vartheta-y_i)^2],
\]
the joint majorant has the form
\[
 P_\vartheta(X)=c+\sum_i f_i(x_i)+\mu^{-1}\Big\|\sum_i\gamma_i(x_i)\Big\|^2,
\]
where $c=J_0(\beta)-J_0(\vartheta)$ contains the complete rows and ridge term.

\begin{proposition}[Row-edge coupling certificate]\label{prop:edge-dual-v22}
Let $Q\succeq0$ have scalar block $Q_{00}=\alpha$, blocks $Q_{i0}=d_i$, diagonal blocks $Q_{ii}=H_i-\mu^{-1}I$, and off-diagonal row blocks $Q_{ij}=-\mu^{-1}I$. Define
\[
 R_i=f_i+\gamma_i^\top H_i\gamma_i+2d_i^\top\gamma_i,
 \qquad b_i\ge\max_{x_i\in E_i}R_i(x_i),
\]
where $E_i$ is the union of the row box's edges. Then
\[
 H^*\le U:=c+\alpha+\sum_i b_i.
\]
\end{proposition}
\begin{proof}
For $v=(1,\gamma_1,\ldots,\gamma_m)$, direct expansion gives
\begin{align*}
 U-P_\vartheta(X)&=v^\top Qv+\sum_i[b_i-R_i(x_i)]\\
 &\ge0\quad\text{on }\prod_iE_i.
\end{align*}
The row-face theorem places an original-regret maximizer in this product, and the joint majorant bounds its value.
\end{proof}

Each row feature is quadratic in its normalized cells. On one edge, $R_i$ is a quartic in one scalar. The numerical exchange method solves the coupling SDP on a finite set of row points and adds distinct violating points from the row oracles. Its reported upper uses the complete row-oracle bounds. The verifier reconstructs the exact polynomial coefficients and cross-row matrix blocks, checks a rational congruence of $Q$ with positive diagonal-dominance margins, and encloses every quartic maximum by integer Bernstein subdivision.

The certificate also restricts the remaining search. For a target value $\tau<U$, every edge-domain completion with original regret exceeding $\tau$ must satisfy
\begin{equation}
 R_i(x_i)>b_i-(U-\tau)\quad\text{for every uncertain row }i.
 \label{eq:edge-screen-v22}
\end{equation}
Indeed, a violation of one inequality makes the corresponding nonnegative slack at least $U-\tau$. An exact Bernstein upper below this cutoff excludes the entire edge segment. The retained segments cover every possible target violation.

We partition this cover using the fixed-reference residuals $x_i^\top\vartheta-y_i$. A residual slab clips each edge to an interval, so the same quartic verifier applies. The numerical master proposes the row with largest residual variance and a dyadic split at its weighted mean, clipped to the inner 70\% of the available range. Both complementary children are verified before replacing their parent. The final original-regret upper is
\[
 \max\{\tau,\max_{B\in\mathcal L}U_B\},
\]
intersected with the verified root bound; $\mathcal L$ is the complete active cover. Every node retains the original candidate, feature intervals, ridge penalty and full-feature comparator.

The dense-row protocol fixes the target $\varepsilon=10^{-8}$, uses $\tau=L+0.99\varepsilon$, and sets the complete cap at 1800 seconds. It uses a 4096-start feasible search, an eight-step reference exchange with 32 starts per step, up to 180 seconds for the root proposal and 90 seconds for each subsequent proposal, and at most 65 evaluated nodes. Exact bound evaluation, screening, reference construction and all completed child work enter the algorithm clock. A separate rational replay checks the original witness, matrix, exclusions and complete tree. The implementation and all four dense-row conditions use this common protocol.

\subsection{Algebraic Row Compression and Observed-Column Elimination}\label{app:compact-profile}
The general-coarsening solver combines two exact reductions with the validated row-edge bound. Algebraic row coordinates reduce the coupling matrices, and analytic elimination of fully observed columns reduces the outer comparator search when uncertainty is concentrated in a few columns.

\paragraph{Algebraic row coordinates.}
For row $i$, let $x_i=x_i^0+D_i z_i$, where $D_i$ has one scaled coordinate vector for each of its $k_i$ uncertain cells and $z_i\in[0,1]^{k_i}$. Fix a reference $\vartheta$, put $e_i^0=x_i^{0\top}\vartheta-y_i$ and $q_i=\vartheta^\top D_i z_i$, and define
\[
 u_i(z_i)=\begin{bmatrix}q_i\\z_i(e_i^0+q_i)\end{bmatrix},\qquad
 V_i=w_i\begin{bmatrix}T^\top x_i^0&T^\top D_i\end{bmatrix}.
\]
The transformed row-gradient increment is exactly $V_i u_i$. Thus, for a constant $a_0$ containing the endpoint gradients, observed rows and ridge,
\[
 T^\top g_\vartheta(X)=a_0+\sum_i V_i u_i(z_i).
\]
This uses at most $k_i+1$ scalar coordinates per row, independently of the number of model features. Exact zero columns can be omitted, and diagonal dyadic rescaling preserves the identity.

Write the joint majorant as
\[
 C+\sum_i\bar f_i(z_i)+c\Big\|\sum_iV_i u_i(z_i)\Big\|^2,
 \qquad c=\mu^{-1},
\]
where $C$ absorbs $c\|a_0\|^2$ and $\bar f_i$ absorbs $2c\,a_0^\top V_i u_i$. Choose $H_i\succ0$, $d_i$ and $\alpha$ satisfying
\begin{equation}
 \begin{bmatrix}\alpha&0\\0&c^{-1}I\end{bmatrix}
 -\sum_i\begin{bmatrix}d_i^\top\\V_i\end{bmatrix}
 H_i^{-1}\begin{bmatrix}d_i&V_i^\top\end{bmatrix}\succeq0.
 \label{eq:compact-schur}
\end{equation}
Two Schur complements give
\[
 c\Big\|\sum_i V_i u_i\Big\|^2
 \le\alpha+\sum_i\{u_i^\top H_i u_i+2d_i^\top u_i\}.
\]
Consequently, the sum of the corresponding row-edge quartic maxima plus $C+\alpha$ bounds the original regret. The numerical master uses one coupling cone and one small block cone per row. The independent verifier checks each local inverse exactly, encloses its Schur contribution by a dyadic positive-semidefinite upper matrix, verifies the central matrix by rational LDL elimination, and checks all row quartics by Bernstein subdivision. These checks retain all cross-row terms.

\paragraph{Curvature from uncertain rows.}
Let $c_i,h_i$ be row midpoints and halfwidths, and let
\[
 D=\operatorname{diag}\!\left(\sum_{i\notin\mathcal A}w_i k_i h_{ij}^2\right)_j.
\]
For $0<t<1$, matrix Young's inequality gives the uniform lower matrix
\begin{equation}
 M_X\succeq G_t:=G_0+(1-t)\sum_{i\notin\mathcal A}w_i c_i c_i^\top
 -(t^{-1}-1)D.\label{eq:young-curvature}
\end{equation}
Indeed, $(c+\delta)(c+\delta)^\top\succeq(1-t)cc^\top-(t^{-1}-1)\delta\delta^\top$ and $\delta\delta^\top\preceq k_i\operatorname{diag}(h_i^2)$. The implementation uses $t=1/2$ when its rounded metric is verified positive. This strengthens the common curvature used by the reference certificate.

\begin{proposition}[Profiling fully observed columns]\label{prop:observed-profile}
Let $J$ contain every uncertain feature column, let $K$ be its complement, and put $A=\lambda I+X_K^\top W X_K\succ0$. For $t\in\mathbb R^{|J|}$, define
\[
 F_J(t)=\max_{X\in\mathcal X}\{J_X(\beta)-\min_b J_X(\theta_J=t,\theta_K=b)\}.
\]
For any reference $\vartheta$ with $\vartheta_J=t$, its fixed-completion objective equals
\begin{equation}
 J_X(\beta)-J_X(\vartheta)+g_{\vartheta,K}(X)^\top A^{-1}g_{\vartheta,K}(X).
 \label{eq:profile-identity}
\end{equation}
Each fixed-completion function is concave quadratic in $t$, with negative quadratic coefficient
\[
 S_X=M_{JJ}-M_{JK}A^{-1}M_{KJ}\preceq M_{JJ}.
\]
If $U_0\ge H^*$, then
\[
 H^*=\max_{|t_j-\beta_j|\le\sqrt{U_0/\lambda}}F_J(t).
\]
\end{proposition}
\begin{proof}
Completing the square in $b$ proves~\eqref{eq:profile-identity}; a second completion of the square gives $S_X$. The full optimum satisfies $\lambda\|\theta_X^*-\beta\|^2\le h_\beta(X)\le U_0$, proving the containing box. Maximization over that box and the original completion domain recovers the full retraining optimum.
\end{proof}
For each corner query, a verified factor of $A$ bounds the last term in~\eqref{eq:profile-identity}. The transformed gradient of an uncertain row varies along the fixed direction obtained from its observed entries, giving a single algebraic row coordinate. A maximizing pair $(X,b)$ at fixed $t$ admits a row-edge completion by Theorem~\ref{thm:row}. Thus the compressed row certificate applies to each query. Theorem~\ref{thm:interp}, with
\[
 H_j=\lambda+\sum_iw_i\max\{L_{ij}^2,U_{ij}^2\},
\]
certifies every outer box. This computational reduction retains unrestricted retraining in $K$. Its search dimension is $|J|$; the residual-energy theorem separately controls certification dimension through residual structure.

\paragraph{Choosing the conic representation.}
The implementation compares the sizes of the two formulations using only $p$ and the uncertain-cell counts $k_i$. Put
\[
 \begin{aligned}
 q&=1+\sum_i k_i+\sum_i k_i(k_i+1)/2,\\
 C_{\rm moment}&=q(q+1)/2,
 \end{aligned}
\]
and
\[
 C_{\rm row}=\frac{(p+1)(p+2)}2+
 \sum_i\frac{(p+k_i+2)(p+k_i+3)}2.
\]
These count symmetric entries in the principal moment cone and the compressed block cones. The solver chooses the moment formulation when $C_{\rm moment}<C_{\rm row}$ and compressed rows otherwise. This geometry-based rule uses the same original objective and certificate tolerance in both branches.

\paragraph{Common general-coarsening protocol.}
Every condition uses rank-one/two/three structural selection and a one-second initial search. Both refinement branches start with 64 feasible starts. The moment branch uses eight reference-exchange rounds with 32 starts, one root-reference update, 60-second root and 40-second child proposals, and at most 33 evaluated nodes. A leading relaxed covariance direction and its mean propose complementary halfspaces; their nonnegative linear slacks and products with normalized variables enter the validated dual. The compressed branch uses sixteen reference exchanges with 32 starts and a 90-second root proposal. Remaining compressed cases with at most three uncertain columns use the profiled comparator search with at most 1023 evaluated boxes; other cases use fixed-reference residual slabs with at most 513 nodes and 60 seconds per proposal. The complete cap is 9000 seconds. Representation selection, input construction, reference fitting, numerical proposals, exact evaluations and completed partial-child work enter the clock. A parent stays active until both covering children are verified. A separate rational replay checks the representation rule, original witness, every query and the full cover.

\begin{algorithm}[t]
\caption{Original-table regret certification}\label{alg:certificate}
\begin{algorithmic}[1]
\REQUIRE Row boxes, $y,w,\lambda,\beta$, rank cap, tolerance $\varepsilon$, total budget $B$, structural-stage budget $B_s\le B$
\STATE Start the complete clock; charge input and preprocessing to $B$.
\STATE Form $G_0,r_0$ from complete rows and ridge once.
\STATE Propose dyadic $T$; verify its metric and an original upper $U_0$.
\STATE Propose uncertainty-gradient spans or signal-row spans plus an anchor-correction direction; verify the actual Gram/projector.
\STATE Select a span using verified spectral, moment or structural energy budgets.
\STATE Establish a finite root; initialize a feasible original lower $L$.
\WHILE{complete time is below $B_s$ and $U-L>\varepsilon$}
\STATE Pop the leaf with largest upper; bisect its weighted widest side.
\STATE Query row maxima; bound both children; update a feasible table.
\STATE Retrain that table in all $p$ features and update $L$.
\STATE Update $U$ from surviving leaf uppers, tail and $U_0$.
\ENDWHILE
\STATE If $U-L>\varepsilon$, apply the validated joint residual refinement; use observed-column elimination or row-edge subdivision for the general-coarsening protocol (Appendix~\ref{app:compact-profile}).
\STATE \textbf{return} $[L,U]$, completed-table witness, error ledger and stop reason
\end{algorithmic}
\end{algorithm}

\end{document}